\documentclass[ejsv2,noshowframe]{imsart}

\RequirePackage[authoryear]{natbib}
\RequirePackage[colorlinks,citecolor=blue,urlcolor=blue]{hyperref}
\RequirePackage{graphicx}

\RequirePackage{amssymb,amsthm,amsmath,amsfonts,mathtools}
\usepackage{subcaption}
\usepackage{amsmath,bm}
\usepackage{algorithm,algpseudocode}
\usepackage{tikz}
\usepackage{multirow}

\arxiv{2303.11786}
\startlocaldefs
\numberwithin{equation}{section}
\theoremstyle{plain}

\newtheorem{theorem}{Theorem}[section]

\newtheorem{thm}{Theorem}

\newtheorem{prop}[thm]{Proposition}

\theoremstyle{definition}

\newtheorem{remark}{Remark}

\theoremstyle{remark}

\newcommand{\norm}[1]{\left\Vert#1\right\Vert}
\newcommand{\abs}[1]{\left\vert#1\right\vert}
\newcommand{\brac}[1]{\left(#1\right)}

\newcommand{\sbrac}[1]{\left[#1\right]}
\newcommand{\set}[1]{\left\{#1\right\}}

\newcommand{\eps}{\varepsilon}

\newcommand{\lam}{\lambda}

\newcommand{\sig}{\sigma}

\newcommand{\CC}{{\mathbb C}}

\newcommand{\EE}{{\mathbb E}}

\newcommand{\PP}{{\mathbb P}}

\newcommand{\RR}{{\mathbb R}}

\newcommand{\calB}{{\mathcal B}}
\newcommand{\calC}{{\mathcal C}}

\newcommand{\calE}{{\mathcal E}}

\newcommand{\calI}{{\mathcal I}}

\newcommand{\calS}{{\mathcal S}}

\newcommand{\calV}{{\mathcal V}}

\newcommand{\calX}{{\mathcal X}}

\newcommand{\bbeta}{\bm{\beta}}

\newcommand{\bR}{\bm{R}}

\newcommand{\bS}{\bm{S}}
\newcommand{\bs}{\bm{s}}

\newcommand{\bv}{\bm{v}}

\newcommand{\bz}{\bm{z}}
\newcommand{\bZ}{\bm{Z}}

\newcommand{\by}{\bm{y}}
\newcommand{\bX}{\bm{X}}
\newcommand{\bx}{\bm{x}}
\newcommand{\bY}{\bm{Y}}

\endlocaldefs

\begin{document}
\begin{frontmatter}

\title{Skeleton Regression: A Graph-Based Approach to Estimation with Manifold Structure}
\runtitle{Skeleton Regression}

\begin{aug}
\author[A]{\fnms{Zeyu}~\snm{Wei}\ead[label=e1]{zwei5@uw.edu} \orcid{0000-0003-1614-4458}} \and
\author[A]{\fnms{Yen-Chi}~\snm{Chen}\ead[label=e2]{yenchic@uw.edu}\orcid{0000-0002-4485-306X}}

\address[A]{Department of Statistics,
University of Washington\printead[presep={,\ }]{e1,e2}}

\end{aug}

\begin{abstract}
We introduce a new regression framework designed to deal with large-scale, complex data that lies around a low-dimensional manifold with noises. 
Our approach first constructs a graph representation, referred to as the \textit{skeleton}, to capture the underlying geometric structure. 
We then define metrics on the skeleton graph and apply nonparametric regression techniques, along with feature transformations based on the graph, to estimate the regression function.
We also discuss the limitations of some nonparametric regressors with respect to the general metric space such as the skeleton graph.
The proposed regression framework suggests a novel way to deal with data with underlying geometric structures and provides additional advantages in handling the union of multiple manifolds, additive noises, and noisy observations.
We provide statistical guarantees for the proposed method and demonstrate its effectiveness through simulations and real data examples. 

\end{abstract}



\end{frontmatter}
\tableofcontents

\section{Introduction}


~~~~ Many data nowadays are geometrically structured that the covariates lie around a low-dimensional manifold embedded inside a large-dimensional vector space. 
Among many geometric data analysis tasks, the estimation of functions defined on manifolds has been extensively studied in the statistical literature. 
A classical approach to explicitly account for geometric structure takes two steps: map the data to the tangent plane or some embedding space, and then run regression methods with the transformed data.
This approach is pioneered by the Principle Component Regression (PCR) \cite{PCR} and the Partial Least Squares (PLS) \cite{Wold1975}.
\cite{Aswani2011} innovatively relates the regression coefficients to exterior derivatives.  
They propose to learn the manifold structure through local principal components and then constrain the regression to lie close to the manifold by solving a weighted least-squares problem with Ridge regularization.
\cite{Cheng2013} present the Manifold Adaptive Local Linear Estimator for the Regression (MALLER) that performs the local linear regression (LLR) on a tangent plane estimate. 
However, because those methods directly exploit the local manifold structures in an exact sense, they are not robust to variations in the covariates that perturb them away from the true manifold structure.

Many other manifold estimation approaches exist in the statistical literature.
\cite{Guhaniyogi2016} utilize random compression of the feature vector in combination with Gaussian process regression.
\cite{Yuchen2013} follows a divide-and-conquer approach that computes an independent kernel Ridge regression estimator for each randomly partitioned subset and then aggregates.
Other nonparametric regression approaches such as kernel machine learning \citep{Bernhard2002}, manifold regularization \citep{belkin06a}, and the spectral series approach \citep{Lee2016} also account for the manifold structure of the data.
More recently, \cite{Green2021} proposes the Principal Components Regression with Laplacian-Eigenmaps (PCR-LE) that projects data onto the eigenvectors output by Laplacian Eigenmaps and provides the rates of convergence of such nonparametric regression method over Sobolev spaces. 
However, those methods still suffer from the curse of dimensionality with large-dimensional covariates.


In addition to data with manifold-based covariates, manifold learning has been applied to other types of manifold-related data. 
\cite{Marzio2014} develop nonparametric smoothing for regression when both the predictor and the response variables are defined on a sphere.
\cite{Zhang2019} deal with the presence of grossly corrupted manifold-valued responses.
\cite{Lin2020} address data with functional predictors that reside on a finite-dimensional manifold with contamination. 
In this work, we focus on manifold-based covariates and may incorporate other types of manifold-related data in the future.

The main goal of this work is to estimate a scalar response with covariates lying around some manifold structures in a way that utilizes the geometric structure and bypasses the curse of dimensionality.
This is achieved by proposing a new framework that combines graphs and nonparametric regression techniques. 
Our framework follows the two-step idea: first, we learn a graph representation, which we call the \textit{skeleton}, of the manifold structure based on the methods from \citet{skelclus} and project the covariates onto the skeleton. Then we apply different nonparametric regression methods with the skeleton-projected covariates. We give brief descriptions of the relevant nonparametric regression methods below.

Kernel smoothing is a widely used technique that estimates the regression function as locally weighted averages with the kernel as the weighting function. Pioneered by the famous Nadaraya–Watson estimator from \cite{Nadaraya1964} and \cite{Watson1964}, this technique has been widely used and extended by recent works \citep{Fan1992, Hastie1993, Fan1996, Kpotufe2017}. 
Splines \citep{ESL, Friedman1991} are popular nonparametric regression constructs that take the derivative-based measure of smoothness into account when fitting a regression function.
Moreover, k-Nearest-Neighbors (kNN) regression \citep{Altman1992, ESL} has a simple form based on a distance metric but is powerful and widely used in many applications. These techniques are incorporated into our proposed regression framework.

In recent years, many nonparametric regression techniques have been shown to adapt to the manifold structure of the data, with convergence rates that depend only on the intrinsic dimension of the data space. Specifically, the classical kNN and kernel regressor have been shown to be manifold-adaptive with proper parameter tuning procedures \citep{Kpotufe2009,  Kpotufe2011, Kpotufe2013, Kpotufe2017}, while recent methods like the Spectral Series regression and PCR-LE also enjoy this property \citep{Green2021}.
The proposed regression framework in this work also adapts to the manifold, as the nonparametric regression models fitted on a graph are dimension-independent. This framework has several additional advantages such as the ability to account for predictors from distinct manifolds and being robust to additive noise and noisy observations.

\begin{figure}
\centering
    \begin{subfigure}[t]{0.25\textwidth}
        \centering
        \includegraphics[width=\linewidth]{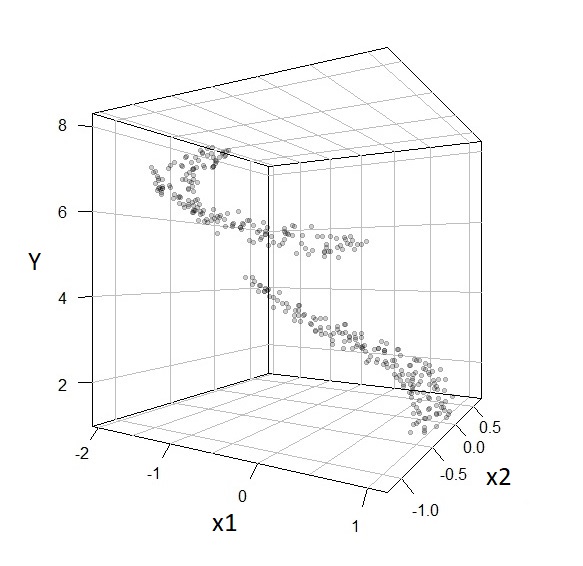} 
        \caption{Data}
    \end{subfigure}
        \begin{subfigure}[t]{0.3\textwidth}
        \centering
        \includegraphics[width=\linewidth]{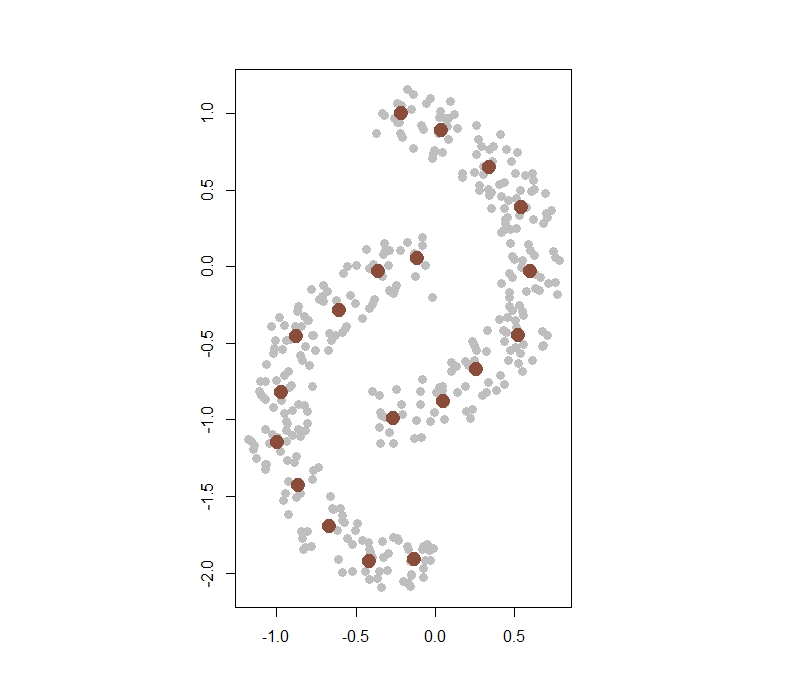}
        \caption{Knots}
    \end{subfigure}
    \begin{subfigure}[t]{0.3\textwidth}
        \centering
        \includegraphics[width=\linewidth]{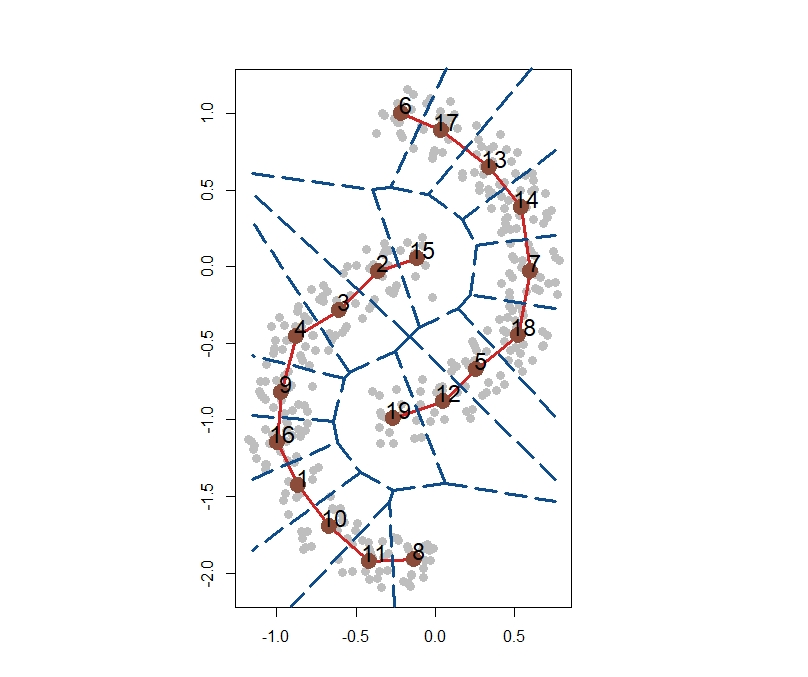}
        \caption{Skeleton}
    \end{subfigure}\\
        \begin{subfigure}[t]{0.25\textwidth}
        \centering
        \includegraphics[width=\linewidth]{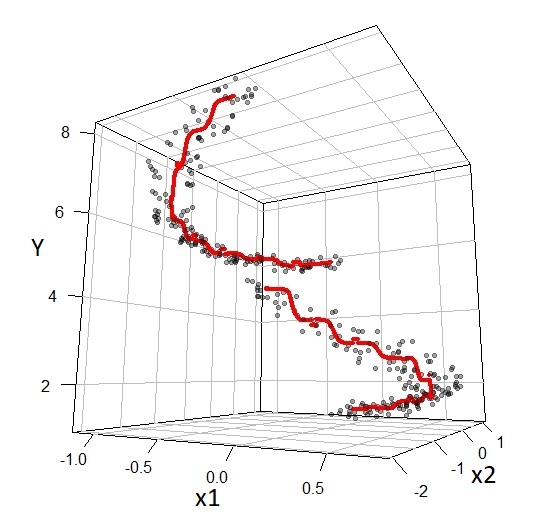} 
        \caption{S-Kernel Regression}
    \end{subfigure}
    \begin{subfigure}[t]{0.25\textwidth}
        \centering
        \includegraphics[width=\linewidth]{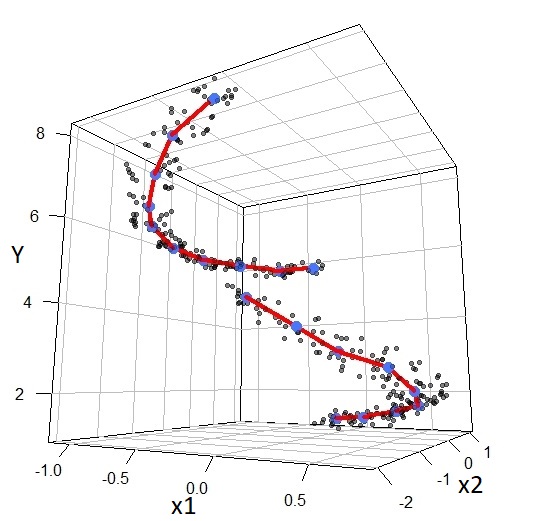} 
        \caption{Linear Spline}
    \end{subfigure}
\caption{Skeleton Regression illustrated by data with covariates having the shape of two moons in a 2D space.}
\label{fig::ex1}
\end{figure}



\emph{Outline.} 
We start by presenting the procedures of the skeleton regression framework in section \ref{sec::framework}.
In section \ref{sec::regression}, we apply nonparametric regression techniques to the constructed skeleton graph along with theoretical justifications.
In section \ref{sec::simulation}, we present some simulation results for skeleton regression and demonstrate the effectiveness of our method on real datasets in Section \ref{sec::real}. 
In section \ref{sec::conclusion}, we conclude the paper and point out some directions for future research.


\section{Skeleton Regression Framework} \label{sec::framework}

~~~~In this section, we introduce the skeleton regression framework. 
Given design vectors $\set{\bx_i}_{i=1}^n$ where $\bx_i \in \calX \subseteq \RR^d$ for each $i$ and the corresponding responses $\set{Y_i}_{i=1}^n$ in $\RR$, 
a traditional regression approach is to estimate the regression function $m(\bx) = \EE(Y|X = \bx)$.
However, the ambient dimension $d$ can be large while the covariates are distributed around some low-dimensional manifold structures. In this case, $\calX$ can be the union of several disjoint components with different manifold structures, and the regression function can have discontinuous changes from one component to another. 
To handle such geometrically structured data, we approach the regression task by first representing the sample covariate space with a graph, which we call the \textit{skeleton}, to summarize the manifold structures. We then focus on the regression function over the skeleton graph as a surrogate estimator to the true regression function (defined in Equation \ref{eq::projectedRegression}), which incorporates the covariate geometry in a dimension-independent way.

We illustrate our regression framework on the simulated Two Moon data in Figure \ref{fig::ex1}. The covariates of the Two Moon data consist of two $2$-dimensional clumps with intrinsically 1-dimensional curve structure, and the regression response increases polynomially with the angle and the radius (Figure \ref{fig::ex1} (a)). 
We construct the skeleton presentation to summarize the geometric structure (Figure \ref{fig::ex1} (b,c) ) and project the covariates onto the skeleton.
The regression function on the skeleton is estimated using kernel smoothing (Section \ref{sec::skelkernel}, illustrated in Figure \ref{fig::ex1} (d) ) and linear spline (Section \ref{sec::lspline}, illustrated in Figure \ref{fig::ex1} (e)).

The estimated regression function can be used to predict new projected covariates. We summarize the overall procedure in Algorithm \ref{alg::Skelreg}.

\begin{algorithm}
\caption{Skeleton Regression Framework}
\label{alg::Skelreg}

\begin{algorithmic}
\State \textbf{Input:} 
Observations $(\bx_1, Y_1), \dots, (\bx_n, Y_n)$.
\State 1. {\bf Skeleton Construction.} Construct a data-driven skeleton representation of the covariates preferably assisted with subject knowledge.
\State 2. {\bf Data Projection.} 
Project the covariates onto the skeleton.
\State 3. {\bf Skeleton Regression Function Estimation.}
Fitting regression function on the skeleton using nonparametric techniques such as kernel smoothing (Section \ref{sec::skelkernel}), k-Nearest Neighbor (Section \ref{sec::SkNN}), and linear spline (Section \ref{sec::lspline}). 
\State 4. {\bf Prediction.}
Project new covariates onto the skeleton and use the estimated regression function for prediction.
\end{algorithmic}

\end{algorithm}

\subsection{Skeleton Construction}
\label{sec::skeletoncons}


~~~~A skeleton is a graph constructed from the sample space representing regions of interest. 
From a statistical perspective, a region is of interest if it encompasses a sufficient measure of the probability distribution.
For given covariate space $\calX \subseteq \RR^d$,
let $\calV = \{V_j \in \RR^d : j =1 , \dots k\}$ 
be a collection of points of interest
and $E$ be a set of edges connecting points in $\calV$
such that an edge $e_{j\ell}\in E$ 
if the region between $V_j$ and $V_\ell$ is also of interest. 
The tuple $({\cal V}, E)$ together forms 
a graph that represents the focused regions in the sample space.
Notably, different from common graph-based regression approaches that take each sample covariate as a vertex, the set $\calV$ takes representative points of the covariate space and has size $k \ll n$ where $n$ is the sample size.
Moreover, the points on the edges are also part of the analysis as belonging to the regions of interest, which is different from the usual knot-edge graph.
While the graph $({\calV}, E)$ contains the region of interest, it is not easy to work with this graph directly. 
Thus, we introduce the concept of the skeleton induced by this graph.

Let $\calE = \{tV_j  + (1-t) V_\ell: t \in (0,1),  e_{j\ell} \in E \}$
be the collection of line segments induced by the edge set $E$. 
We define the skeleton of $({\cal V}, E)$
as $\calS = \calV \cup \calE $, i.e.,
$\calS$ is the points of interest and the associated line segments
representing the regions of interest. 
Clearly, $\calS$ is a collection of one-dimensional line segments
and zero-dimensional points, so it is independent of the ambient dimension $d$, but the physical location of $\calS$ is meaningful as representing the region of interest. 
The idea of skeleton regression is to build a regression model
on the skeleton $\calS$.


\subsubsection{A data-driven approach to construct skeleton}

~~~~The skeleton should ideally be constructed based on the analyst's judgment or prior knowledge of the focus regions. 
However, this information may be unavailable and we have to construct a skeleton from the data.
In this section, we give a brief description of a data-driven approach proposed in \cite{skelclus} that constructs the skeleton to represent high-density regions. 
The method constructs knots as the centers from the $k$-means clustering with a large number of centers \footnote{By default $[\sqrt{n}]$. We explore the effect of choosing different numbers of knots with empirical results.}.
The edges are connected by examining the sample 2-Nearest-Neighbor (2-NN) region of a pair of knots $(V_j,V_\ell)$  (see Figure \ref{fig::2nn}) defined as
\begin{align}
    B_{j\ell} = \{ X_m, m=1,\dots,n :\norm{X_m - V_i} > \max\{ \norm{X_m-V_j}, \norm{x- V_\ell} \}, \forall i \neq j, \ell \},
\label{eq::2NNregion}
\end{align}
where $||.||$ denotes the Euclidean norm, and an edge between $V_j$ and $V_\ell$ is added if $B_{j\ell}$ is non-empty. 
The method can further prune edges or segment the skeleton by using hierarchical clustering with respect to the Voronoi Density weights defined as
$S_{j\ell}^{VD} = \frac{\frac{1}{n}\abs{B_{j\ell}}}{\norm{V_j - V_\ell}}.$
We provide more details about this approach in Appendix \ref{ref::skelconsVoron}.


\begin{figure}
\centering
\includegraphics[height=5cm]{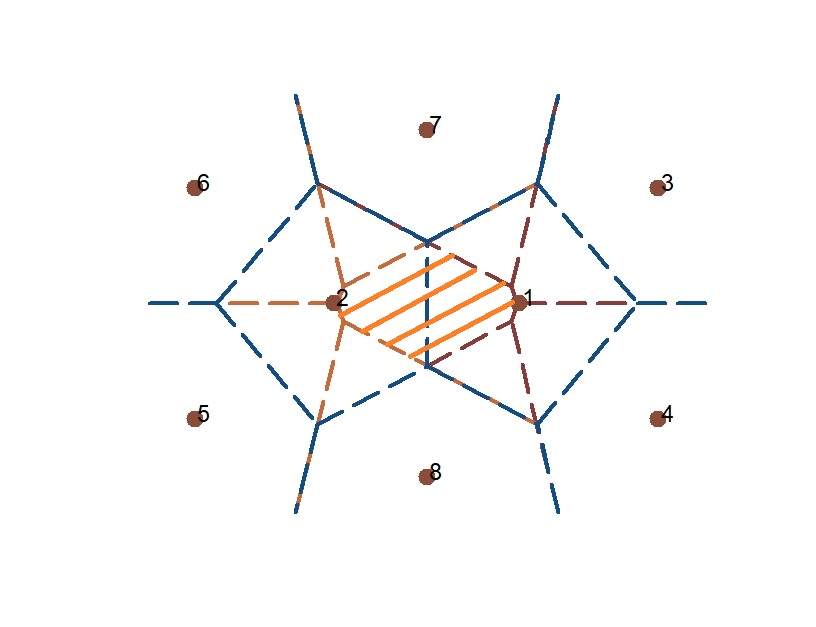}
\caption{Orange shaded area illustrates the 2-NN region between knots $1$ and $2$.}
\label{fig::2nn}
\end{figure}

\begin{remark}
The idea of using the $k$-means algorithm to divide data into cells
and perform analysis based on the cells has been proposed
in the literature for fast computation.
\cite{InvertedFiles}, when carrying out an approximate nearest neighbor search, proposed to divide the data into Voronoi cells by $k$-means and do a neighbor search only in the same or some nearby cells. \cite{InvertedMultiIndex} adopted the Product Quantization technique to construct cell centers for high-dimensional data as the Cartesian product of centers from sub-dimensions. 
\end{remark}



\subsection{Skeleton-Based Distance}

~~~~One of the advantages of the physically located skeleton is that it allows for a natural definition of the skeleton-based distance function $d_\calS(.,.): \calS \times \calS \to \RR^{+}\cup \{\infty\}$. 
Let $\bs_j, \bs_\ell \in \calS$ be two arbitrary points on the skeleton and note that, different from the usual geodesic distance on a graph, in our framework $\bs_j, \bs_\ell$ can be on the edges. We measure the skeleton-based distance between two skeleton points as the graph path length as defined below:

\begin{itemize}
    \item If $\bs_j, \bs_\ell$ are disconnected that they belong to two disjoint components of $\calS$, we define
    \begin{align}
    d(\bs_j, \bs_\ell) = \infty
    \label{eq::graphdist}
    \end{align} 
    \item If $\bs_j$ and $\bs_\ell$ are on the same edge, we define the skeleton distance as their Euclidean distance that
\begin{align}
    d_\calS(\bs_j, \bs_\ell) = ||\bs_j - \bs_\ell|| 
    \label{eq::graphdist0}
\end{align}
    \item For $\bs_j$ and $\bs_\ell$ on two different edges that share a knot $V_0$, the skeleton distance is defined as
    \begin{align}
    d_\calS(\bs_j, \bs_\ell) = ||\bs_j - V_{0}|| + ||\bs_\ell - V_{0}||
    \label{eq::graphdist1}
\end{align}
    \item Otherwise, let knots $V_{i(1)}, \dots, V_{i(m)}$ be the vertices on a path connecting $\bs_j, \bs_\ell$, where $V_{i(1)}$ is one of the two closest knots of $\bs_j$ and $V_{i(m)}$ is the other closest knots of $\bs_\ell$. We add the edge lengths of the in-between knots to the distance that 
\begin{align}
\begin{split}
    d_\calS(\bs_j, \bs_\ell) &= ||\bs_j - V_{i(1)}|| + ||\bs_\ell - V_{i(m)}|| + \sum_{p=1}^{m-1}\norm{V_{i(p)}, V_{ i(p+1)}}
\end{split}
\label{eq::graphdist2}
\end{align}
and we use the shortest path length if there are multiple paths connecting $\bs_j$ and $ \bs_\ell$.
\end{itemize}

An example illustrating the skeleton-based distance is shown in Figure \ref{fig::skeldist}.
Like the shortest path (geodesic) distance that makes a usual knot-edge graph into a metric space, the skeleton-based distance is also a metric on the skeleton graph. 
In the following sections, we will discuss methods to perform regression on space only with the defined metric.

\begin{figure}[ht]
\centering
\includegraphics[height=3cm]{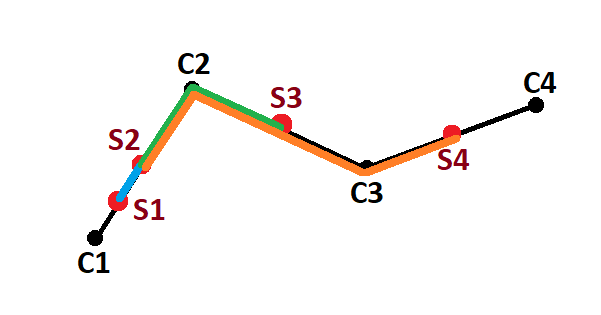}
\caption{Illustration of skeleton-based distance. Let $C_1, C_2, C_3, C_4$ be the knots, and let $S_2,S_3,S_4$ be the mid-point on the edges $E_{12},E_{23},E_{34}$ respectively. Let $S_1$ be the midpoint between $C_1$ and $S_2$ on the edge. Let $d_{ij} = \norm{C_i - C_j}$ denotes the length of the edge $E_{ij}$. $d_\calS(S_1,S_2) = \frac{1}{4} d_{12}$ illustrated by the blue path.  $d_\calS(S_2,S_3) = \frac{1}{2} d_{12} + \frac{1}{2} d_{23}$ illustrated by the green path. $d_\calS(S_2,S_4) = \frac{1}{2} d_{12} + d_{23} + \frac{1}{2} d_{34}$ illustrated by the orange path.}
\label{fig::skeldist}
\end{figure}

%

\begin{remark}
We may view the skeleton-based distance as an approximation of the geodesic distance on the underlying data manifold. Moreover, to make a stronger connection to the manifold structure, it is possible to define edge lengths through local manifold learning techniques that have better approximations to the local manifold structure. However, using more complex local edge weights can pose additional challenges for the data projection step described in the next section and we leave this as a future direction.
\end{remark}

\subsection{Data Projection}
\label{sec:dataProjection}

~~~~For the next step, we project the sample covariates onto the constructed skeleton.
For given covariate $\bx$, let $I_1(\bx), I_2(\bx) \in \{1,\dots,k\}$ be the index of its closest and second closest knots in terms of the Euclidean metric. 
We define the projection function $\Pi(.): \calX \to \calS$ for $\bx \in \calS$ as (illustrated in Figure \ref{fig::skelproject}):
\begin{itemize}
    \item[Case I: ] If $V_{I_1(\bx)}$ and $V_{I_2(\bx)}$ are not connected, $\bx$ is projected onto the closest knot that
    $\Pi(\bx) = V_{I_1(\bx)}$ 
    \item[Case II: ] If $V_{I_1(\bx)}$ and $V_{I_2(\bx)}$ are connected, $\bx$ is projected with the Euclidean metric onto the line passing through $V_{I_1(\bx)}$ and $V_{I_2(\bx)}$ that, let $t = \frac{\left(\bx - V_{I_1(\bx)}\right)^T\cdot \left(V_{I_2(\bx)}-V_{I_1(\bx)}\right)}{\norm{V_{I_2(\bx)}-V_{I_1(\bx)}}^2}$ be the projection proportion,
    \begin{align}
    \Pi(\bx)  = V_{I_1(\bx)} + \left(V_{I_2(\bx)}-V_{I_1(\bx)}\right) \cdot 
    \begin{cases}
    0, \text{ if } t <0\\
    1, \text{ if } t>1\\
    t, \text{ otherwise}
    \end{cases}
    \end{align}
    where we constrain the covariates to be projected onto the closest edge.
\end{itemize}

\begin{figure}[ht]
\centering
\includegraphics[height=3cm]{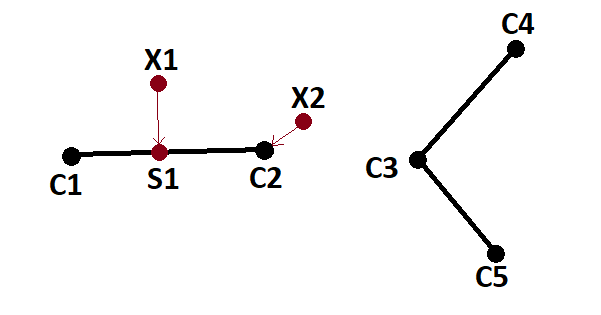}
\caption{Illustration of projection to the skeleton. The skeleton structure is given by the black dots and lines. Data point $X_1$ is projected to $S_1$ on the edge between $C_1$ and $C_2$. Data point $X_2$ is projected to knot $C_2$.}
\label{fig::skelproject}
\end{figure}

Note that with the projection defined above, a non-trivial volume of points can be projected onto the knots of the skeleton graph as belonging to Case I or due to the truncation in Case II. 
This adds complexities to the theoretical analysis of the proposed regression framework and leads to our separate analysis of the different domains of the graph in Section \ref{sec:kernelConsistent}.

\section{Skeleton Nonparametric Regression}
\label{sec::regression}

~~~~Covariates are mapped onto the skeleton after the data projection step and are equipped with the skeleton-based distances. 
In this section, we apply nonparametric
regression techniques to the skeleton graph with projected data points. 
We study three feasible nonparametric approaches: the skeleton-based kernel regression (S-Kernel), the skeleton-based k-nearest-neighbor method (S-kNN), and the linear spline on the skeleton (S-Lspline).
At the end of this section, we discuss the challenges of applying some other nonparametric regression methods in the setting of skeleton graphs.

\subsection{Skeleton Kernel Regression}
\label{sec::skelkernel}
~~~~
We start by adapting kernel smoothing to the skeleton graph. 
Let $\bs_1,\cdots, \bs_n$ be the projections on the skeleton
from $\bx_1,\cdots, \bx_n$, i.e., $\bs_i = \Pi(\bx_i)$. 
With the skeleton-based distances, the skeleton kernel regression makes a prediction at the location
$\bs \in \calS$ as
\begin{align}
    \hat{m}(\bs) = \frac{\sum_{i=1}^N K(d_\calS(\bs_i, \bs)/h) Y_i}{\sum_{j=1}^N K(d_\calS(\bs_j, \bs)/h)},
\end{align}
where $K(\cdot) \geq 0$ is a smoothing kernel such as the Gaussian kernel and $h>0$ is the smoothing bandwidth that controls the amount of smoothing.
In practice, we choose $h$ by cross-validation.
Essentially, the estimator $\hat{m}(\bs)$ is the kernel regression applied to a general metric space (skeleton) rather than the usual Euclidean space.
Notably, the kernel function calculation only depends on the skeleton distances and hence is independent of neither the ambient dimension of the original input nor the intrinsic dimension of the manifold structure.

It should be noted that $\hat{m}(\bs)$ only makes predictions on the skeleton $ \calS$. 
If we are interested in predicting the outcome  at any arbitrary point $\bx\in\calX$, the prediction will be based on the projected point, i.e.,
$
\hat m(\bx) = \hat m\left(\Pi(\bx)\right), 
$
where $\Pi(\bx) \in \calS.$
Because of the above projection property, one can think of the skeleton kernel regression as an estimator
to the following skeleton-projected regression function
\begin{align} \label{eq::projectedRegression}
    m_\calS(\bs) = \EE(\bY|\Pi(\bX) = \bs), \bs \in \calS.
\end{align}
We study the convergence of $\hat{m}(\bs)$ to $m_\calS(\bs)$ in what follows.

\begin{remark}
    Admittedly, the projection of the covariates onto the skeleton as described in Section \ref{sec:dataProjection} introduces the projection error between the true regression function $m(\bx) = \EE(Y|X = \bx)$ and the skeleton-projected regression function. 
    Bounding this projection error involves not only a precise characterization of the underlying manifolds and the data distribution around them but also the physical locations of the skeleton relative to the local manifold structure. 
    Due to such complexity, a theoretical result bounding the projection error under some general conditions requires careful formulation (despite that results are straightforward for particular cases such as having covariates exactly on a 1D circular segment).
    We leave the in-depth analysis of the projection as future work and focus on generalizing the nonparametric regression methods to the skeleton graph in this work.
\end{remark}

\subsubsection{Consistency of S-Kernel Regression}
\label{sec:kernelConsistent}

~~~~
Our analysis assumes that the skeleton is fixed and given and focuses on the estimation of the regression function.
To evaluate the estimation error, we must first impose some concepts of distribution on the skeleton. 
However, due to the covariate projection procedure, the probability measures on the knots and edges are different, and we analyze them separately. On an edge, the domain of the projected regression function varies in one dimension, resulting in a standard univariate problem for estimation.
For the case of knots, a nontrivial region of the covariate space can be projected onto a knot, leading to a nontrivial probability mass at the knot. 


For simplicity, we write
$ K_h(\bs_j, \bs_\ell) \equiv  K(d_\calS(\bs_j, \bs_\ell)/h)$
for $\bs_j, \bs_\ell \in \calS$.
Let $\calB(\bs, h) = \set{\bs' \in \calS: d_\calS(\bs', \bs) < h} $ be the ball on skeleton centered at the point $\bs \in \calS$ with radius $h$.
We can decompose the kernel regression estimator into edge parts and knot parts as
\begin{align} \label{eq::kerneldecomp}
\begin{split}
    &\hat{m}(\bs) = \frac{\sum_{j=1}^n Y_j  K_h(\bs_j, \bs) }{\sum_{j=1}^n K_h(\bs_j, \bs)} \\
    &= \frac{\frac{1}{n}\sum_{j=1}^n Y_j K_h(\bs_j, \bs) I(\bs_j \in \calE)  + \frac{1}{n}\sum_{j=1}^n Y_j K_h(\bs_j, \bs) I(\bs_j \in \calV )}{\frac{1}{n}\sum_{j=1}^n K_h(\bs_j, \bs)I(\bs_j \in \calE ) + \frac{1}{n}\sum_{j=1}^n K_h(\bs_j, \bs) I(\bs_j \in \calV ) } \\
    &= \frac{\frac{1}{n}\sum_{j=1}^n Y_j K_{h}(\bs_j, \bs) I(\bs_j \in \calE \cap  \calB(\bs, h))  + \frac{1}{n}\sum_{j=1}^n Y_j  I(\bs_j =\bs)}{\frac{1}{n}\sum_{j=1}^n K_{h}(\bs_j, \bs)I(\bs_j \in \calE \cap  \calB(\bs, h)) + \frac{1}{n}\sum_{j=1}^n  I(\bs_j =\bs) }
\end{split}
\end{align}
In the last line, we emphasize that the knots and edges in the kernel estimator have a meaningful contribution only within the support of the kernel function. We inspect the different domain cases separately in the following sections.

For the model and assumptions, we let $Y_j =  m_\calS(\bS_j)+U_j, \bS_j \in \calS$, and $\EE(U_j|\bS_j) = 0$ almost surely. 
Let $\sig^2(\bs) = \EE(U_j^2 | \bS_j = \bs)$. 
Let the density on the skeleton edge be defined as the 1-Hausdorff density that
$g(\bs) = \lim_{r\downarrow 0 }\frac{P(\bS\in \calB(\bs, r))}{2r}$.
Note that $g(\bs) = \infty$ if $\bs$ is at a knot point
that has a probability mass.
We consider the following assumptions: 
\begin{itemize}
    \item[\textbf{A1}] $\sig^{2}(\bs)$ is continuous and  uniformly bounded.
    \item[\textbf{A2}] The skeleton edge density function $g(\bs) > 0$ and are bounded and Lipschitz continuous for $\bs \in \calE$.
    \item[\textbf{A3}] $m_\calS(\bs) g(\bs)$ is bounded and Lipschitz continuous for $\bs \in \calE$.
    \item[\textbf{K}] The kernel function has compact support and satisfies $\int K(x) dx = 1$, $\int K^2(x) dx < \infty$, $\int xK(x) d x = 0$,  and $\int x^2 K(x) dx < \infty$ 
\end{itemize}
Conditions A1 and K are common when the domain is the common Euclidean space can be viewed as a generalization of the common assumptions from classical kernel regression analysis \cite{Nadaraya1964, wasserman2006all}.
A2 and A3 are mild conditions that can be sufficiently implied by the boundedness and Lipschitz continuity of the density and regression function in the ambient space along with non-overlapping knots that the area of the orthogonal complements have Lipschitz changes.
We do not assume the second-order smoothness commonly required for kernel regression because requiring higher-order derivative smoothness would necessitate specifying directions on the graph, which may present difficulties in model formulation. 
We include further discussions on formulating the derivatives on the skeleton in Section \ref{sec::otherParametric}.

\subsubsection{Convergence of the Edge Point}

~~~~We first look at an edge point $\bs \in E_{j\ell} \in \calE$. 
In this case, as $n\to \infty, h\to 0$, for sufficiently large $n$, we have $\calB(\bs, h) \subset E_{j\ell}$, and the skeleton distance is the $1$-dimensional Euclidean distance for any point within the support. Therefore, we have a convergence rate similar to the $1$-dimensional kernel regression estimator \citep{Bierens1983, wasserman2006all, Chen2017}.

\begin{thm}[Consistency on Edge Points]
Let $\bs \in \calE$ be a point on the edge. Assume conditions (A1-3) hold for all points in $\calE \cap  \calB(\bs, h)$ and (K) for the kernel function. When $n \to \infty$, $h\rightarrow0$, $nh\rightarrow\infty$, we have
\begin{align}
    \abs{\hat{m}_n(\bs) - m_\calS(\bs)} = O(h) + O_p\bigg(\sqrt{\frac{1}{n h}}\bigg) 
\end{align}
\label{thm::edge}
\end{thm}

We leave the proof in Appendix \ref{sec::contproof}.
Theorem \ref{thm::edge} gives the convergence rate for a point on the edge of the constructed skeleton. 
The convergence rate at the bias is $O(h)$, which is
the usual rate when we only have Lipschitz smoothness (A2) of $m_{\calS}$.
One may be wondering if we can obtain a faster rate such as $O(h^2)$
if we assume higher-order smoothness of $m_{\calS}$. 
While it is possible to obtain a faster rate if we have a higher-order smoothness, we note that this assumption will not be reasonable on the skeleton
because $m_{\bS}(\bs) = \EE(Y|\Pi(X) = \bs)$ is defined via projection.
The region being projected onto $\bs$ is continuously changing
and may not be differentiable due to the boundary of Voronoi cells.
Therefore, the Lipschitz continuity (A2) is reasonable while
higher-order smoothness is not.

\subsubsection{Convergence of the Knots with Nonzero Mass}
~~~~We then look at the knots with nonzero probability mass that $\bs \in \calV$ with $p(\bs) > 0$, where we use $p(\bs)$ to denote the probability mass on a knot. 
This case mainly occurs for knots with degree $1$ on the skeleton graph, when a non-trivial region of points is projected onto such knots. For example, refer to knot C2 in Figure \ref{fig::skelproject}. 

\begin{thm}[Consistency on Knots with Nonzero Mass]
Let $\bs \in \calV$ be a point at a knot and the probability mass at $\bs$ be $P(\Pi_\calS(X)=\bs) \equiv p(\bs) > 0$ and assume $\sig^2(\bs)$ is bounded. Also, assume conditions (A1-3) hold for all points in $\calE \cap  \calB(\bs, h)$ and (K) for the kernel function. When $n\to \infty$, $h\rightarrow0$, we have
\begin{align}
    \abs{\hat{m}(\bs) - m_\calS(\bs)} = O(h)+ O_p\left(\sqrt{\frac{1}{n }}\right)
\end{align}
\label{thm::knotconsistency}
\end{thm}
Theorem \ref{thm::knotconsistency} gives the convergence result for a knot point with a nontrivial mass of the skeleton. 
The bias term $O(h)$ comes from the influence of nearby edge points.
For the stochastic variation part, instead of having the  $O_p\left(\sqrt{\frac{1}{n h}}\right)$ rate as the usual kernel regression and in Theorem \ref{thm::edge}, we have $O_p\left(\sqrt{\frac{1}{n}}\right)$ rate which comes from averaging the observations projected onto the knots. The proof of Theorem \ref{thm::knotconsistency} is provided in Appendix \ref{sec::knotproof}.

\subsubsection{Convergence of the Knots with Zero Mass}
~~~~We now look at a knot point $\bs \in \calV$ with no probability mass that $p(\bs) = 0$. 
This can be the case for a knot with a degree larger than $1$ like knot C3 in Figure \ref{fig::skelproject}. Since we define edge sets excluding the knots, there will be no density as well as no probability mass at $\bs$. Note that, with some reformulation, degree $2$ knots can be parametrized together with the two connected edges and, under the appropriate assumptions, Theorem \ref{thm::edge} applies, giving consistency estimation with $O(h) + O_p\left(\sqrt{\frac{1}{n h}}\right) $ rate. However, density cannot be extended directly to knots with a degree larger than $2$, but the kernel estimator still converges to some limits as presented in the Proposition below.
\begin{prop}
\label{prop::zeroknot}
Let $\bs \in \calV$ be a point at a knot
such that the probability mass at $\bs$ be $P(\Pi_\calS(X)=\bs) \equiv p(\bs) = 0$. 
Assume conditions (A1-3) hold for all points in $\calE \cap  \calB(\bs, h)$ and (K) for the kernel function.
Let $\calI$ collect the indexes of edges with one knot being $\bs$.
For $\ell \in \calI$ and edge $E_\ell$ connects $\bs$ and $V_\ell$,
let $g_\ell(t) = g((1-t)\bs + t V_\ell)$ and $g_\ell(0) = \lim_{x\downarrow 0} g_{\ell}(x)$.
Let  $m_\ell(t) = m_\calS( (1-t)\bs + t V_\ell)$ and $m_\ell(0) = \lim_{t \downarrow 0} m_\ell(t)  $. 
When $n\to \infty$, $h\rightarrow0$, and $nh\rightarrow\infty$, we have
\begin{align}
    \hat{m}(\bs) 
    &= \frac{ \sum_{\ell \in \calI}   m_\ell(0) g_\ell(0)  }{ \sum_{\ell \in \calI}   g_\ell(0) } + O(h) + O_p\left(\sqrt{\frac{1}{n h}}\right).
\end{align}
\end{prop}

Proposition \ref{prop::zeroknot} shows that, under proper conditions, the skeleton kernel estimator on a zero-mass knot converges to the weighted average of the limiting regression values of the connected edges, and the convergence rate is the same as the edge points shown in Theorem \ref{thm::edge}.
The proof is included in Appendix \ref{sec::zeroknotproof}.

\begin{remark}
The domain $\calS$ of the regression function can be seen as bounded, and hence the boundary bias issue can arise. 
The true manifold structure's boundary can be different from the boundary of the skeleton graph, making the consideration of the boundary more complicated.
However, the boundary of the skeleton is the set of degree $1$ knots, and, under our formulation, knots have discrete measures, so the consideration of boundary bias may not be necessary for the proposed formulation.
However, some boundary corrections can potentially improve the empirical performance and we leave it for future research.
\end{remark}


\subsection{Skeleton kNN regression}
\label{sec::SkNN}

~~~~
The $k$-Nearest Neighbor (kNN) method
can be easily applied to the skeleton
using the distance on the skeleton.
For a given point on the skeleton at $\bs \in \calS$, 
we define the distance to the k-th nearest observation
on the skeleton as 
\begin{align}
    R_k(\bs) = \min\left\{r>0: \sum_{i=1}^n I(d_\calS(\bs_i, \bs)\leq r)\geq k\right\}.
\end{align}
Note that it is possible to have multiple
observations being the $k$-th nearest observation
due to observations being projected to the vertices.
In this case, we can either randomly choose
from them or consider all of them.
Here we include all of them in the calculation.
The skeleton-based $k$NN regression (S-kNN) 
predicts the value of outcome at $\bs$ as
\begin{align}
    \hat{m}_{SkNN}(\bs) = \frac{\sum_{i=1}^k Y_{i} I(d_\calS(\bs_i, \bs)\leq R_k(\bs))}{\sum_{j=1}^k I(d_\calS(\bs_j, \bs)\leq R_k(\bs))}.
\end{align}

Different from the usual kNN regressor with the covariates $\bx_1, \dots, \bx_n$, which selects neighbors through Euclidean distance in the ambient space, the S-kNN regressor chooses neighbors with skeleton-based distances after projection onto the skeleton graph.
Measuring proximity with the skeleton can improve the regression performance when the dimension of the covariates is large, which we empirically show in Section \ref{sec::simulation}.


\begin{remark}
It is well known that the usual $k_n$NN regressor can be consistent if we let $k_n$ grow as a function of the sample size $n$, and under appropriate assumptions, \cite{DistributionFreeThoery} give the convergence rate of the $k_n$NN estimate $m_n$ to the true function $m$ as 
\begin{align*}
\mathbb{E}\left\|m_n-m\right\|^2 \leq \frac{\sigma^2}{k_n}+c_1 \cdot C^2\left(\frac{k_n}{n}\right)^{2 /d}
\end{align*}
Later, \cite{Kpotufe2011} has shown that the convergence rate of $k$NN regressor depends on the intrinsic dimension. 
We conjecture that a similar result with $d=1$ rate holds for the skeleton $k$NN regression at an edge point, but leave the proof for future work.
\end{remark}

\subsection{Linear Spline Regression on Skeleton}

\label{sec::lspline}

~~~~In this section, we propose a skeleton-based linear spline model (S-Lspline) for regression estimation.
By construction, this approach results in a continuous model across the graph.
Moreover, we show that the skeleton-based linear spline corresponds to an elegant parametric regression model on the skeleton.
As the skeleton $\calS$
can be decomposed into the edge component $\mathcal{E}$
and the knot component $\mathcal{V}$, the linear spline regression on the skeleton can be written as the following constrained model:

\begin{equation}
\begin{aligned}
f: \calS\,\,\rightarrow \,\,\mathbb{R} \ \
\mbox{  such that }&\mbox{1. $f(x)$ is linear on $x\in\mathcal{E}$,}\\
&\mbox{2. $f(x)$ is continuous at $x\in \mathcal{V}$.}
\end{aligned}
\label{eq::LS_original}
\end{equation}
While solving the above constrained problem may not be easy,
we have the following elegant representer theorem
showing that a linear spline on the skeleton can be uniquely characterized by the values on each knot. 

\begin{theorem}[Linear spline representer theorem]
\label{thm::spline}
Any function satisfying equation \eqref{eq::LS_original}
can be characterized by $\{f(v): v\in\mathcal{V}\}$
and for $x\in\mathcal{E}$, $f(x)$ is linear interpolation
between the values on the two knots on the edge that $x$ belongs to.
\end{theorem}

\begin{proof}
Let $f$ be a function satisfying equation \eqref{eq::LS_original}. 
By construct, $f$ is linear for $x\in\calE$ and is continuous at
$x\in \calV$.
Let $V_j$ and $V_\ell$ be two knots that share an edge
and let $E_{j\ell} =\{x = t V_j + (1-t) V_\ell: t\in(0,1)\}$
be the shared edge segment. 
For any $x\in\calE$, there exists a pair $(V_j,V_\ell)$
such that $x\in E_{j\ell}$.
Because $f$ is linear in $E_{j\ell}$, 
$f$ can be uniquely characterized by the pairs $(f(e_1), e_1), (f(e_2), e_2)$ for two distinct points $e_1,e_2 \in \bar E_{j\ell}$,
where $\bar E_{j\ell} =\{x = t V_j + (1-t) V_\ell: t\in[0,1]\}$
is the closure of $E_{j\ell}.$
Thus, we can pick $e_1 = V_j$ and $e_2 =V_\ell$, which implies
that $f$ on the segment $E_{j\ell}$ is parameterized 
by $f(v_j)$ and $f(V_\ell)$, the values on the two knots. 

By applying this procedure to every edge segment,
we conclude that 
any function satisfying the first condition in \eqref{eq::LS_original}
can be characterized by the values of the knots. 
The second condition in \eqref{eq::LS_original} will require
that every knot has one consistent value. 
As a result, any function $f$ satisfying \eqref{eq::LS_original}
can be uniquely characterized by the values on the knot $\{f(x): x\in \calV\}$
and $f(x)$ will be a linear interpolation when $x\in \calE$. 

\end{proof}



Using Theorem~\ref{thm::spline},
we only need to determine the values on the knots. 
Let $\bbeta \in \mathbb{R}^k$
be the values of the skeleton linear spline model on each knot
with $k = \abs{\calV}$ being the number of knots.
As is argued previously, the spline model is parameterized by $\bbeta$,
so we only need to estimate $\bbeta$ from the data. 
Given $\bbeta$, the predicted value of each $\by_i$ is a linear interpolation
depending on the projected location of each $\bx_i$.

To derive an analytic form of $\by_i$,
we introduce a transformed covariate matrix 
$\bZ = (\bz_1, \dots, \bz_n)^T \in \bR^{n\times k}$ as follows:
\begin{enumerate}
    \item If $\bx_i$ is projected onto a vertex that $\bs_i = V_j$ for some $j$, then
    \begin{align*}
        \bz_{ij'}=I(j'=j).
    \end{align*}
    \item If $\bx_i$ is projected onto an edge between knots $V_j$ and $V_\ell$, then 
    \begin{align*}
    \bz_{ij} = \frac{||\bs_i - V_j||}{||V_j - V_\ell||},
    \quad \bz_{i\ell} = \frac{||\bs_i - V_\ell||}{||V_j - V_\ell||}, \quad \text{ and } \bz_{ij'}=0 \text{ for } j'\neq j,\ell.
    \end{align*}
    
\end{enumerate}
With the above feature transform, the predicted value of $\by_i$ by the S-Lspline model is
\begin{align}
    \hat \by_i = \bbeta^T \bz_i.
\end{align}

To see this,
if $\bx_i$ is projected onto a vertex that $\bs_i = V_j$ for some $j$, the linear model with transformed covariates gives $\bbeta^T\bz_i = \bbeta_j$, the predicted value on vertex $V_j$.
In the case where $\bx_i$ is projected onto an edge between knots $V_j$ and $V_\ell$,
let $\bbeta_j$ and $\bbeta_\ell$ be the corresponding predicted values at $V_j$ and $V_\ell$, and the linear interpolation between $\bbeta_\ell$ and $\bbeta_j$ at $\bs_i$ can be written as
\begin{align*}
    \bbeta_j + \frac{||\bs_i - V_j||}{||V_j - V_\ell||} \cdot \left(\bbeta_\ell - \bbeta_j\right) = \frac{||\bs_i - V_\ell||}{||V_j - V_\ell||} \cdot \bbeta_j + \frac{||\bs_i - V_j||}{||V_j - V_\ell||} \cdot \bbeta_\ell = \bbeta^T\bz_i.
\end{align*}
To estimate $\bbeta$, we can apply the least squares procedure to get:
\begin{align*}
\hat \bbeta &= {\sf argmin}_{\bbeta} \sum_{i=1}^n (\by_i-\hat \by_i)^2\\
& = {\sf argmin}_{\bbeta} \sum_{i=1}^n ( \by_i- \bbeta^T \bz_i)^2.
\end{align*}
So it becomes a linear regression model
and the solution can be elegantly written as
\begin{align*}
    \hat{\bbeta} = (\bZ^T \bZ)^{-1} \bZ \by.
\end{align*}
Note that in a sense, the above procedure can be viewed
as placing a linear model 
\begin{align*}
    \mathbb{E}(\by|\bX) = \bbeta^T \bZ,
\end{align*}
where $\bZ$ is a transformed covariate matrix from $\bX$.

Note that the S-Lspline model with the graph-transformed covariates does not include an intercept.

\begin{remark}
An alternative justification of the value-on-knots parameterization is to calculate the degree of freedom. On each graph, the sum of the vertex degrees is twice the number of edges since each edge is counted from both ends. Let $e$ be the number of edges in the graph, let $v$ be the number of vertices, and let $r$ be the sum of all the vertex degrees, we have $r = 2e$. For the S-Lspline model, we construct a linear model with $2$ free parameters for each edge, and thus without any constraints, the total number of degrees of freedom is $2e$. For each vertex $V_i$ with degree $r_i$, the continuity constraint imposes $r_i - 1$ equations, and as a result, the continuity constraints consume a total of $\sum_{i=1}^v (r_i-1) = r - v$ degrees of freedom. Combining it, we have $2e - (r -v) = v$ degrees of freedom, which matches the degrees of freedom given by the parametrization of values on the knots.
\end{remark}


\subsubsection{Regularized Linear Spline Method}
\label{sec:penalLspline}
~~~~Given the formulation of the S-Lspline as a linear regression with transformed data, it is natural to incorporate penalization with this method. In this section, we introduce penalization into the S-Lspline method by making connections to the literature about regularization on graphs, with a particular focus on graph Laplacian smoothing by \cite{GraphLaplacianSmoothing} and graph trend filtering by \cite{Wang2016}.

Let $B$ be the (unoriented) incidence matrix of the skeleton graph that 
\begin{align*}
    B_{ij} =
    \begin{cases}
    1 & \text{if vertex $v_i$ is incident with edge $e_j$ },\\
    0 & \text{otherwise.}
\end{cases}
\end{align*}
for $i = 1, \dots, k$ where $k$ is the number of knots and $j = 1, \dots, m$ where $m$ is the number of edges in the skeleton graph.
Let $L$ denote the Laplacian matrix that $L = D - A = B B^T$ where $D$ is the degree matrix and $A$ is the adjacency matrix of the skeleton graph.
The $q$-th order trend filtering matrix, for $q \in \set{0,1,2,\dots}$,  is defined as
\begin{align*}
    \Delta^{(q+1)} = \begin{cases}
    L^{\frac{q+1}{2}} & \text{for odd }q,\\
    B L^{q/2} & \text{for even }q.
\end{cases}
\end{align*}

The $q$-th order Laplacian smoothing can be taken as the $L_2$ penalty with the trend filtering matrix, and we have the regularized problem to be 
\begin{align*}
     {\sf argmin}_{\bbeta} \norm{\bY - \bZ \bbeta}_2^2 + \lam \norm{\Delta^{(k+1)} \bbeta}_2 
\end{align*}
where $\norm{\Delta^{(k+1)} \bbeta}_2  = \bbeta^T L^{k+1} \bbeta$ for Laplacian matrix $L$, and $\bZ$ the transformed covariate matrix from $\bX$.
This can be solved as a Generalized Ridge problem\footnote{Generally if the penalty matrix $L^{k+1} $ is positive definite, the generalized penalty is a non-degenerated quadratic form in $\beta$, and hence strictly convex. The analytical solution is then
\begin{align*}
    \hat \beta = \brac{X^T X + \lam L^{k+1}}^{-1} ,\brac{X^T Y}
\end{align*}
However, the Laplacian matrix is only positive semi-definite, and therefore the loss function need not be strictly convex. Some work suggests adding $\norm{\bbeta}_2^2$ as an additional penalty to address this, but we do not implement that to be consistent with the trend filtering penalization.}. 

The Trend Filtering regularization similarly applies a $L_1$ penalty and the problem becomes
\begin{align*}
     {\sf argmin}_{\bbeta} \norm{\bY - \bZ \bbeta}_2^2 + \lam \norm{\Delta^{(k+1)} \bbeta}_1.
\end{align*}
We follow \cite{SolutionPathLasso} to get the solution to the generalized Lasso problem. We include the algorithm in Appendix \ref{sec:lassoSolution} for completeness.
Empirically, we observe that penalization does not improve the regression results of the S-Lspline model (see Appendix \ref{sec:penalSplineSims}). 
To account for this, note that the skeleton graph is a summarizing presentation of the data with a concise structure, and the S-Lspline method assumes a simple piecewise linear model on the skeleton which inherits the simple geometric structure and is not a complex model in nature, and hence adding penalization does not improve the performance of this method.

\subsection{Challenges of Other Nonparametric Regression}
\label{sec::otherParametric}
~~~~In this section, we discuss the challenges when applying other nonparametric regression methods to the skeleton. 
Particularly, the skeleton graph is only equipped with a metric and 
does not have a well-defined inner product or orientation,
which makes many conventional approaches not directly applicable.

\subsubsection{Local polynomial regression}

~~~~Local polynomial regression \cite{fan2018local}
is a common generalization of the kernel regression that
tries to improve the kernel regression estimator by using higher-order polynomials as local approximations to the regression function. 
In the Euclidean space, a $p$-th order local polynomial regression aims to choose $\beta(\bx)$ via minimizing
\begin{align}
    \sum_{i=1}^n \sbrac{Y_i - \sum_{j = 0}^p \beta_j (\bx_i - \bx)^j}^2 K \brac{\frac{\bx_i - \bx}{h}}
\end{align}
and predict $m(\bx)$ via $\hat \beta(\bx)$, the first element in the minimizer.
Note that when $p=0 $, one can show that this is equivalent to the kernel regression.

Unfortunately, the local polynomial regression cannot be easily adapted to the skeleton because the polynomial $(\bx_i - \bx)^j$ requires a well-defined orientation, which is ill-defined at a knot (vertex). 
Directly replacing $(\bs_i -\bs)$ with the distance $d_\calS (\bs_i -\bs)$ will make all the polynomials to be non-negative, which will be problematic for odd orders. 
Unless in some special skeletons such as a single chain structure,
the local polynomial regression cannot be directly applied.


\subsubsection{Higher-Order Spline}

~~~~In Section \ref{sec::lspline}, we introduce the linear spline model.
One may be curious about the possibility of using a higher-order spline (enforcing higher-order smoothness on knots; see, e.g., Chapter 5.5 of \cite{wasserman2006all}).
Unfortunately, the higher-order spline is generally not applicable to the skeleton
because a higher-order spline requires derivatives and the concept of a derivative may be ill-defined on a knot because of the lack of orientation.
To see this, consider a knot with three edges connecting to it. 
There is no simple definition of derivative at this knot unless we specify
the orientation of these three edges.

One possible remedy is to introduce an orientation for every edge.
This could be done by ordering the knots first and, for every edge,
the orientation is always from a lower index vertex to the higher index vertex. 
With this orientation, it is possible to create a higher-order spline
on the skeleton
but the result will depend on the orientation we choose.


Even with edge directions provided and the derivatives on the skeleton defined, higher-order spline on the skeleton can be prone to overfitting.
Classical spline methods use degree $p+1$ polynomial functions to achieve continuity at $p$-th order derivative. For example, univariate cubic splines use polynomials up to degree $3$ to ensure the second-order smoothness of the regression function at each knot. 
However, on a graph, degree $p+1$ polynomial functions may fail to achieve continuity at $p$-th order derivative, and on complete graphs, which is the worst case, $2p+1$ degree polynomials are needed instead.

\subsubsection{Smoothing Spline}
~~~~Smoothing spline \cite{wang2011smoothing, wahba1975smoothing} is another popular approach
for curve-fitting that attempts to 
find a smooth curve that minimizes the square loss in the prediction
with a penalty on the curvature (second or higher-order derivatives).

The major difficulty of this method is that the concept of a \emph{smooth} function
is ill-defined at a knot even if we have a well-defined orientation.
In fact, the `linear function' is not well-defined in general on a skeleton's knot.
To see this, consider a knot $V_0$ with three edges $e_1,e_2,e_3$ connecting to $V_1,V_2,V_3$, respectively.
Suppose we have a linear function $f_0$ and
$f_0$ is linearly increasing on paths $V_1-V_0-V_2$
and $V_1-V_0- V_3$. 
However, on the path $V_2-V_0-V_3$, the function $f_0$ will be 
decreasing ($V_2-V_0$) and then increasing ($V_0-V_3$), leading to a non-smooth structure.



\subsubsection{Orthonormal Basis and Tree}

~~~~Orthonormal basis approach (see, e.g., Chapter 8 of \cite{wasserman2006all}) uses a set of orthonormal basis functions to approximate the regression function. 
In general, it is unclear how to find a good orthonormal basis 
for a skeleton unless the skeleton is simply a circle or a chain. 

Having said that, it is possible to construct an orthonormal basis
borrowing the idea from wavelets \citep{torrence1998practical}.
The key idea is that the skeleton is a measurable set
that we can measure its (one-dimensional) volume. 
Thus, we can partition the skeleton $\calS$ into two equal-volume sets $A_1, A_2$. Note that the resulting sets $A_1, A_2$ are not necessarily skeletons because we may cut an edge into two pieces. 
For each set $A_j$, we can further partition it again into equal volume sets $A_{j,1}, A_{j,2}$. 
And we can repeat this dyadic procedure to create many equal-volume subsets.
We then define a basis as follows:
\begin{align*}
f_0(s) &= 1,\\
f_1(s) & =  I\brac{s\in A_1} -  I\brac{s\in A_2}\\
f_2(s) & =  I(s\in A_{1,1}) -  I(s\in A_{1,2})\\
f_3(s) & =  I(s\in A_{2,1}) -  I(s\in A_{2,2})\\
\vdots
\end{align*}
After normalization, this set of functions forms
an orthonormal basis.
With this basis, it is possible to fit an orthonormal basis on the skeleton. 
However, the above construction creates the partition arbitrarily. 
The fitting result depends 
on the particular partition we use to generate the basis
and it is unclear how to pick a reasonable partition in practice.

The regression tree \cite{breiman2017classification, loh2014fifty} is a popular idea in nonparametric regression that fits the data via creating a tree of partitioning the whole sample space 
whose leaves represent a subset of the sample space
and predicts the response using a single parameter at each leaf (region). 
This idea could be applied to the skeleton using a similar
procedure as the construction of an orthonormal basis
that we keep splitting a region into two subsets (but
we do not require the two subsets to be of equal size). 
However, unlike the usual regression tree (in Euclidean space) that
the split of two regions is often at a threshold at one coordinate,
the split of a skeleton may not be easily represented
as the skeleton is just a connected subregion of Euclidean space.
Therefore, similar to the orthonormal basis, 
regression tree may be used in skeleton regression, but
there is no simple and principled way to create
a good partition.


\section{Simulations}	\label{sec::simulation}
~~~~In this section, we use simulated data to evaluate the performance of the proposed skeleton regression framework. 
\footnote{R implementation of the proposed skeleton regression methods can be accessed at \url{https://github.com/JerryBubble/skeletonMethods} and Python implementation can be accessed at \url{https://pypi.org/project/skeleton-methods/}.}
We first demonstrate an example with the intrinsic domain composed of several disconnected components, which we call the Yinyang data (Section \ref{sec::Yinyang}). Then, we add noisy observations to the Yinyang data (Section \ref{sec::NoisyYinyang}) to show the effectiveness of our method in handling noisy data points. Moreover, we present an example where the domain is a continuous manifold with a Swiss roll shape (Section \ref{sec::SwissRoll}). In all the simulations in this section, there are random perturbations in the intrinsic dimensions, and we add random Gaussian variables as covariates to increase the ambient dimension.

\subsection{Analysis Procedure}
\label{sec:analysisProcedure}
~~~~We apply the following analysis procedure for all the simulations in this section.
We randomly generate the dataset for $100$ times, and, on each dataset, we use $5$-fold cross-validation to calculate the sum of squared errors (SSE) as the performance assessment. 
We use the skeleton construction method described in Section \ref{sec::skeletoncons} to construct skeletons with varying numbers of knots on each training set. 
In this section, we present results where the construction procedure cuts the skeleton into a given number of disjoint components according to the Voronoi Density weights (Section \ref{sec::skeletoncons}).
We also empirically tested using different cuts to get skeleton structures with different numbers of disjoint components under the same number of knots and noticed little change in the squared error performance (see Appendix \ref{sec::extraSim}). 


We evaluate the skeleton-based nonparametric regressors introduced in Section \ref{sec::regression}: skeleton kernel regression (S-Kernel), $k$-NN regressor using skeleton-based distance (S-kNN), and the skeleton spline model(S-Lspline). For S-Kernel and S-kNN methods, to simplify the calculation, we only compute the skeleton-based distances between points in the same or neighboring Voronoi cells. That is the skeleton-based distance between a pair of points is calculated when they share at least one knot from their respective set of two closest knots.
For the S-Lspline method, we include the results without additional penalization in this section. We compare the empirical performance of the S-Lspline method with various penalizations discussed in Section \ref{sec:penalLspline} on the simulated datasets and present the results in Appendix \ref{sec:penalSplineSims}, and we observe that incorporating penalization terms does not improve the empirical performance of the S-Lspline method.

For comparisons, we apply the classical k-nearest-neighbors regression based on Euclidean distances (kNN).
For penalization regression methods, we test Lasso and Ridge regression.
Among the recent manifold and local regression methods, we include the Spectral Series approach \citep{Lee2016} with the radial kernel (SpecSeries) for its superior performance\footnote{The Spectral Series approach demonstrates similar empirical performance as the kernel machine learning methods with regularization in RKHS as in \cite{Lee2016}.} and readily available R implementation \footnote{\url{https://projecteuclid.org/journals/supplementalcontent/10.1214/16-EJS1112/supzip_1.zip}}. 
For kernel machine learning approaches, we include the Divide-and-Conquer Kernel Ridge Regression (Fast-KRR) method as in \cite{Yuchen2013}. For Fast-KRR, we set the penalization hyperparameter $\lambda = 1/n$ and set the number of random partitions $m = \sqrt{n}$ where $n$ is the size of the training sample, and use the radial kernel where the best bandwidth $\sigma$ is given by grid search. 

For a neural network approach, we implemented a simple Multi-Layer Perceptron (MLP) Autoencoder using \texttt{MLPRegressor} networks\footnote{\url{https://scikit-learn.org/stable/modules/generated/sklearn.neural_network.MLPRegressor.html}}  with a 3-layer architecture that the encoder maps directly from input dimension $D$ to hidden dimension $h$, and the decoder maps directly from $h$ back to $D$, and use the Ridge regression model to predict the response $Y$ from the learned embeddings with penalization parameter $\lambda$ that $\hat{Y} = \text{Ridge}(\text{Encoder}(X); \lambda)$, and explore different combinations of hidden dimension and penalization parameter configurations (see Appendix \ref{sec:MLPAE} for more details and analysis with an autoencoder with more layers).

For the simulations presented in this section, we add random Gaussian variables to create settings with a large ambient dimension of $1000$ to demonstrate that the proposed skeleton regression framework is robust under such challenging scenarios. 
For completeness, we also include the simulation results on low-dimensional data settings in Appendix \ref{sec::extraSim}, and the skeleton-based regression methods also show competitive performance in such settings.

\subsection{Yinyang Data}
\label{sec::Yinyang}

~~~~The covariate space of Yinyang data is intrinsically composed of $5$ disjoint structures of different geometric shapes and different sizes: a large ring of $2000$ points, two clumps each with $400$ points (generated with the \texttt{shapes.two.moon} function with default parameters in the \texttt{clusterSim} library in R \citep{clusterSim}), and two 2-dimensional Gaussian clusters each with $200$ points (Figure \ref{fig::YinyangData} left). Together there are a total of $3200$ observations. 
Note that the intrinsic structures of the components are curves and points, and, with perturbations, the generated covariates do not lay exactly on the corresponding manifold structures.
The responses are generated from a trigonometric function on the ring and constant functions on the other structures with random Gaussian error(Figure \ref{fig::YinyangData} right). That is, let $\epsilon \sim N(0,0.01)$ and let $\theta$ be the angle of the covariates, then
\begin{align*}
\label{eq::YinyangResponse}
    Y = \epsilon +\begin{cases}
  \sin(\theta*4) + 1.5& \text{for points on the outer ring} \\
  0& \text{for points on the bottom-right Gaussian cluster}\\
  1& \text{for points on the right clump}\\
  2& \text{for points on the left clump}\\
  3& \text{for points on the upper-left Gaussian cluster}
    \end{cases}
\end{align*}
To make the task more challenging with the presence of noisy variables, we add independent and identically distributed random $N(0,0.01)$ variables to the generated covariates. In this section, we increase the dimension of the covariates to a total of $1000$ with those added Gaussian variables.


For the Yinyang data, we cut the skeleton into $5$ disjoint components during the skeleton construction process according to the Voronoi Density weights.
We take the median, 5th percentile, and 95th percentile of the 5-fold cross-validation Sum of Squared Errors (SSEs) for each parameter setting of each method on the 100 datasets. We present the smallest median SSE for each method in Table \ref{table:Yinyangd1000} along with the corresponding best parameter setting.

\begin{figure}
\centering
    \begin{subfigure}[t]{0.3\textwidth}
        \centering
        \includegraphics[width=\linewidth]{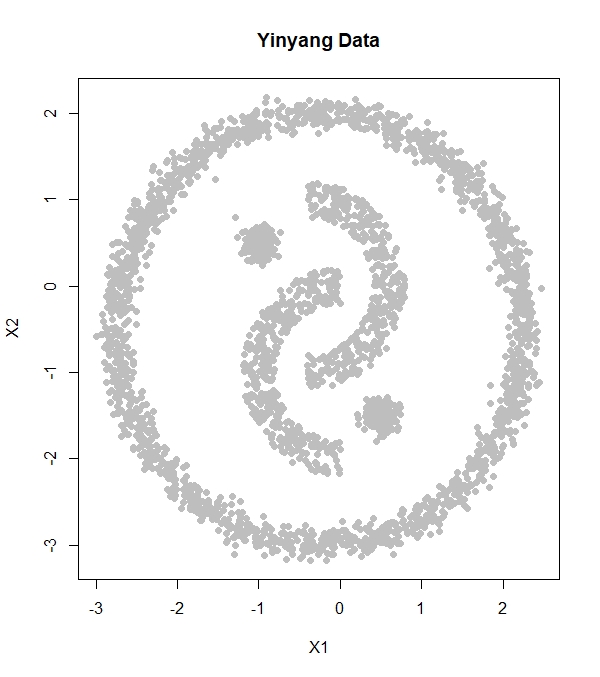} 
    \end{subfigure}
        \begin{subfigure}[t]{0.3\textwidth}
        \centering
        \includegraphics[width=\linewidth]{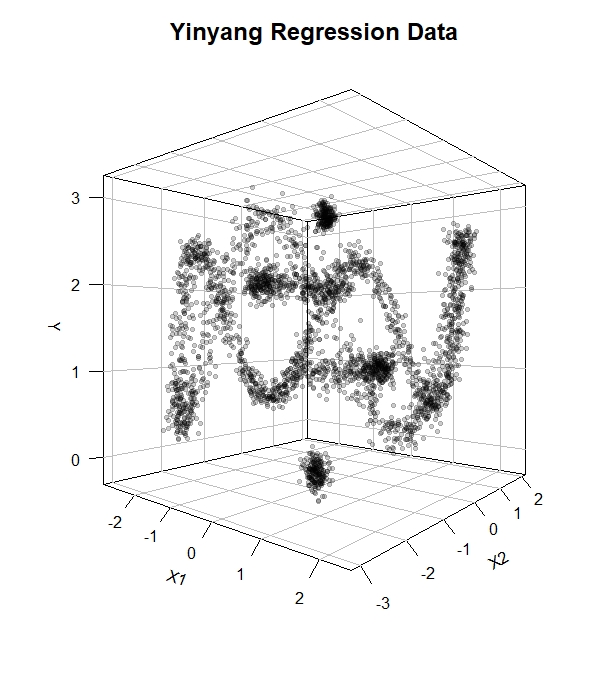}
    \end{subfigure}
    \vspace{-1em}
\captionof{figure}{Yinyang Regression Data}
\label{fig::YinyangData}
\vspace{2em}

\begin{tabular}{c l c l} 
 \hline
 Method & Median SSE ($5$\%, $95$\%) & nknots & Parameter \\ [1ex] 
 \hline
 kNN &204.5 (192.3, 221.9) & - & neighbor=18 \\ 
 Ridge & 2127.0 (2100.2, 2155.2) & & $\lambda  = 7.94$\\
 Lasso & 1556.8 (1515.4, 1607.9) & & $\lambda = 0.0126$\\
 SpecSeries & 1506.4 (1469.1,1555.6) & - &bandwidth = 2\\
 Fast-KRR & 2404.0 (2370.0, 2440.2) & - &$\sigma$ = 0.1\\
 3-layer MLP AE &  2737.9 (2692.8, 2777.6) & - & $h = 50, \lambda = 100$\\
 S-Kernel & 91.6 (81.6, 103.5) & 38 & bandwidth = 4 $r_{hns}$ \\
 S-kNN & 92.7  (84.5, 102.8) & 38 & neighbor = 36 \\
 S-Lspline & 94.4 (87.7, 103.2) & 38 & -  \\[1ex] 
 \hline
\end{tabular}
\captionof{table}{Regression results on Yinyang $d=1000$ data. The smallest medium 5-fold cross-validation SSE from each method is listed with the corresponding parameters used. The $5$th percentile and $95$th percentile of the SSEs from the given parameter settings are reported in brackets.}
\label{table:Yinyangd1000}
\vspace{2em}

\centering
\includegraphics[width=\textwidth]{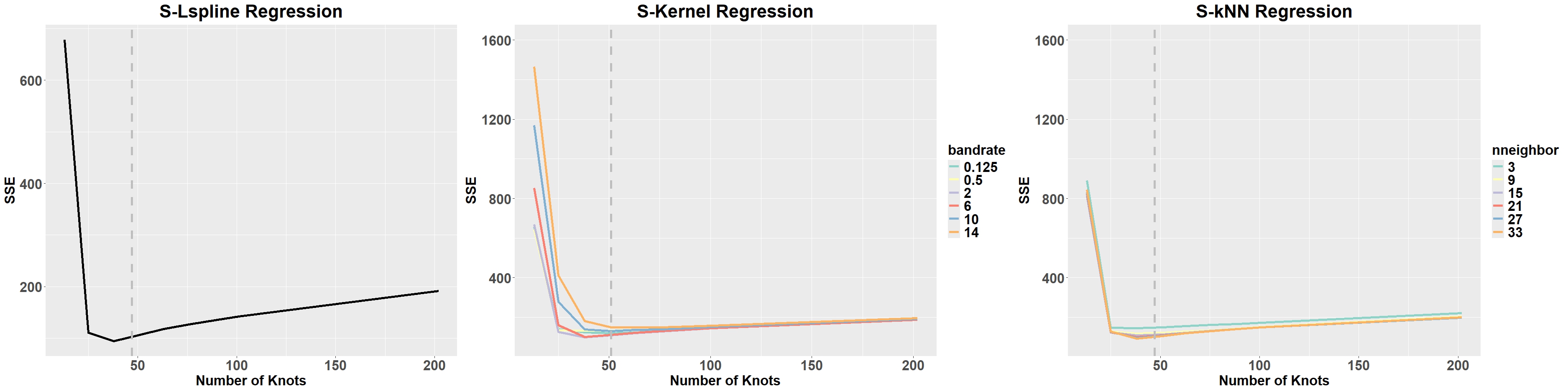}
\captionof{figure}{Yinyang $d = 1000$ data regression results with varying number of knots. The median SSE across the $100$ simulated datasets with each given parameter setting is plotted.}
\label{fig::Yinyangd1000Numknots}

\end{figure}

We observed that all the skeleton-based methods (S-Kernel, S-kNN, and S-Lspline) perform better than the standard kNN in this setting. That is, the skeleton better captures the geometric structures of the data and improves the downstream regression performance.
The three skeleton-based methods have similar performance on this simulated Yinyang data, but S-Lspline method can be preferred in this case in terms of computation as it does not require the skeleton-based distance computations.
The spectral method SpecSeries, the kernel machine learning approach Fast-KRR, and the 3-layer MLP autoencoder regressor all perform worse than the classical kNN. The underlying data structure being comprised of multiple disconnected components in this case can diminish the power of such manifold learning methods.
Ridge and Lasso regression, despite the regularization effect, resulted in relatively high SSEs. Therefore, the skeleton regression framework has the empirical advantage when dealing with covariates that lie around manifold structures.

In Figure \ref{fig::Yinyangd1000Numknots}, we present the median SSE of the S-Lspline, S-Kernel, and S-kNN methods on skeletons with various numbers of knots. The vertical dashed line indicates $[\sqrt n] = 51$ knots as suggested by the empirical rule, where $n$ is the training sample size.
The empirical rule seems to produce satisfactory results in this simulation study, roughly identifying the ``elbow'' position, but it's advised to use cross-validation for fine-tuning in practice.




\subsection{Noisy Yinyang Data}
\label{sec::NoisyYinyang}
~~~~To show the robustness of the proposed skeleton-based regression methods, we add $800$ noisy observations to the Yinyang data in Section \ref{sec::Yinyang} ($20\%$ of a total of $4000$ observations). 
The first two dimensions of the noisy covariates are uniformly sampled from the $2$-dimensional square $[-3.5,3.5]\times [-3.5,3.5]$ and independent random normal $N(0,0.01)$ variables are added to make the covariates $1000$-dimensional in total. 
The responses of the noisy points are set as $1.5 + \epsilon$ with $\epsilon \sim N(0,0.01)$, while the responses on the Yinyang covariates are generated the same as in Equation \ref{eq::YinyangResponse}. 
The first two dimensions of the Noisy Yinyang covariates are plotted in Figure \ref{fig::NoiseYinyangData} left and the $Y$ values against the first two dimensions of the covariates are illustrated in Figure \ref{fig::NoiseYinyangData} right.

\begin{figure}
\centering
    \begin{subfigure}[t]{0.3\textwidth}
        \centering
        \includegraphics[width=\linewidth]{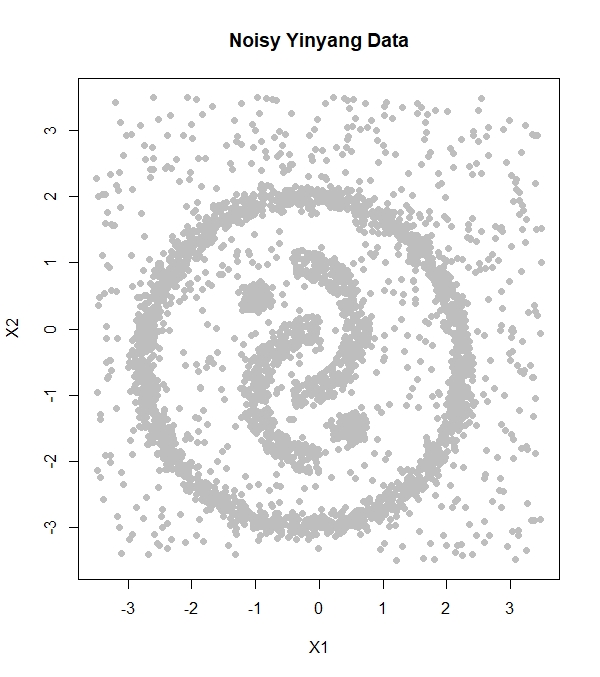} 
    \end{subfigure}
        \begin{subfigure}[t]{0.3\textwidth}
        \centering
        \includegraphics[width=\linewidth]{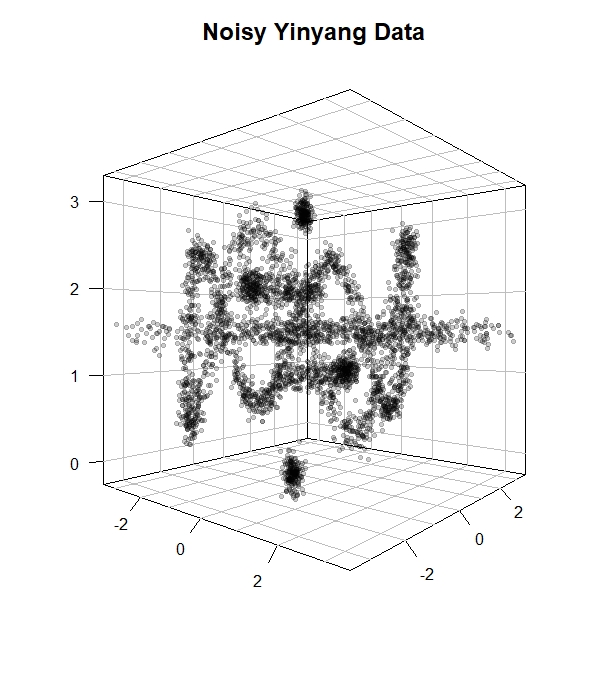}
    \end{subfigure}
\vspace{-1em}
\captionof{figure}{Noisy Yinyang Regression Data}
\label{fig::NoiseYinyangData}
\vspace{2em}
\centering
\begin{tabular}{c l c l} 
 \hline
 Method & Median SSE ($5$\%, $95$\%)  & Number of knots & Parameter \\ [1ex] 
 \hline
 kNN & 440.8 (420.4, 463.0) & -  & neighbor=18 \\ 
 Ridge & 2139.1 (2102.6, 2171.1) &- & $\lambda  = 6.31$\\
 Lasso &  2029.2 (1988.7, 2071.0) &- & $\lambda = 0.02$\\
 SpecSeries & 1532.0 (1490.7, 1563.2) & - & bandwidth = $2$\\
 Fast-KRR & 2584.6 (2556.3, 2624.5) &- & $\sigma$ =0.1\\
 3-layer MLP AE &  2785.1 (2739.2, 2821.6) & - & $h = 20, \lambda = 100$\\
 S-Kernel & 313.5 (293.2, 331.1) & 28 & bandwidth = 2 $r_{hns}$ \\
 S-kNN & 352.9 (332.4, 376.7) & 28 & neighbor = 15 \\
 S-Lspline & 376.5 (354.3, 399.2) & 57 &  - \\[1ex] 
 \hline
\end{tabular}
\captionof{table}{Regression results on Noisy Yinyang $d=1000$ data.The smallest medium 5-fold cross-validation SSE from each method is listed with the corresponding parameters used. The $5$ percentile and $95$ percentile of the SSEs from the given parameter settings are reported in brackets.}
\label{table:NoiseYinyangd1000}
\vspace{2em}
\centering
\includegraphics[width=\textwidth]{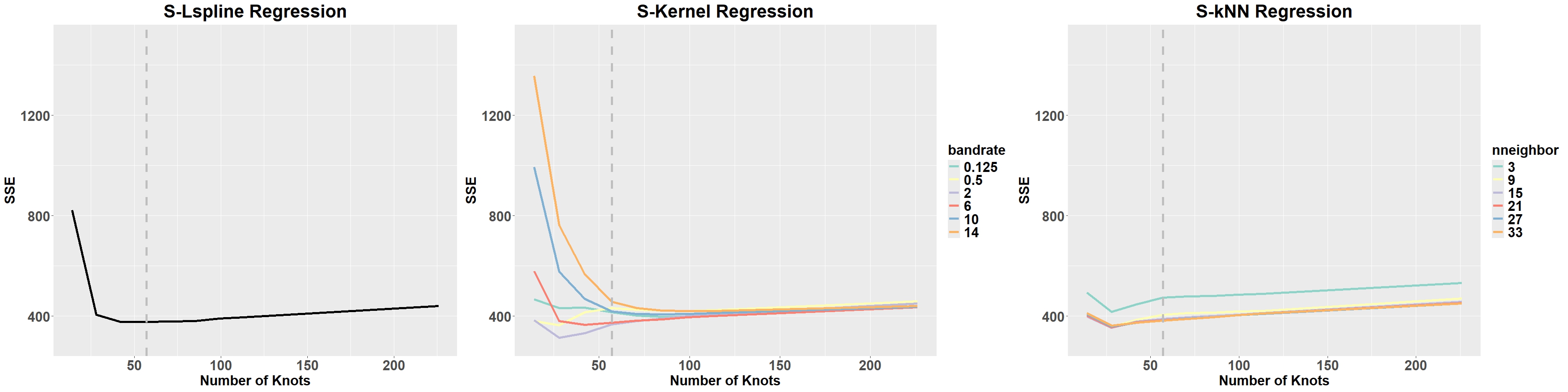}
\captionof{figure}{Noisy Yinyang $d = 1000$ data regression results with varying number of knots. The median SSE across the $100$ simulated datasets with each given parameter setting is plotted.}
\label{fig::NoiseYinyangd1000Numknots}

\end{figure}


To evaluate the robustness of the proposed skeleton-based regression methods, we randomly generate the Noisy Yinyang data 100 times and follow the analysis procedure as in Section \ref{sec:analysisProcedure}, except that we leave the skeleton to be a fully connected graph. We also took the median, 5th percentile, and 95th percentile of the 5-fold cross-validation SSEs for each parameter setting of each method on the 100 datasets. The smallest median SSE for each method is reported in Table \ref{table:NoiseYinyangd1000} along with the corresponding best parameter setting.

It can be observed that all the skeleton-based regression methods outperform the standard kNN approach, which indicates that the skeleton regression framework can capture the data structure in the presence of noisy observations and give good regression performance.
Among the skeleton-based methods, the S-Kernel has the best performance, and kernel smoothing can be a helpful nonparametric technique to deal with noisy observations.
The Ridge regression, Lasso regression, SpecSeries, Fast-KRR, 3-layer MLP autoencoder regressor,  again fail to provide good performance on this simulated dataset.
The advantage of the skeleton regression framework is more manifesting with noisy observations.

In Figure \ref{fig::Yinyangd1000Numknots}, we plot the median SSE of the skeleton-based methods on skeletons with different numbers of knots. Using the empirical rule to construct a skeleton with $[\sqrt{3200}] = 57$ knots results in good regression performance and approximately identifies the ``elbow'' position in Figure \ref{fig::Yinyangd1000Numknots}. However, for some skeleton-based methods, using a number of knots larger than that given by the empirical rule leads to better regression performance.
This improvement is related to the phenomenon observed in \citet{skelclus} that when dealing with noisy observations, it's better to have a skeleton with more knots and cut the skeleton into more disjoint components in order to have a cleaner representation of the key manifold structures.
Therefore, when facing data with noisy feature vectors, it's advised to empirically tune the number of knots favoring larger values.




\subsection{SwissRoll Data}
\label{sec::SwissRoll}
~~~~The intrinsic components of the covariates in Yinyang data are all well-separated, which, admittedly, can give an advantage to skeleton-based methods. 
Moreover, the intrinsic dimensions of the structural components for Yinyang data covariates are all lower than or equal to $1$ and can be straightforwardly represented by knots and line segments, potentially giving another advantage to skeleton-based methods.
To address such concerns, we present another simulated data which has covariates lying around a Swill Roll shape (Figure \ref{fig::SwissRollData} left), an intrinsically $2$-dimensional manifold in the $3$-dimensional Euclidean space. 
To make the density on the Swill Roll manifold balanced, we sample points inversely proportional to the radius of the roll in the $X_1 X_3$ plane. 
Specifically, let $u_1, u_2$ be independent random variables from $\text{Uniform}(0,1)$ and let the angle in the $X_1 X_3$ plane be generated as $\theta_{13} = \pi 3^{u_1}$. 
Then for the first $3$ dimensions of the covariates we have
\begin{align*}
  X_1 = \theta_{13} \cos(\theta_{13}), \ \ X_2 = 4 u_2, \ \ X_3 =  \theta_{13} \sin(\theta_{13}) 
\end{align*}
The true response has a polynomial relationship with the angle on the manifold if the $X_2$ value of the point is within some range. 
Let $\Tilde{\theta}_{13} = \theta_{13} - 2 \pi$, and let $\epsilon \sim N(0, 0.3)$. Then we set
\begin{align*}
    Y =  0.1\times \Tilde{\theta}_{13}^3 \times \left[ I(X_2<\pi) + I( 2\pi < X_2<3\pi) \right] + \epsilon
\end{align*}
The response versus the angle $\theta_{13}$ and $X_2$ is demonstrated in Figure \ref{fig::SwissRollData} right. Independent random Gaussian variables from $N(0,0.1)$ are added to make the covariates $1000$-dimensional in total, and $2000$ observations are sampled to make the Swiss Roll dataset.

\begin{figure}

\centering
    \begin{subfigure}[t]{0.4\textwidth}
        \centering
        \includegraphics[width=\linewidth]{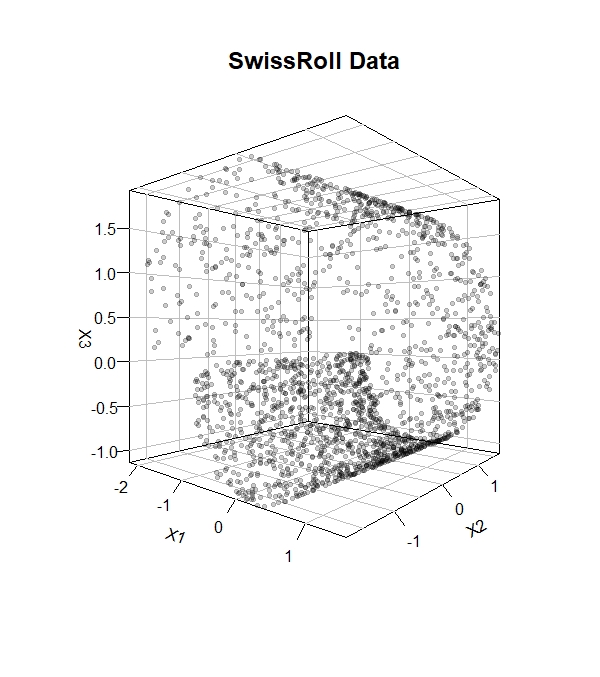} 
    \end{subfigure}
        \begin{subfigure}[t]{0.4\textwidth}
        \centering
        \includegraphics[width=\linewidth]{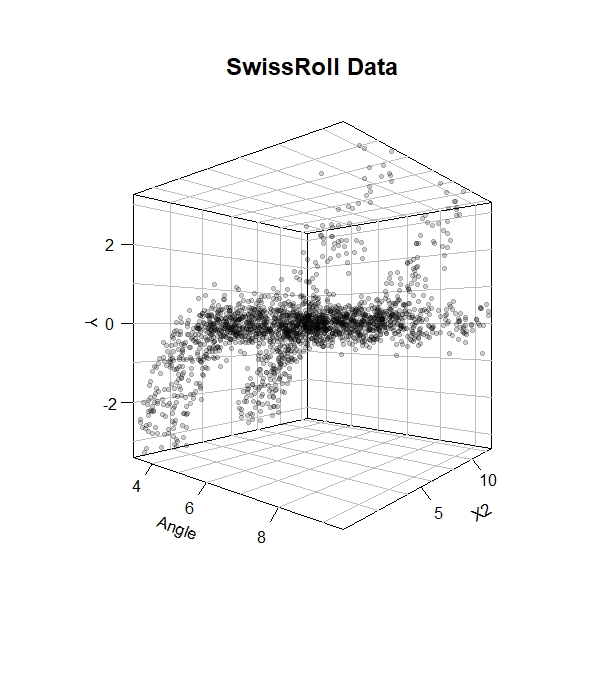}
    \end{subfigure}
\captionof{figure}{SwissRoll Regression Data}
\label{fig::SwissRollData}
\vspace{2em}
\begin{tabular}{c l c l} 
 \hline
 Method & Median SSE ($5$\%, $95$\%)  & nknots & Parameter \\ [1ex] 
 \hline
 kNN & 648.5 (607.1, 696.0) & - & neighbor=12 \\ 
 Ridge &  1513.7 (1394.4, 1616.2) & -& $\lambda  = 2.0$\\
 Lasso &  1191.4 (1106.7, 1260.7) & -& $\lambda = 0.032$\\
 SpecSeries & 1166.5 (1081.4, 1238.8) &- & bandwidth = $2.0$\\
 Fast-KRR & 1503.5 (1403.2, 1592.9) & - & $\sigma = 0.1$ \\
 3-layer MLP AE &  1505.2 (1405.1, 1649.5) & - & $h = 50, \lambda = 100$\\
 S-Kernel & 458.2 (409.0, 511.8) & 30 & bandwidth = 2 $r_{hns}$ \\
 S-kNN & 474.7 (417.6, 553.4) & 30 & neighbor = 18 \\
 S-Lspline & 569.8 (519.5, 645.8) & 60 & $\lambda = 0$ \\[1ex] 
 \hline
\end{tabular}
\captionof{table}{Regression results on the Swiss Roll $d=1000$ data. The smallest medium 5-fold cross-validation SSE from each method is listed with the corresponding parameters used. The $5$ percentile and $95$ percentile of the SSEs from the given parameter settings are reported in brackets.}
\label{table:Swissd1000}
\vspace{2em}
\centering
\includegraphics[width=\textwidth]{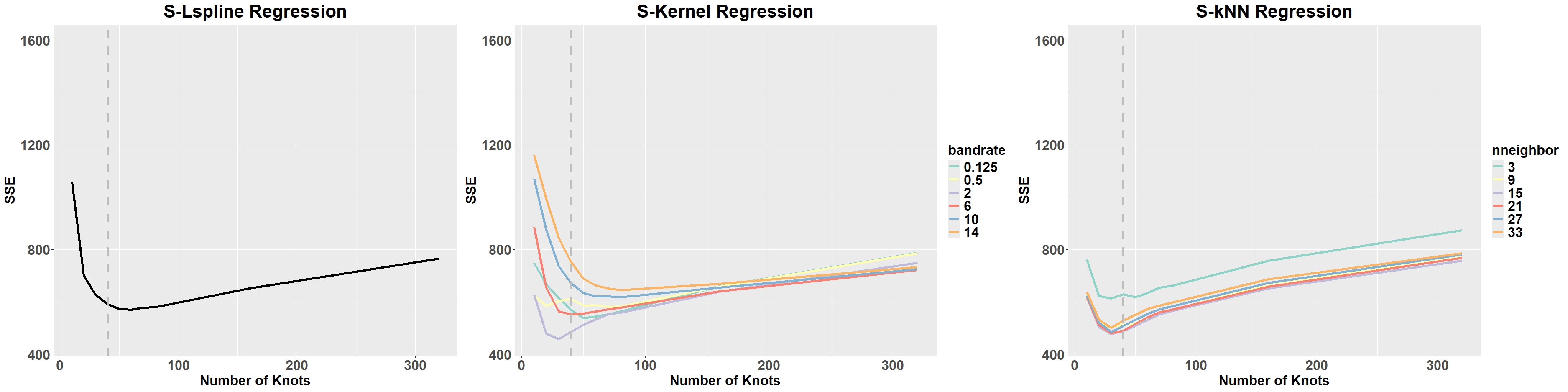}
\captionof{figure}{SwissRoll $d = 1000$ data regression results with varying number of knots. The median SSE across the $100$ simulated datasets with each given parameter setting is plotted.}
\label{fig::SwissRolld1000}

\end{figure}

We follow the same analysis procedures as in Section \ref{sec:analysisProcedure} with the skeletons constructed to be fully connected graphs without additional graph cuts. We took the median, 5th percentile, and 95th percentile of the 5-fold cross-validation SSEs across each parameter setting for each method on the 100 datasets, and reported the smallest median SSE for each method along with the corresponding best parameter setting in Table \ref{table:Swissd1000}.

All the proposed skeleton-based methods have better performance than the standard kNN regressor, while the S-Kernel method had the best performance in terms of SSE.
Particularly, the methods that utilize the skeleton-based distances, S-Kernel and S-kNN, have significantly better performance compared to the S-Lspline method which only utilizes the knot-edge structure of the skeleton graph.
Intuitively, the skeleton-based distances are good approximations to the geodesic distances on the manifold and hence lead to improvements in the regression performance.
The spectral, penalization-based, and neural network-based approaches do not demonstrate good performance on this simulated data.
Therefore, the proposed skeleton regression framework can also be powerful for data on connected, multi-dimensional manifolds.

By plotting the median SSE under skeletons with a varying number of knots in Figure \ref{fig::SwissRolld1000}, we observed that the best performance for all the skeleton-based methods is achieved with the number of knots larger than $[\sqrt{1600}] = 40$ knots. Given the intrinsic structure of the Swiss Roll input space is a $2$D plane, having more knots on the plane can give a better representation of the data structure and, therefore, lead to better prediction accuracy. We conjecture that the optimal number of knots should depend on the intrinsic dimension of the covariates, and we plan to discuss this further in future work. However, it's recommended to use cross-validation to choose the number of knots in practice.


\section{Real Data}	\label{sec::real}

~~~~In this section, we present analysis results on two real datasets.
We first predict the rotation angles of an object in a sequence of images taken from different angles (Section \ref{sec::cup}). 
For the second example, we study the galaxy sample from the Sloan Digital Sky Survey (SDSS) to predict the spectroscopic redshift (Section \ref{sec::sdss}), a measure of distance from a galaxy to Earth.

\subsection{Cup Images Data}
\label{sec::cup}
~~~~
This dataset consists of $72$ gray-scale images of size $128\times 128$ pixels taken from the COIL-20 processed dataset \citep{COIL20}. They are 2D projections of a 3D cup obtained by rotating the object by $72$ equispaced angles on a single axis.
Several examples of the images are given in Figure \ref{fig::cup}. 

The response in this dataset is the angle of rotation. However, this response has a circular nature where degree 0 is the same as degree 360. To avoid this issue, we removed the last 8 images from the sequence, only using the first 64 images. As a result, our dataset consists of 64 samples from a 1-dimensional manifold embedded in $\RR^{16384}$ along with scalar values representing the angle of rotation.
To assess the performance of each method, we use leave-one-out cross-validation, that, in each iteration, one image is taken out of the dataset and the regression methods are fitted to the remaining images to estimate the angle of the left-out image.

\begin{figure}
\centering
    \begin{subfigure}[t]{0.15\textwidth}
        \centering
        \includegraphics[width=\linewidth]{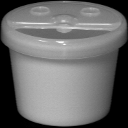} 
    \end{subfigure}
        \begin{subfigure}[t]{0.15\textwidth}
        \centering
        \includegraphics[width=\linewidth]{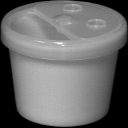}
    \end{subfigure}
            \begin{subfigure}[t]{0.15\textwidth}
        \centering
        \includegraphics[width=\linewidth]{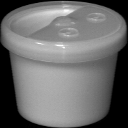}
    \end{subfigure} \begin{subfigure}[t]{0.15\textwidth}
        \centering
        \includegraphics[width=\linewidth]{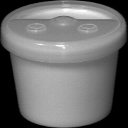}
    \end{subfigure} \begin{subfigure}[t]{0.15\textwidth}
        \centering
        \includegraphics[width=\linewidth]{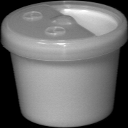}
    \end{subfigure} \begin{subfigure}[t]{0.15\textwidth}
        \centering
        \includegraphics[width=\linewidth]{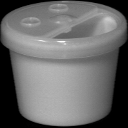}
    \end{subfigure}
\captionof{figure}{A part of the cup images from the COIL-20 processed dataset. Each image
is of size $128\time 128$ pixels.}
\label{fig::cup}
\vspace{2em}

\centering
\begin{tabular}{c c l} 
 \hline
 Method & SSE &  Parameter \\ [1ex] 
 \hline
 kNN & 1147.2 & neighbor=3 \\ 
 Ridge& -&-\\
 Lasso& -&-\\
 SpecSeries& -&-\\
 Fast-KRR & -&-\\
 S-Kernel & 1735.0 &  bandwidth = 2$r_{hns}$ \\
 S-kNN & 2068.8 & neighbor = 2\\
 S-Lspline &  1073.4  &- \\[1ex] 
 \hline
\end{tabular}
\captionof{table}{Regression results on cup images data from COIL-20. The best SSE from each method is listed with the corresponding parameters used.}
\label{table:cup}

\end{figure}

Similarly to the simulation studies, we use the skeleton construction method with Voronoi weights in \citet{skelclus} to construct the skeleton on the training set. 
In practice, we found that a small number of knots can still lead to loops in the constructed skeleton structure, and, after some tuning, we fit $2[\sqrt{n}] = 16$ knots to each training set. Additionally, since the underlying manifold should be one connected structure, we do not cut the constructed skeleton structure in this experiment. 
Due to the high-dimensional nature of the data, Ridge regression, Lasso regressions, and the Spectral Series approach failed to run with the implementations in R. 
The best result from each method is listed in Table \ref{table:cup} along with the corresponding parameters.

We observe that the S-Lspline method gives outstanding performance on this real data, outperforming the kNN regressor, while the other skeleton-based methods also demonstrate good performance. 
The lightening conditions of this series of images do not vary much by the rotation angle, which poses challenges to the similarity calculations based on the Euclidean distance and hence limits the performance of the classical kNN method. 
Note that the S-Kernel and S-kNN methods depend on the skeleton-based distances between data points while the S-Lspline methods do not, and hence the difference between such methods may imply that, although the skeleton graph can capture the data structure which leads to the good performance of the S-Lspline method, the skeleton-based distances can give inaccurate relations between data points compared to the true underlying data structure.
However, the skeleton graph still provides information about the data structure as the S-Lspline method has good performance with the simple piecewise linear model assumption on the skeleton.

\subsection{SDSS Data} \label{sec::sdss}

~~~~In this section, we applied the skeleton regression to a galaxy sample of size $5000$, taken from a random subsample of the Sloan Digital Sky Survey (SDSS), data release 12 \citep{york2000sloan, alam2015eleventh}. We repeat the random data subsampling for 100 times to get 100 different datasets. One dataset consists of $5$ covariates measuring apparent magnitudes of galaxies from images taken using $5$ photometric filters. These covariates can be understood as the color of a galaxy and are inexpensive to obtain. The response variable is the spectroscopic redshift, which is a very costly but accurate measurement of the distance
to the Earth. 
It is known that the $5$ photometric color measurements are correlated with
the spectroscopic redshift. So the goal is to use the photometric
information to predict the redshift; this is known as 
the clustering redshift problem in Astronomy literature \citep{morrison2017wizz, rahman2015clustering}.

We construct the skeleton with the same method in the simulation studies. The resulting skeleton graph is shown in Figure \ref{fig::SDSScolor}. 
In the left panel of Figure \ref{fig::SDSScolor}, we color the knots by their predicted redshift values according to the S-Lspline method and color the edges by the average predicted values of the two connected knots. 
For comparison, we color the knots and edges using the true values in the right panel of Figure \ref{fig::SDSScolor}. 
The predictions given by S-Lspline are very close to the true values. 

\begin{figure}
\centering
    \begin{subfigure}[t]{0.38\textwidth}
        \centering
        \includegraphics[width=\linewidth]{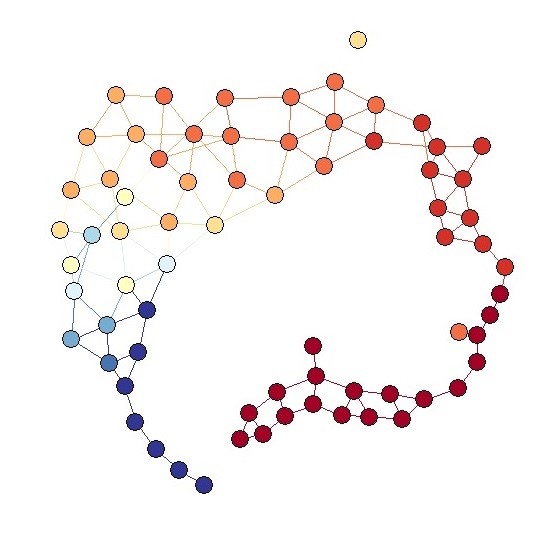} 
        \caption{S-Lspline}
    \end{subfigure}
        \begin{subfigure}[t]{0.49\textwidth}
        \centering
        \includegraphics[width=\linewidth]{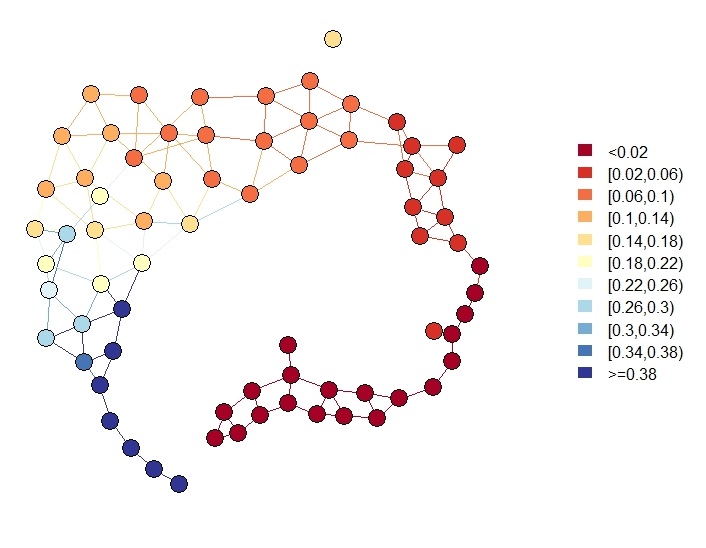}
        \caption{True Value}
    \end{subfigure}
\captionof{figure}{SDSS Skeleton Colored by values predicted by S-Lspline (left) and by true values (right).}
\label{fig::SDSScolor}
\vspace{2em}
\centering
\begin{tabular}{c l c l} 
 \hline
 Method & SSE & nknots &  Parameter \\ [0.5ex] 
 \hline
 kNN & 58.6 ( 46.7, 79.1) &  &neighbor=6 \\ 
 Ridge &868.4 (771.8, 984.5) & & $\lam = 0.001 $\\
 Lasso & 861.7 (750.3, 993.8) & & $\lam = 0.0013 $\\
 SpecSeries& 73.0 (54.1, 114.0) & &  bandwidth = 5\\
 Fast-KRR & 312.3 (242.5, 396.9) & & $\sigma = 0.1$\\
 S-Kernel & 78.6 (71.5, 92.4) &  126 &  bandwidth = $4 r_{hns}$ \\
 S-kNN& 83.1 (73.9, 98.7) & 126 & neighbor = 9\\
 S-Lspline &  75.9 (69.0, 90.4) & 126 & $\lambda$ = 0 \\[1ex] 
 \hline
\end{tabular}
\captionof{table}{Regression results on SDSS data. The best SSE from each method is listed with the corresponding parameters used. The 5 percentile and 95 percentile of the SSEs from the 100 runs are reported in brackets.}
\label{table:SDSS}
\end{figure}

For completeness, we perform the same analysis as in Section \ref{sec::simulation} by comparing the 5-fold cross-validation SSEs of different regression methods on this dataset and include results in Table \ref{table:SDSS}. 
The classical kNN shows superior performance on this dataset, which can imply that the kNN method adapts nicely to the complex structure of the data. 
Notably, the Spectral Series regression shows good performance in this low-dimensional setting.
However, note that the Spectral Series regression has a large variation in its performance ranging over the different subsampled datasets, with SSEs a 5 percentile of 54.1 to 95 percentile of 114.0. 
Overall, kNN and SpecSeries methods work well in this data and both methods can adapt to the underlying manifold, while the skeleton-based regression methods also show comparable results.
The Fast-KRR approach demonstrates performance better than the usual Ridge and Lasso regression, demonstrating the effectiveness of kernel tricks in this setting.
While skeleton approaches do not provide the best prediction accuracy,
the skeleton structure obtained in Figure~\ref{fig::SDSScolor}
shows a clear one-dimensional structure in the underlying covariate distribution and an approximate monotone trend in the response. 
Thus, even if our method does not provide the best prediction accuracy, the skeleton itself can be used as a tool to investigate the structure of the covariate distribution, which can be valuable for practitioners.



\section{Conclusion}	
\label{sec::conclusion}

~~~~In this work, we introduce the skeleton regression framework to handle regression problems with manifold-structured inputs. We generalize the nonparametric regression techniques, such as kernel smoothing and splines, onto graphs. 
Our methods provide accurate and reliable prediction performance and are capable of recovering the underlying manifold structure of the data.
Both theoretical and empirical analyses are provided to illustrate the effectiveness of the skeleton regression procedures.
The skeleton method not only learns a predictive model of the outcome but also outputs an elegant skeleton graph to summarize the structure of the feature space. Moreover, such a skeleton graph is a metric space, enabling us to use various prediction methods, and the whole procedure does not require much computational cost compared to deep learning models.

In what follows, we describe some possible future directions:
\begin{itemize}

\item {\bf Generalizing skeleton graphs to a simplicial complex.}
From a geometric perspective, the skeleton graph constructed in this work only focuses on $0$-simplices (points) and $1$-simplices (line segments). Additional geometric information can be encoded using higher-dimensional simplices. Recent research in deep learning has explored the use of simplicial complexes for tasks such as clustering and segmentation \citep{GeometricDL, MPSN2021}. Higher-dimensional simplicies
offer a finer approximation to the covariate distribution
but have a higher computational cost and a more complex model.
Thus, it is unclear if using a higher-dimensional simplex 
will lead to better prediction accuracy.
We will explore the possibility of extending skeleton graphs to the skeleton complex in the future. 

\item {\bf Nonparametric smoothers on graphs.} The kernel regression and spline regression
are not the only possibilities for performing nonparametric smoothing on graphs.
For example, \citet{Wang2016} generalized the concept of trend filtering \citep{Kim2009, Tibshirani2014} to graphs and compared it to Laplacian smoothing and Wavelet smoothing. In contrast to our work, these regression estimators for graphs are applied to data where both the inputs and responses are located on the vertices of a given graph. As a result, these graph smoothers, which include different regularizations, can only fit values on the vertices and do not model the regression function on the edges (\citet{Wang2016} mentioned the possibility of linear interpolation with trend filtering).

It is possible to generalize these methods to the skeleton by constructing responses on the knots in the skeleton graph as the mean values of the corresponding Voronoi cell, and then graph smoothers can be applied. Some interpolation methods can again be used to predict the responses on the edge, and this can lead to another skeleton-based regression estimator.

\item {\bf Time-varying covariates and responses.} A possible avenue for future research is to extend the skeleton regression framework to handle time-varying covariates and responses. Specifically, covariates collected at different times could be used together to construct knots in a skeleton. The edges in the skeleton can change dynamically according to the covariate distribution at different times, providing insight into how the covariate distributions have evolved. Additionally, representing the regression function on the skeleton would make it simple to visualize how the function changes over time.

\item {\bf Streaming data and online skeleton update.}
As streaming data becomes increasingly common, a potential area of future research is to investigate methods for updating the skeleton structure and its regression function in a real-time or online fashion. Reconstructing the entire skeleton can be computationally costly, but local updates to edges and knots can be more efficient. We plan to explore ways to develop a simple yet reliable method for updating the skeleton in the future.


\end{itemize}

\newpage
\begin{center}
    {\huge Appendices}
\end{center}
\begin{appendix}


\section{Skeleton Construction with Voronoi Density}
\label{ref::skelconsVoron}
~~~~
In this section, we provide a more detailed description of the procedures for constructing the skeleton and computing the density-aided edge weight called the Voronoi density, following the work in \cite{skelclus}.

\subsection{Knots Construction}	
\label{sec::knots}
~~~~
The knots in the skeleton serve as reference points within the data, allowing us to focus our attention from the overall data to these specific locations of interest. 
We utilize the $k$-means algorithm with a relatively large value of a number of knots $k$ to create these knots in a data-driven way. 
The number of knots is a crucial parameter in this procedure as it governs the trade-off between the summarizing power of the representation and the preservation of information. 
Empirical evidence from \cite{skelclus} suggests that setting $k$ to around $\sqrt{n}$ can be a helpful reference rule, while the dimensionality of the data should be taken into consideration when choosing $k$.

In practice, since the $k$-means algorithm may not always find the global optimum, we repeat it $1,000$ times with random initial points and select the result corresponding to the optimal objective. 
We also advise pruning knots with only a small number of with-in-cluster observations. Additionally, it can be helpful to preprocess or denoise the data by removing observations in low-density areas to address issues that could arise for $k$-means clustering. Figure \ref{fig::NoiseYinyangVaryKnotsIllustrate} illustrates the constructed knots on the Noisy Yinyang example with a varying number of knots; their performance is in Figure~\ref{fig::NoiseYinyangd1000Numknots}.

\begin{figure}[ht]
\centering
\includegraphics[width=0.8\linewidth]{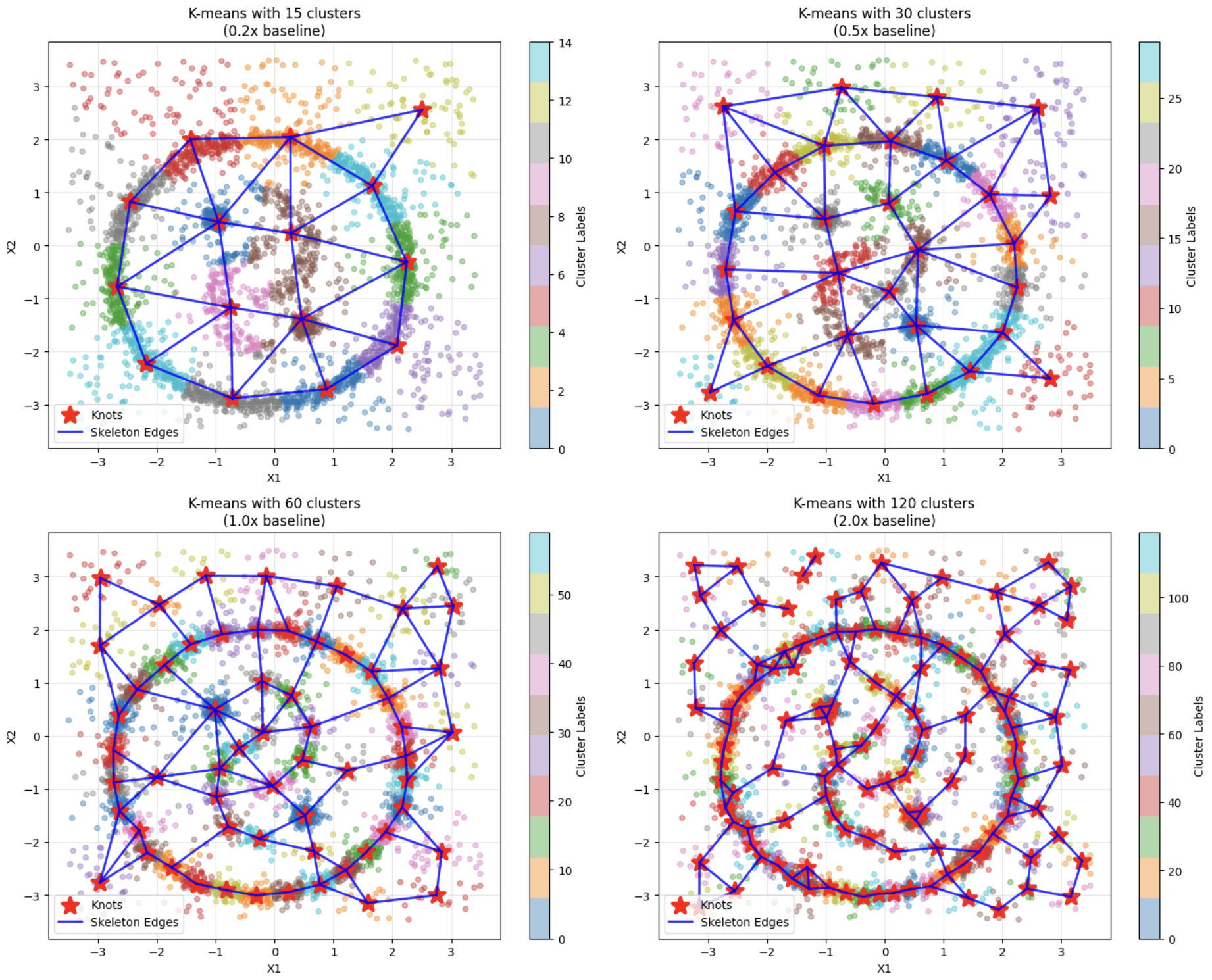}
\caption{Noisy Yinyang data with the skeleton constructed with a varying number of knots.}
\label{fig::NoiseYinyangVaryKnotsIllustrate}
\end{figure}

\subsection{Edges Construction}	\label{sec::edge}
~~~~
We denote the given knots as $c_1,\cdots, c_k$ and represent their collection as $\calC = {c_1,\cdots, c_k}$. An edge is added between two knots if they are neighbors, which is determined by whether their corresponding Voronoi cells share a common boundary. The Voronoi cell associated with a knot $c_j$ is defined as the set of points in $\calX$ whose distance to $c_j$ is the smallest among all knots. 
That is, 
\begin{align}
    \CC_j = \{x \in \calX: d(x, c_j) \leq d(x, c_\ell) \ \  \forall \ell \neq j\},
\end{align}
where $d(x,y)$ is the usual Euclidean distance.
We add an edge between knots $(c_i,c_j)$ if their Voronoi cells have a non-empty intersection. This graph is referred to as the Delaunay triangulation of $\calC$, denoted as $DT(\calC)$. 

Although the Delaunay triangulation graph is conceptually intuitive, the computational complexity of the exact Delaunay triangulation algorithm has an exponential dependence on the ambient dimension $d$, making it unfavorable for multivariate or high-dimensional data settings. To overcome this issue, we approximate the Delaunay triangulation with $\hat{DT}(\calC)$ by examining the 2-nearest knots of the sample data points. We query the two nearest knots for each data point and add an edge between $c_i, c_j$ if there is at least one data point whose two nearest neighbors are $c_i, c_j$. The computational complexity of this sample-based approximation depends linearly on the dimension $d$, making it suitable for high-dimensional settings.

\subsection{Voronoi Density}	\label{sec::VD}

~~~~The Voronoi density (VD) measures the similarity between a pair of knots $(c_j,c_\ell)$ based on the number of observations whose 2-nearest knots are $c_j$ and $c_\ell$. We first define the Voronoi density based on the underlying probability measure and then introduce its sample analog. Given a metric $d$ on $\RR^d$, the 2-Nearest-Neighbor (2-NN) region of a pair of knots $(c_j,c_\ell)$ is defined in Equation \ref{eq::2NNregion} as
\begin{align*}
    B_{j\ell} = \{ X_m, m=1,\dots,n :\norm{x- V_i} > \max\{ \norm{x-V_j}, \norm{x- V_\ell} \}, \forall i \neq j, \ell \}.
\end{align*}
Figure~\ref{fig::2nn} provides an illustration of an example 2-NN region of a pair of knots.
If two knots $c_j, c_\ell$ are in a connected high-density region, then we expect the 2-NN region of $c_j, c_\ell$ to have a high probability measure. Therefore, the probability $\PP(B_{j\ell}) = P(X_1 \in B_{j\ell})$ can measure the association between $c_j$ and $c_\ell$. Based on this insight, the Voronoi density measures the edge weight of $(c_j,c_\ell)$ as
\begin{align}
S_{j\ell}^{VD} = \frac{\PP(B_{j\ell})}{|c_j - c_\ell |}.
\end{align}
The Voronoi density adjusts for the fact that 2-NN regions have different sizes by dividing the probability of the in-between region by the mutual Euclidean distance. 

In practice, we estimate $S_{j\ell}^{VD}$ by a sample average. The numerator $\PP(B_{j\ell}) $ is estimated by $\hat{P}_n (B_{j\ell}) = \frac{1}{n}\sum_{i=1}^n I(X_i\in B_{j\ell})$, and the final estimator for the VD is:

\begin{align}
\hat{S}_{j\ell}^{VD} &= \frac{\hat{P}_n ({B}_{j\ell})}{|{c}_j - {c}_\ell|}.
\end{align}

Calculating the Voronoi density is fast. The numerator, which only depends on 2-nearest-neighbors calculation, can be computed efficiently by the k-d tree algorithm. For high-dimensional space, space partitioning search approaches like the k-d tree can be inefficient, but a direct linear search still gives a short run-time.

\subsection{Graph Segmentation}	\label{sec::segmenting}

~~~~After obtaining the weighted skeleton graph, it can be helpful to prune certain edges that are not of interest or segment the skeleton into disconnected components. 
The edge weights defined above can be utilized to achieve this. 
We start by first converting the edge weights into dissimilarity measures. Specifically, let ${s_{ij}}{i\neq j}$ be the edge weights, where only connected pairs can take non-zero entries, and let $s{\max} = \max_{i \neq j} s_{ij}$.
We then define the corresponding dissimilarities as $d_{ij} = 0$ if $i = j$, and $d_{ij} = s_{\max} - s_{ij}$ otherwise. 
Next, we apply hierarchical clustering using these distances. The choice of linkage criterion for hierarchical clustering depends on the underlying geometric structure of the data. Single linkage is recommended when the components are well-separated, while average linkage works better when there are overlapping clusters of approximately spherical shapes. 
To determine the resulting segmented skeleton graph, dendrograms can be useful in displaying the clustering structure at different resolutions, and analysts can experiment with different numbers of final clusters and choose a cut that preserves meaningful structures based on the dendrograms. However, it is important to note that the presence of noisy data points may require a larger number of final clusters $S$ to achieve better clustering results. 
Figure \ref{fig::YinyangVaryingCutIllustration} illustrates the segmented skeleton with a varying number of disjoint components.
\begin{figure}[ht]
\centering
\includegraphics[width=\linewidth]{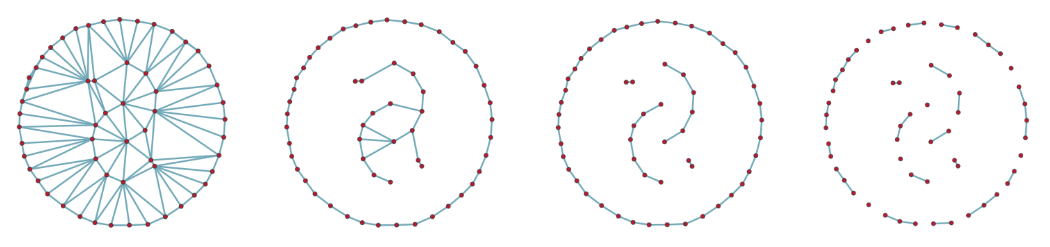}
\caption{Yinyang data skeleton cut into a varying number of disjoint components.}
\label{fig::YinyangVaryingCutIllustration}
\end{figure}

\section{Computational Complexity}	\label{sec::complexity}
~~~~In this section, we briefly analyze the computational costs of the proposed skeleton regression framework. 
The first main computational burden of the proposed regression procedure is at the skeleton construction step. \citet{skelclus} has provided the computational analysis on this. In particular, when constructing knots, the $k$-means algorithm of Hartigan and Wong \cite{Hartigan1979} has time complexity $O(ndkI)$, where $n$ is the number of points, $d$ is the dimension of the data, $k$ is the number of clusters for $k$-means, and $I$ is the number of iterations needed for convergence. For the edge construction step, the approximate Delaunay Triangulation only depends on the 2-NN neighborhoods,  and the k-d tree algorithm for the 2-nearest knot search gives the worst-case complexity of $O(nd k^{(1-1/d)})$.  For the edge weights with Voronoi density, the numerator can be computed directly from the $2$-NN search without additional computation, and the denominators as pairwise distances between knots can be computed with the worst-case complexity of $O(dk^2)$.

Given the skeleton, we then project original feature vectors onto the skeleton, which is not very time-consuming. Finding the edge to project on depends on identifying the two nearest knots, which is provided in the skeleton construction step. The projection takes inner product computations and takes $O(nd)$ for all the covariates.

The next computational task is to calculate the skeleton-based distance between points on the skeleton. 
Note that this step is not needed for the S-Lspline method but is necessary for S-Kernel and S-kNN.
To find the shortest path on a graph between two faraway knots, the general version of Dijkstra's algorithm \cite{Dijkstra1959} takes $ \Theta (|\calE|+|\calV|^{2})=\Theta (k^{2})$ for each run. However, in practice, we don't need the $\frac{n(n-1)}{2}$ pairwise distances between all the projected points as the skeleton-based regressors proposed can perform with distances in local neighborhoods, which do not require path-finding algorithm for the skeleton-distance calculation.

With all the pairwise skeleton-based distances between projected feature points given, the S-kernel estimate at one point takes $n_{loc}$ kernel weights computation where $n_{loc}$ refers to the local support of the kernel function. S-Lspline takes $O(n)$ time to transform the data and then a single run of matrix multiplication and inversion to get the coefficients.

\section{Proofs}	\label{sec::proof}


\subsection{Kernel Regression: Convergence on Edge Point (Theorem \ref{thm::edge})} \label{sec::contproof}


\begin{proof}
Let $\calB(\bs, h) \subset \calS$ be the support for the kernel function $K_h(.)$ at point $\bs \in \calS$ with bandwidth $h$.
For an edge point $\bs \in E_{j\ell} \in \calE$, where $\calE$ is the overall set of edges defined as open sets. As $n\to \infty, h\to 0$, for sufficiently large $n$, by the property of an open set, we have
$$\calB(\bs, h) \subset E_{j\ell} $$ 
and by our definition of skeleton distance, for two points $\bs, \bs' \in E_{j\ell}$ on the same edge in the skeleton, $d_\calS(\bs, \bs') = \norm{\bs-\bs'}$ where $\norm{.}$ denotes the Euclidean distance and is 1-dimensional as parametrized on the same edge. Also we have $$K_h(\bs_j, \bs_\ell) \equiv  K(d_\calS(\bs_j, \bs_\ell)/h) = K(\norm{\bs_j-\bs_\ell}/h) = K\left(\frac{\bs_j-\bs_\ell}{h}\right)$$

Consequently, the skeleton-based kernel regression estimator reduces to 
\begin{align}
    \hat{m}_n(\bs) = \frac{\frac{1}{n h}\sum_{j=1}^n Y_j K(\frac{\bs_j- \bs}{h})  }{\frac{1}{n h}\sum_{j=1}^n K(\frac{\bs_j- \bs}{h})}
\end{align}
and we can use the classical asymptotic results for kernel regression in the continuous case \cite{Bierens1983, wasserman2006all,Chen2017}.

Let $\hat{g}_n(\bs) = \frac{1}{nh}\sum_{j=1}^n K\left(\frac{\bs_j-\bs}{h}\right)$. 
We express the difference as
\begin{equation}
\begin{aligned}
     \hat{m}_n(\bs) - m_\calS(\bs) &= \frac{[\hat{m}_n(\bs) - m_\calS(\bs)] \hat{g}_n(\bs)}{\hat{g}_n(\bs)} = \frac{\frac{1}{n h}\sum_{j=1}^n [Y_j - m_\calS(\bs)] K(\frac{\bs_j- \bs}{h})  }{\frac{1}{n h}\sum_{j=1}^n K(\frac{\bs_j-\bs}{h})}
\end{aligned}
\label{eq::decom}
\end{equation}
and we analyze the denominator and numerator below.

Let $g(\bs)$ be the density at point $\bs$ on the skeleton.
For the denominator, we start with the bias:
\begin{align*}
    \abs{\EE \hat{g}_n(\bs) - g(\bs)} &= \abs{\frac{1}{h} \int K\bigg(\frac{\bs - y}{h}\bigg) g(y) d y - g(x) \int K(y) dy}\\
    &= \abs{\int K(z) [g(\bs - h z) - g(\bs)] dz}\\
    &\leq \int K(z) C_1 \abs{h z} dz = C_1 h \int K(z)  \abs{ z} dz = O(h),
\end{align*}
where $C_1$ is the Lipschitz constant of the density function.
For the variance, we have
\begin{align*}
    {\text Var}\left( \hat{g}_n(\bs) \right) &\leq \frac{1}{n h^2} \int K^2\bigg(\frac{\bs - y}{h}\bigg)g(y) dy\\
    &= \frac{1}{n h} \int K^2(z) g(\bs - h z) dz\\
    &\leq \frac{1}{n h} \int K^2(z) [g(\bs) + C_1 \abs{h z} ] dz\\
    &= \frac{1}{n h}\left[ g(\bs) \int K^2(z) dz + C_1 h \int K^2(z) \abs{ z}  dz \right]\\
    &= \frac{1}{n h} g(\bs) \int K^2(z) dz  + o\left(\frac{1}{n h}\right).
\end{align*}

Putting it all together, we have
\begin{align*}
    \abs{\hat{g}_n(\bs) - g(\bs)} = O(h) + O_p\left(\sqrt{\frac{1}{n h}}\right).
\end{align*}
Note that we only assume Lipschitz continuity and hence have the bias of rate $O(h)$ rather than the usual $O(h^2)$ rate with second-order smoothness. 
Higher-order smoothness of $g$ may not improve the overall
estimation rate due to the fact that
we only have Lipschitz continuity of the regression function.


Now we analyze the numerator of equation \eqref{eq::decom}.
We start with the decomposition
\begin{align*}
    &[\hat{m}_n(\bs) - m_\calS(\bs)] \hat{g}(\bs) =  \underbrace{\frac{1}{n h} \sum_{j=1}^n U_j K\bigg(\frac{\bs - \bs_j}{h}\bigg)}_{q_1(\bs)} +  \\
    & \underbrace{\frac{1}{n} \sum_{j=1}^n \bigg\{ [m_\calS(\bs_j) - m(\bs)]K\bigg(\frac{\bs - \bs_j}{h}\bigg) \frac{1}{h} - \EE\bigg[ [m_\calS(\bs_j) - m_\calS(\bs)]K\bigg(\frac{\bs - \bs_j}{h}\bigg) \frac{1}{h} \bigg]  \bigg\} }_{q_2(\bs)} +\\
    &  \underbrace{ \frac{1}{n} \sum_{j=1}^n \EE\bigg[ [m_\calS(\bs_j) - m_\calS(\bs)]K\bigg(\frac{\bs - \bs_j}{h}\bigg) \frac{1}{h} \bigg]}_{q_3(\bs)}.
\end{align*}

First, we show that
\begin{align*}
    q_1(\bs)  =  O_p\left(\sqrt{\frac{1}{n h} } \right).
\end{align*}
Let
\begin{align*}
    v_{n,j} (\bs) = U_j K\bigg(\frac{\bs - \bs_j}{h}\bigg)\frac{1}{\sqrt{h}}
\end{align*}
and we have
\begin{align*}
     \sqrt{n h} q_1(\bs) = \frac{1}{\sqrt{n}} \sum_{j=1}^n v_{n,j}(\bs).
\end{align*}
Thus, its mean is
\begin{align*}
    \EE v_{n,j}(\bs) &= \EE\left\{ U_j K\bigg(\frac{\bs - \bs_j}{h}\bigg)\frac{1}{\sqrt{h}} \right\}
    = 0
\end{align*}
and the variance is
\begin{align*}
    \EE [v_{n,j}(\bs)^2] &= \EE U_j^2 K\bigg(\frac{\bs - \bs_j}{h}\bigg)^2\frac{1}{h} = \int \sig_u^2(\bs - h z) g(\bs - h z) K(z)^2 dz\\
    &\to \sig_u^2(\bs) g(\bs) \int K(z)^2 dz = O(1),
\end{align*}
where for the second equality we use the change of variable and by assumption, we have $\int K(z)^2 dz <\infty$. Therefore, 
\begin{align*}
    q_1(\bs)  =  O_p\left(\sqrt{\frac{1}{n h }} \right).
\end{align*}

For the second term, note that $\EE (q_2(\bs)) = 0$ and the variance is
\begin{align*}
    \EE \left[\sqrt{n h} q_2(\bs)\right]^2 &= \int [m_\calS(\bs - h z) - m_\calS(\bs)]^2 g(\bs - h z) K(z)^2 dz \\
    &\ \ - h \bigg\{ \int [m_\calS(\bs - h z) - m_\calS(\bs)] g(\bs - h z) K(z) dz  \bigg\}^2\\
    &\to 0
\end{align*}
when $h\rightarrow0$, and hence,
\begin{align*}
    q_2(\bs)  =  o_p\left(\sqrt{\frac{1}{n h} } \right).
\end{align*}

For the last term, note that we have 
\begin{align*}
    q_3(\bs) &= \int[m_\calS(\bs - h z) - m_\calS(\bs)]g(\bs - h z) K(z)  dz \\
    &=\int [ m_\calS(\bs - h z)g(\bs - h z) - m_\calS(\bs) g(\bs) ] K(z) dz\\
    & \ \ \ \   - m_\calS(\bs) \int [g(\bs - h z) - g(\bs) ] K(z) dz\\
    &\leq C_1 h \int |z| K(z)dz + C_2 h \int |z| K(z) dz
\end{align*}
where $C_1$ is the Lipschitz constant for $m(\bs)g(\bs)$ and $C_2$ is the Lipschitz constant for $g(\bs)$. Therefore, 
\begin{align*}
     q_3(\bs) = O(h)
\end{align*}

Putting all three terms together, $[\hat{m}(\bs) - m(\bs)] \hat{g}(\bs)  = O(h) + O_p\left(\sqrt{\frac{1}{n h} } \right)$. 
As a result, equation \eqref{eq::decom} becomes
\begin{align*}
    \hat{m}_n(\bs) - m_\calS(\bs) &= \frac{[\hat{m}_n(\bs) - m_\calS(\bs)] \hat{g}(\bs)}{\hat{g}(\bs)} = \frac{O(h) + O_p\left(\sqrt{\frac{1}{n h} } \right)}{ g(\bs) + O(h) + O_p\left(\sqrt{\frac{1}{n h} } \right)}\\
    &= O(h) + O_p\left(\sqrt{\frac{1}{n h} } \right)
\end{align*}
by the Taylor expansion of the fraction. 

\end{proof}

\subsection{Kernel Regression: Convergence on Knot with Zero Mass (Proposition \ref{prop::zeroknot})}
\label{sec::zeroknotproof}


For the ease of proof, we first prove
Proposition \ref{prop::zeroknot}
and then prove Theorem \ref{thm::knotconsistency}.

\begin{proof}
Let $\bs \in \calV$ be a knot with no mass, i.e.,
$P(\bs_j = \bs) = 0$. 
The kernel regression can be decomposed
as
\begin{align*}
    \hat{m}(\bs)
    &= \frac{\frac{1}{n}\sum_{j=1}^n Y_j K_{h}(\bs_j, \bs) I(\bs_j \in \calE \cap  \calB(\bs, h))  + \frac{1}{n}\sum_{j=1}^n Y_j  I(\bs_j =\bs)}{\frac{1}{n}\sum_{j=1}^n K_{h}(\bs_j, \bs)I(\bs_j \in \calE \cap  \calB(\bs, h)) + \frac{1}{n}\sum_{j=1}^n  I(\bs_j =\bs) }\\
    &= \frac{\eps_{1,n}(\bs) + \nu_{1,n}(\bs) }{\eps_{2,n}(\bs) + \nu_{2,n}(\bs) }.
\end{align*}

Because $\bs$ is a point without probability mass,
$\nu_{1,n}(\bs)=\nu_{2,n}(\bs) =0$,
so the above can be further reduced to
\begin{align*}
    \hat{m}(\bs)
     &= \frac{\frac{1}{n h}\sum_{j=1}^n Y_j K_{h}(\bs_j, \bs) I(\bs_j \in \calE \cap  \calB(\bs, h))  }{\frac{1}{n h}\sum_{j=1}^n K_{h}(\bs_j, \bs)I(\bs_j \in \calE \cap  \calB(\bs, h))}.
\end{align*}

However, different from the case on edges, the support of the kernel intersects with multiple edges even when $h\rightarrow0$, so we study the contribution of each edge individually. 
Note that when $h\rightarrow0$, the only knot that exists in the intersection $\calB(\bs, h)\cap \calE$ is $\bs$.
So we only need to consider contributions of edges 
adjacent to $\bs$.

Let $\calI$ collect all the edge indices with one knot being $\bs$, i.e., $\ell \in \calI$ implies that there is an edge
between $\bs$ and $\bv_\ell \in \calV$.
Let $E_\ell$ be the edge connecting $\bs$ and $\bv_\ell$.
The indicator function $I(\bs_j \in \calE \cap  \calB(\bs, h)) = \sum_{\ell\in\calI}I(\bs_j \in E_\ell \cap  \calB(\bs, h)).$
With this, we can rewrite $\hat m(\bs)$ as 
\begin{equation*}
\begin{aligned}
    \hat m(\bs) &= \frac{ \sum_{\ell\in  \calI} \frac{1}{n h}\sum_{j=1}^n Y_j K_{h}(\bs_j, \bs) I(\bs_j \in E_\ell \cap  \calB(\bs, h))  }{\sum_{\ell\in \calI}\frac{1}{n h}\sum_{j=1}^n K_{h}(\bs_j, \bs)I(\bs_j \in E_\ell \cap  \calB(\bs, h))}\\
    & = \frac{\sum_{\ell\in \calI} \hat m_{n,\ell}(\bs) \hat g_{n,\ell}(\bs)}{\sum_{\ell\in\calI} \hat g_{n,\ell}(\bs)}.
    \end{aligned}
\end{equation*}
where
\begin{align*}
    \hat{g}_{n, \ell}(\bs) &= \frac{1}{n h}\sum_{j=1}^n K\left(\frac{\bs_j-\bs}{h}\right)I(\bs_j \in E_\ell \cap  \calB(\bs, h)),\\
\hat{m}_{n,\ell}(\bs) \cdot \hat{g}_{n, \ell}(\bs) &= \frac{1}{n h}\sum_{j=1}^n Y_j K\left(\frac{\bs_j-\bs}{h}\right) I(\bs_j \in E_\ell \cap  \calB(\bs, h)).
\end{align*}
Thus, we will analyze $\hat g_{n,\ell}(\bs)$
and $\hat m_{n,\ell}(\bs) \hat g_{n,\ell}(\bs)$. 
For a point $\bs_j$ on the edge $E_\ell$, we can reparamterize it as $\bs_j =  T_j \bv_\ell  + (1-T_j) \bs$ for some  $T_j \in (0,1)$. 
The location $\bs$ corresponds to the case $T_j = 0$
and any $\bs_j\in E_\ell$ will be mapped to $T_j>0.$
With this reparameterization, 
we can write
\begin{align*}
    \hat{g}_{n, \ell}(\bs) 
    &= \frac{1}{n h}\sum_{j=1}^n K\left(\frac{T_j}{h}(\bv_\ell-\bs)\right)I(\bs_j \in E_\ell \cap  \calB(\bs, h)),\\
\hat{m}_{n,\ell}(\bs) \cdot \hat{g}_{n, \ell}(\bs) 
&= \frac{1}{n h}\sum_{j=1}^nY_j K\left(\frac{T_j}{h}(\bv_\ell-\bs)\right)I(\bs_j \in E_\ell \cap  \calB(\bs, h)).
\end{align*}

To study the limiting behavior when $h \rightarrow0$,
let $g_\ell(t) = g((1-t)\bs + t \bv_\ell)$, $g_\ell(0) = \lim_{x\downarrow 0} g_{\ell}(x)$;
$m_\ell(t) = m_\calS( (1-t)\bs + t \bv_\ell)$, $m_\ell(0) = \lim_{t \downarrow 0} m_\ell(t)  $;
and
$\sigma_\ell^2(t) = \EE(|U_j|^2 | \bs_j = (1-t) \bs + t \bv_\ell )$ , $\sigma_\ell^2(0) = \lim_{t \downarrow 0} \sigma_\ell^2(t) $.
Then with the new notations, we can write 
\begin{align*}
\EE(f(T_j(\bv_\ell-\bs))I(\bs_j \in E_\ell \cap  \calB(\bs, h))) &= \EE(f(\bs_j-s)I(\bs_j \in E_\ell \cap  \calB(\bs, h)))\\
    &=\int_{t>0}f(t) g_\ell(t)dt
\end{align*}
for any integrable function $f$.
The bias of the denominator can be written as
\begin{align*}
    \abs{\EE \hat{g}_{n,\ell}(\bs) -\frac{1}{2} g_\ell(0)} &= \abs{\frac{1}{h} \int_{t>0} K\bigg(\frac{t}{h}\bigg) g_\ell(t) d t  - g_\ell(0) \int_{z>0} K(z)}\\
    &= \abs{\int_{z>0} K(z) [g_\ell(h z) -  g_\ell(0)] dz } \\
    &\leq \int_{z>0} K(z) C_1 h z  dz \\
    &=  C_1 h  \int_{z>0} K(z) z dz = O(h).
\end{align*}
For stochastic variation, we have
\begin{align*}
    {\text Var}\left( \hat{g}_{n,\ell}(\bs) \right) &\leq \frac{1}{n h^2} \int_{t>0} K^2\bigg(\frac{t}{h}\bigg)g_\ell(t) dt\\
    &= \frac{1}{n h} \int_{z>0} K^2(z) g(h z) dz\\
    &\leq \frac{1}{n h} \int_{z>0} K^2(z) [g(0) + C_1 \abs{h z} ] dz\\
    &= \frac{1}{n h}\left[ g(0) \int_{z>0} K^2(z) dz + C_1 h \int_{z>0} K^2(z) \abs{ z}  dz \right]\\
    &= O\left(\frac{1}{n h}\right).
\end{align*}
Thus, 
\begin{align*}
    \hat{g}_n(\bs) = \sum_{\ell \in \calI} \hat{g}_{n,\ell}(\bs) = \frac{1}{2} \sum_{\ell \in \calI}   g_\ell(0) + O(h) + O_p\left(\sqrt{\frac{1}{n h}}\right)
\end{align*}

For the numerator, 
\begin{align*}
    \hat{m}_{n,\ell}(\bs) \hat{g}_{n,\ell}(\bs) &=  \underbrace{\frac{1}{n h} \sum_{j=1}^n U_j K\bigg(\frac{t_j}{h}\bigg)I(\bs_j \in E_\ell \cap  \calB(\bs, h))}_{Q_1} \\
    &\qquad+  \underbrace{\frac{1}{n h} \sum_{j=1}^n   m_\calS(\bs_j) K\bigg(\frac{t_j}{h}\bigg)I(\bs_j \in E_\ell \cap  \calB(\bs, h))}_{Q_2},
\end{align*}
where $U_j  = Y_j - m_\calS(\bs_j)$.
Using the fact that $\EE(U_j|\bs_j) = 0$,
$\EE(Q_1)=0$, and the variance is
\begin{align*}
    {\text Var}(Q_1) &\leq \frac{1}{n h^2 } \int_{t > 0} \sig_\ell^2(t) K^2\bigg(\frac{t}{h}\bigg) g_\ell(t) dt \\
    &= \frac{1}{n h } \int_{z > 0} \sig_\ell^2(h z) K^2(z) g_\ell(h z) dz\\
    &= \frac{1}{n h } \int_{z > 0} \sig_\ell^2(0) K^2(z) g_\ell(0) dz + O\left(\frac{1}{n h}\right) = O\left(\frac{1}{n h}\right) .
\end{align*}
For $Q_2$, we have
\begin{align*}
    \abs{\EE (Q_2) - \frac{m_\ell(0) g_\ell(0)}{2}  } &= \abs{\frac{1}{h}\int_{t>0} m_\ell( t) K(t/h) g(t)  dt - m_\ell(0) g_\ell(0) \int_{z>0}  K(z)  dz } \\
    &=\abs{\int_{z>0} m_\ell( h z) K(z) g_\ell(h z)  dz - m_\ell(0) g_\ell(0) \int_{z>0}  K(z)  dz } \\
    &\leq \int_{z>0} \bigg\{ \big[m_\ell(0)+ C_2 h z\big]  \big[g_\ell(0) + C_1 h z\big] - m_\ell(0) g_\ell(0)\bigg\} K(z)  dz \\
    &\leq [C_1 m_\ell(0)+C_2g_\ell(0)] h \int_{z>0}  K(z) z dz + o(h)=  O(h).
\end{align*}
The variance of $Q_2$ is bounded via
\begin{align*}
    {\text Var}(q_2) &\leq \frac{1}{n h^2 } \int_{t > 0} m_\ell^2(t) K^2\bigg(\frac{t}{h}\bigg) g_\ell(t) dt\\
    &= \frac{1}{n h} \int_{z > 0} m_\ell^2(h z) K^2(z) g_\ell(h z) dz\\
    &\leq \frac{1}{n h} \int_{z>0} \left\{m_\ell(0)+ C_2 \abs{h z}\right\}^2  K^2(z) \left\{g_\ell(0) + C_1 \abs{h z}\right\}  dz \\
    &= \frac{1}{n h} \left\{ m^2_\ell(0) g_\ell(0) \int_{z>0} z K^2(z)   dz  + O(h)\right\}\\
    &= O\left(\frac{1}{n h}\right) 
\end{align*}
Putting the terms $Q_1$ and $Q_2$ together, we have
\begin{align*}
    \hat{m}_{n,\ell}(\bs) \hat{g}_{n,\ell}(\bs) = \frac{1}{2} m_\ell(0) g_\ell(0) + O(h) + O_p\left(\sqrt{\frac{1}{n h}}\right).
\end{align*}
As a result, we conclude that
\begin{align*}
    \hat{m}(\bs) &= \frac{\sum_{\ell \in \calI}\hat{m}_{n,\ell}(\bs)  \hat{g}_{n,\ell}(\bs)}{\sum_{\ell \in \calI} \hat{g}_{n,\ell}(\bs) }\\
    &= \frac{ \frac{1}{2} \sum_{\ell \in \calI}   m_\ell(0) g_\ell(0) + O(h) + O_p\left(\sqrt{\frac{1}{n h}}\right) }{\frac{1}{2} \sum_{\ell \in \calI}   g_\ell(0) + O(h) + O_p\left(\sqrt{\frac{1}{n h}}\right)} \\
    &=  \frac{ \frac{1}{2} \sum_{\ell \in \calI}   m_\ell(0) g_\ell(0)  }{\frac{1}{2} \sum_{\ell \in \calI}   g_\ell(0) } + O(h) + O_p\left(\sqrt{\frac{1}{n h}}\right)\\
    &= \frac{ \sum_{\ell \in \calI}   m_\ell(0) g_\ell(0)  }{ \sum_{\ell \in \calI}   g_\ell(0) } + O(h) + O_p\left(\sqrt{\frac{1}{n h}}\right),
\end{align*}
which completes the proof.
\end{proof}

\subsection{Kernel Regression: Convergence on Knot with Nonzero Mass (Theorem \ref{thm::knotconsistency})}
\label{sec::knotproof}

\begin{proof}
Let $\bs \in \calV$ be a point where $P(\bs_j =\bs) = p(\bs)>0$.
Recall that the kernel regression can be expressed as
\begin{align*}
    \hat{m}(\bs)
    &= \frac{\frac{1}{n}\sum_{j=1}^n Y_j K_{h}(\bs_j, \bs) I(\bs_j \in \calE \cap  \calB(\bs, h))  + \frac{1}{n}\sum_{j=1}^n Y_j  I(\bs_j =\bs)}{\frac{1}{n}\sum_{j=1}^n K_{h}(\bs_j, \bs)I(\bs_j \in \calE \cap  \calB(\bs, h)) + \frac{1}{n}\sum_{j=1}^n  I(\bs_j =\bs) }\\
    &= \frac{\eps_{1,n}(\bs) + \nu_{1,n}(\bs) }{\eps_{2,n}(\bs) + \nu_{2,n}(\bs) }.
\end{align*}



We look at each term individually and note that we have the edge components terms identical to the proof of  Proposition \ref{prop::zeroknot}, so 
\begin{align*}
    \eps_{1,n}(\bs) &= h \left\{ \sum_{\ell \in \calI}   m_\ell(0) g_\ell(0)   + O(h)+ O_p\left(\sqrt{\frac{1}{n h}}\right) \right\} =  O(h)+ O_p\left(\sqrt{\frac{h}{n }}\right), \\
    \eps_{2,n}(\bs) &= h \left\{ \sum_{\ell \in \calI}   g_\ell(0) + O(h)+ O_p\left(\sqrt{\frac{1}{n h}}\right) \right\} = O(h)+ O_p\left(\sqrt{\frac{h}{n }}\right).
\end{align*}

For the terms on the knots, they are just a sample average,
so
\begin{align*}
    \nu_{2,n}(\bs)  = p(\bs)+ O_p\left(\sqrt{\frac{1}{n}}\right)
\end{align*}
and similarly
\begin{align*}
        \nu_{1,n}(\bs) &= \frac{1}{n}\sum_{j=1}^n \left(m_\calS(\bs\right) + U_j)  I(\bs_j =\bs)\\
        &= m_\calS(\bs) p(\bs) + O_p\left(\sqrt{\frac{1}{n}}\right).
\end{align*}

With the fact that $O_p\left(\sqrt{\frac{1}{n }}\right)$ dominates $O_p\left(\sqrt{\frac{h}{n }}\right)$, 
we conclude
\begin{align*}
    \hat{m}(\bs) &= \frac{O(h)+ O_p\left(\sqrt{\frac{ h}{n}}\right) + m_\calS(\bs)p(\bs) + O_p\left(\sqrt{\frac{1}{n}}\right)}{ O(h)+ O_p\left(\sqrt{\frac{h}{n }}\right) + p(\bs) + O_p\left(\sqrt{\frac{1}{n}}\right)}\\
    &= \frac{O(h)+ O_p\left(\sqrt{\frac{1}{n }}\right) }{ O(h)+ O_p\left(\sqrt{\frac{1}{n}}\right) + p(\bs) } + \frac{m_\calS(\bs)p(\bs) }{ O(h)+ O_p\left(\sqrt{\frac{1}{n }}\right) + p(\bs) }\\
    &= \frac{O(h)+ O_p\left(\sqrt{\frac{1}{n }}\right) }{ p(\bs) } + O\left[\left(\frac{O(h)+ O_p\left(\sqrt{\frac{1}{n }}\right) }{ p(\bs) }\right)^2\right]\\
    &\ \ \ \ + m_\calS(\bs)p(\bs) \left\{ \frac{1}{p(\bs)} + \frac{O(h)+ O_p\left(\sqrt{\frac{1}{n }}\right)}{ p(\bs)^2}  \right\}\\
    &= m_\calS(\bs) +  O(h)+ O_p\left(\sqrt{\frac{1}{n }}\right),
\end{align*}
which completes the proof.


\end{proof}

\subsection{Dual Path Algorithm for Generalized Lasso Problem}
\label{sec:lassoSolution}

For the generalized Lasso problem:
\begin{align*}
     {\sf minimize}_{\bbeta} \norm{y - X \beta}_2^2 + \lam \norm{ D \beta}_1
\end{align*}
If $D$ is invertible or the matrix $D$ has dimension $m \times p$ with $\text{rank}(D) = m$, this can be converted into a standard Lasso problem by setting $\theta = D \beta$, and the problem reduces to
\begin{align*}
     {\sf minimize}_{\bbeta} \norm{y - X D^{-1} \theta}_2^2 + \lam \norm{ \theta}_1
\end{align*}
However, this is not the case for the incidence matrix with the number of edges larger than the number of nodes.
Hence, we turn to the Lagrange dual problem. Let $\text{rank}(D) = m$, then we want to solve
\begin{align*}
    \text{minimize}_{u \in \RR^m } \frac{1}{2} \brac{X^T y - D^T u}\brac{X^T X}^{+} \brac{X^T y - D^T u}\\
    \text{subject to} \norm{u}_{\infty} \leq \lambda, D^T u \in \text{row}(X)
\end{align*}
We then follow the dual path algorithm by \cite{SolutionPathLasso}.
For notation, use $A^+$ to denote the Moore-Penrose pseudo-inverse of matrix $A$, and use subscript $-\calB$ to index over all rows or coordinates except those in set $\calB$. The algorithm is described in Algorithm \ref{alg::lassoPath}.

\begin{algorithm}
\caption{Dual path algorithm for generalized Lasso problem}
\label{alg::lassoPath}
Start with $k = 0, \lam_0 = \infty, \calB = \emptyset, s = \emptyset$.
While $\lam_k > 0$:
\begin{algorithmic}
\State 1. Compute a solution at $\lam_k$ by least squares as
\begin{align}
    \hat{u}_{\lam_k, -\calB} = \brac{D_{-\calB} D_{-\calB}^T}^{+} D_{-\calB}\brac{y - \lam_k D_{\calB}^T s}
\end{align}
\State 2. Compute the next hitting time $h_{k+1}$ by
\begin{align}
    t_i^{(hit)} = \frac{\sbrac{ \brac{D_{-\calB} D_{-\calB}^T}^+ D_{-\calB} y}_i}{\sbrac{ \brac{D_{-\calB} D_{-\calB}^T}^+ D_{-\calB} D_{-\calB}^T s}_i \pm 1 }
\end{align}
where only one of $+1$ or $-1$ will yield a value in $\sbrac{0, \lam_k}$, and this is the ``hitting time'' of coordinate $i$. Hence the next hitting time is
\begin{align}
    h_{k+1} = \max_i t_i^{(hit)}
\end{align}
\State 3. Compute the next leaving time $\ell_{k+1}$ by first defining 
\begin{align}
    c_i &= s_i \cdot \sbrac{D_{\calB} \sbrac{I - D_{-\calB}^T \brac{D_{-\calB} D_{-\calB}^T}^+ D_{-\calB} } y }_i,\\
    d_i &=  s_i \cdot \sbrac{D_{\calB} \sbrac{I - D_{-\calB}^T \brac{D_{-\calB} D_{-\calB}^T}^+ D_{-\calB} } D_{\calB}^T s }_i
\end{align}
and then the leaving time of the $i$th boundary coordinate is
\begin{align}
    t_i^{(leave)} = 
    \begin{cases}
        c_i / d_i, \text{if } c_i <0 \text{ and } d_i <0,\\
        0, \text{otherwise}
    \end{cases}
\end{align}
Therefore, the next leaving time is
\begin{align}
    \ell_{k+1} = \max_i t_i^{(leave)}
\end{align}

\State 4. Set $\lam_{k+1} = \max \set{h_{k+1}, \ell_{k+1}}$. If $h_{k+1} > \ell_{k+1}$, then add the hitting coordinate to $\calB$ and its sign to $s$, otherwise remove the leaving coordinate to $\calB$ and its sign from $s$. Set $k = k+1$.
\end{algorithmic}
\end{algorithm}

\section{Additional Simulation Results}
\label{sec::extraSim}

\subsection{Vary Skeleton Cuts}
~~~~In this section, we examine the effect of cutting the skeleton into various numbers of disjoint components on the final regression performance. 
We use the same simulated datasets from Section \ref{sec::simulation}, including Yinyang data, Noisy Yinyang data, and SwissRoll data. 
The analysis procedure is mainly the same, where we use $5$-fold cross-validation SSE to evaluate the regression results for each dataset and repeat the process 100 times with randomly generated datasets. 
The main difference is that, during the skeleton construction step, we segment the skeleton graph into different disjoint components using single-linkage hierarchical clustering with respect to the Voronoi Density weights, as outlined in Section \ref{sec::skeletoncons}. 
We then fit and evaluate the skeleton-based regression methods on the skeletons that have been differently cut. 



\subsubsection{Vary Skeleton Cuts for Yinyang Data}
~~~~In this section, we investigate how cutting the skeleton into different numbers of disjoint components affects the performance of skeleton-based methods using the Yinyang data (from Section \ref{sec::Yinyang}). We randomly generate 1000-dimensional Yinyang data 100 times and use $5$-fold cross-validation to calculate the SSE on each dataset. We fit the skeleton-based methods in the same manner as in Section \ref{sec::simulation}, with the exception that the number of knots is fixed at 38 and we cut the initial graph into various numbers of disjoint components (ranging from 1 to 25) when constructing the skeleton. The median 5-fold cross-validation SSEs across the 100 datasets for different numbers of disjoint components are plotted in Figure \ref{fig::Yinyangd1000cut}.

Our results show that the S-Lspline method is sensitive to changes in the skeleton structure. In the case of Yinyang data, since there are 5 true disjoint structures in the covariate space, a cut of 5 results in the best regression performance. By design, S-Lspline regressors may incorporate unrelated information from one structure to another when an edge connects two structurally different areas, thus leading to a decline in the regression performance. For future research, incorporating edge weights into the S-Lspline regressor may help to mitigate the interference between different structures. The S-Kernel regressor also achieves optimal performance when the skeleton is segmented into 5 disjoint components. Skeleton-based kernel regression methods exhibit large changes in performance as the skeleton segmentation changes when the bandwidth is large. This is understandable as larger bandwidths allow more information from large distances, which are more likely to be non-informative as the segmentation changes. On the other hand, the S-kNN regressor has the best regression performance when the skeleton is left as a fully connected graph. This may be due to the locally adaptive nature of the k-nearest-neighbor method that ensures regression results are accurate as long as local neighborhoods are identified accurately.

\begin{figure}
\centering
\includegraphics[width = \textwidth]{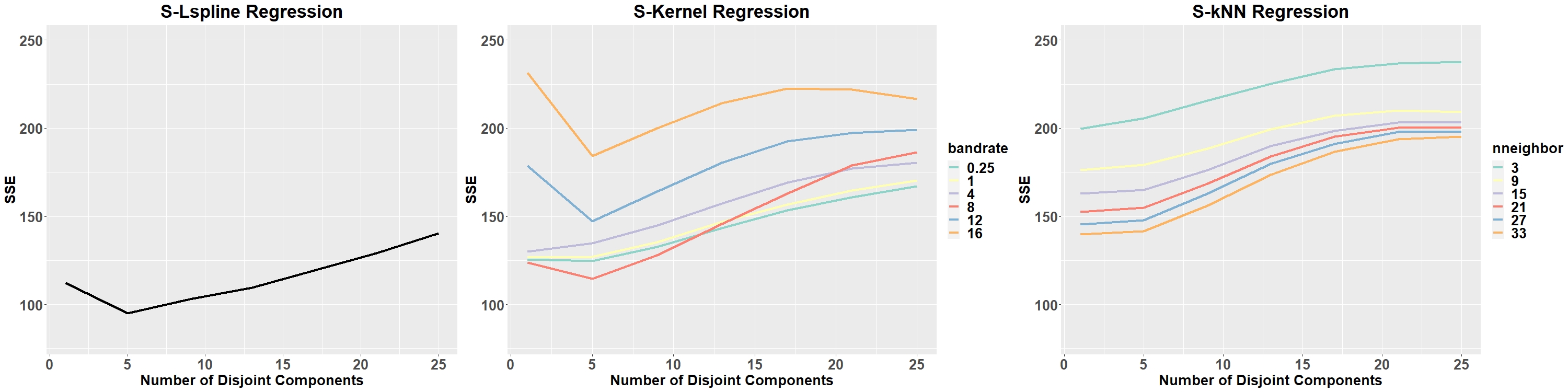}
\caption{Yinyang $d = 1000$ data skeleton regression results with the number of knots fixed as $38$ but segmented into varying numbers of disjoint components. The median SSE across the $100$ simulated datasets with each given parameter setting is plotted. 
}
\label{fig::Yinyangd1000cut}
\end{figure}

\subsubsection{Vary Skeleton Cuts for Noisy Yinyang Data}
~~~~We then evaluate the performance of the skeleton-based regression methods on the Noisy Yinyang data (from Section \ref{sec::NoisyYinyang}) when the skeletons are constructed with different numbers of disjoint components. Similarly, we randomly generate 1000-dimensional Noisy Yinyang data 100 times and use $5$-fold cross-validation to calculate the sum of squared errors (SSE) on each dataset. We fix the number of knots to be 71 and construct skeletons with different numbers of disjoint components. The median 5-fold cross-validation SSEs across the 100 datasets for different numbers of disjoint components are plotted in Figure \ref{fig::NoiseYinyangd1000cut}.

\begin{figure}
\centering
\includegraphics[width = \textwidth]{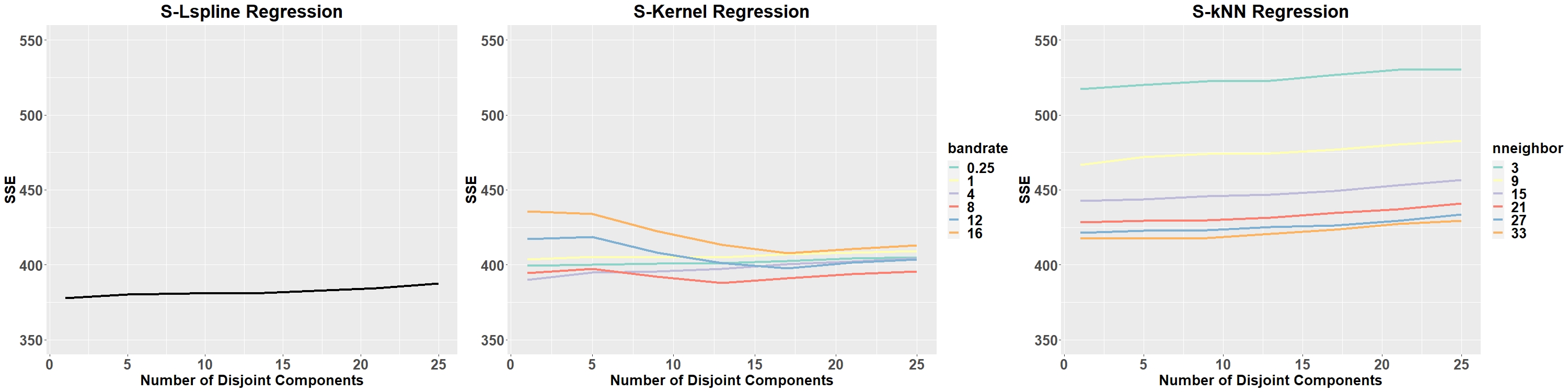}
\caption{Noise Yinyang $d = 1000$ data skeleton regression results with the number of knots fixed as $38$ but segmented into varying numbers of disjoint components. The median SSE across the $100$ simulated datasets with each given parameter setting is plotted. }
\label{fig::NoiseYinyangd1000cut}
\end{figure}

With the presence of noise, the S-Lspline method does not show significant variations in performance when the skeleton graph is cut into different disjoint components. The best regression result is obtained when the graph is left as a fully connected graph.
In contrast, the performance of the S-kernel method varies with the number of disjoint components. The best results, regardless of the bandwidth, are obtained when the skeleton is segmented into around 13 components, which is larger than the true number of 5 components in the data.
Lastly, the S-kNN method demonstrates an increase in SSE with an increase in the number of disjoint components.


\subsubsection{Vary Skeleton Cuts for SwissRoll data}
~~~~In this section, we evaluate the performance of skeleton-based methods on SwissRoll data (from Section \ref{sec::SwissRoll}) with skeletons cut into different numbers of disjoint components. Similarly, we randomly generate 1000-dimensional SwissRoll data 100 times and use $5$-fold cross-validation to calculate the sum of squared errors (SSE) on each dataset. We fix the number of knots to be 70 and construct skeletons with different numbers of disjoint components. The median 5-fold cross-validation SSEs across the 100 datasets for different numbers of disjoint components are plotted in Figure \ref{fig::SwissRolld1000cut}.

We find that the S-Lspline regressor is sensitive to changes in the skeleton structure, with the best regression results obtained when the skeleton is constructed as one connected graph. This makes sense as the covariates lie on one connected manifold.
The S-Kernel regressor also performs best on the fully connected skeleton. After an initial increase in SSE as the number of disjoint components increases, the SSE of the S-kernel regressor remains relatively stable.
The S-kNN regressor also achieves the best regression performance when the skeleton is left as a fully connected graph.
Overall, the SSE of the S-kNN regressor increases with the number of disjoint components, but for a small number of neighbors, there can be a decrease in SSE when the skeleton is cut into more disjoint components. One possible explanation is that, as the response function has discontinuous changes, segmenting the covariate space into more fragments can improve estimation in regions where the response changes abruptly.

\begin{figure}
\centering
\includegraphics[width = \textwidth]{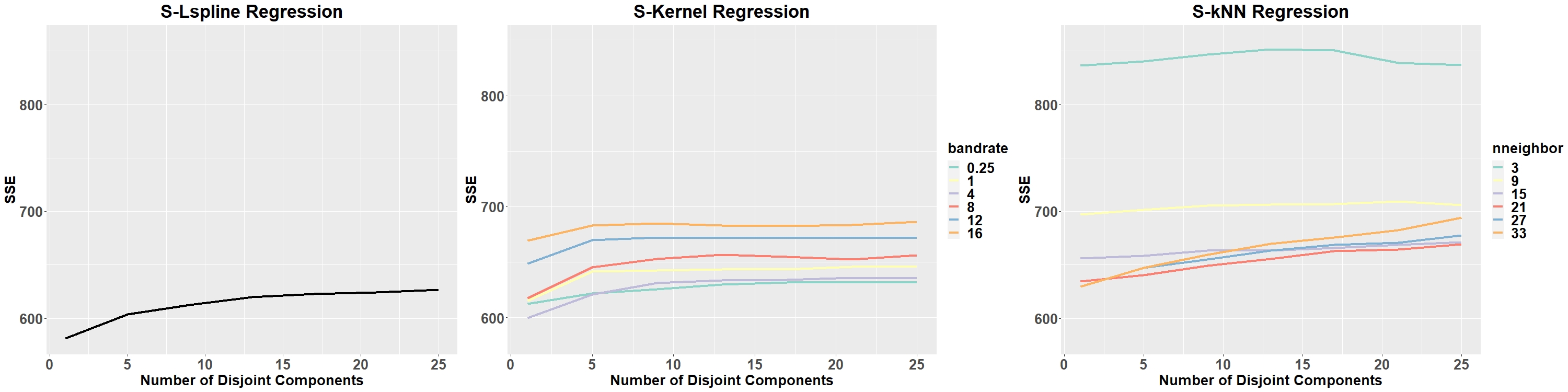}
\caption{SwissRoll $d = 1000$ data skeleton regression results with the number of knots fixed as $70$ but segmented into varying numbers of disjoint components. The median SSE across the $100$ simulated datasets with each given parameter setting is plotted.}
\label{fig::SwissRolld1000cut}
\end{figure}



\subsection{Penalized S-Lspline Empirical Results}
\label{sec:penalSplineSims}

~~~~In this section, we present the results of the S-Lspline regression with penalizations as introduced in Section \ref{sec:penalLspline}.
We use the same datasets as the simulations in Section \ref{sec::simulation} and follow the analysis procedure but with the regression methods to be the S-Lspline method with different types of penalties and varying penalization parameters $\lambda$.
Empirically, we observe that penalization does not improve the regression results of the S-Lspline model. 
This can be due to that the skeleton graph is already a summarizing presentation of the data with a concise structure, and the S-Lspline model based on the skeleton inherits the representational power and is not complex in nature, and hence penalization does not improve much for this method.

The results of the penalized S-Lspline methods on the Yinyang $d=1000$ data are summarized in Table \ref{table:Yinyangd1000PenalLspline} with the plots illustrating the effect of a varying number of knots and varying penalty parameters shown in Figure \ref{fig::Yinyangd1000PenalLspline}.

\begin{figure}
\begin{tabular}{l l c l} 
 \hline
 Method & Median SSE ($5$\%, $95$\%)  & lambda & nknots \\ [1ex] 
 \hline
  No Penalization& 110.4  (105.6,  118.5)  & - & 38 \\ 
  Laplacian Smoothing Order 0 & 110.4 (105.6, 118.5)  & 0.001 & 38 \\ 
 Laplacian Smoothing Order 1 & 111.2 (104.4, 117.8) & 0.001 & 38\\
 Laplacian Smoothing Order 2 & 110.2 (104.7, 118.9) & 0.001& 38\\
 Trend Filtering Order 0 & 110.6 (105.6, 118.9) & 0.001 &38\\
 Trend Filtering Order 1 & 111.3 (104.4, 118.0) & 0.001 & 38 \\
  Trend Filtering Order 2 & 110.3 (104.7, 120.5) & 0.001 & 38 \\[1ex] 
 \hline
\end{tabular}
\captionof{table}{S-Lspline regression results on Yinyang $d=1000$ data with varying penalties and parameters. The smallest medium 5-fold cross-validation SSE from each method is listed with the corresponding parameters used. The $5$th percentile and $95$th percentile of the SSEs from the given parameter settings are reported in brackets.}
\label{table:Yinyangd1000PenalLspline}
\vspace{2em}

\centering
\includegraphics[width=\textwidth]{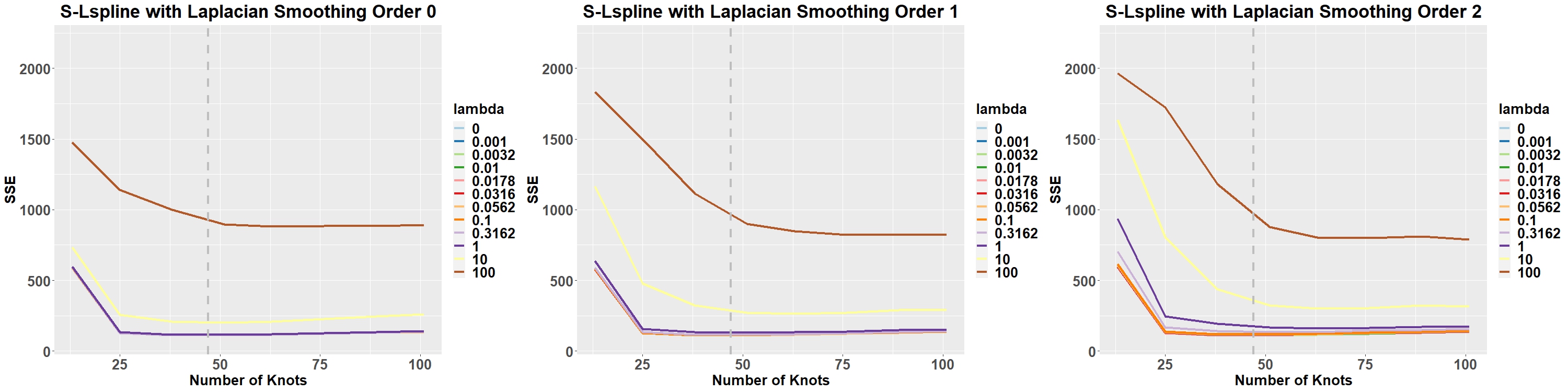}\\
\includegraphics[width=\textwidth]{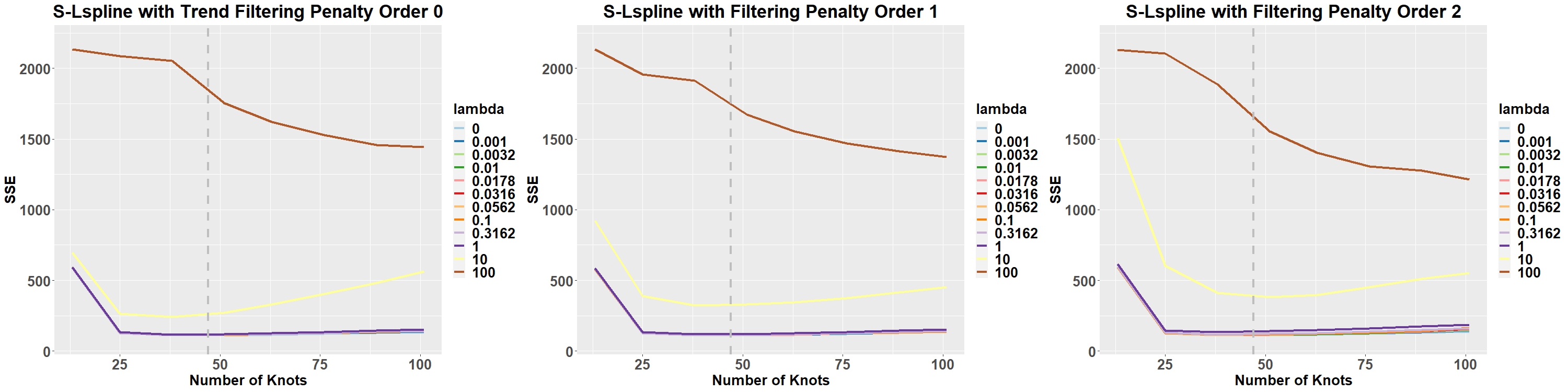}
\captionof{figure}{S-Lspline regression results on Yinyang $d=1000$ data with varying penalties and parameters. The median SSE across the $100$ simulated datasets with each given parameter setting is plotted.}
\label{fig::Yinyangd1000PenalLspline}

\end{figure}


The results of the penalized S-Lspline methods on the Noisy Yinyang $d=1000$ data are summarized in Table \ref{table:NoiseYinyangd1000PenalLspline} with plots in Figure \ref{fig::NoiseYinyangd1000PenalLspline}. We observe that adding penalization does not improve the regression results. 

\begin{figure}
\begin{tabular}{l l c l} 
 \hline
 Method & Median SSE ($5$\%, $95$\%)  & lambda & nknots \\ [1ex] 
 \hline
  No penalization & 375.0 (354.7, 396.3) & - &  57  \\ 
 Laplacian Smoothing Order 0 & 377.0 (355.1,394.4)  & 0.001 & 42 \\ 
 Laplacian Smoothing Order 1 & 375.0 (354.7, 396.3) & 0.001 & 57\\
 Laplacian Smoothing Order 2 & 377.2 (355.9, 403.3) & 0.001& 71\\
 Trend Filtering Order 0 & 377.0 (355.1, 394.4) & 0.003 &42\\
 Trend Filtering Order 1 & 375.0 (354.7, 398.0) & 0.001 & 57 \\
  Trend Filtering Order 2 & 378.0,  (355.9, 399.0) & 0.001 & 42 \\[1ex] 
 \hline
\end{tabular}
\captionof{table}{S-Lspline regression results on Noisy Yinyang $d=1000$ data with varying penalties and parameters. The smallest medium 5-fold cross-validation SSE from each method is listed with the corresponding parameters used. The $5$th percentile and $95$th percentile of the SSEs from the given parameter settings are reported in brackets.}
\label{table:NoiseYinyangd1000PenalLspline}
\vspace{2em}
\centering
\includegraphics[width=\textwidth]{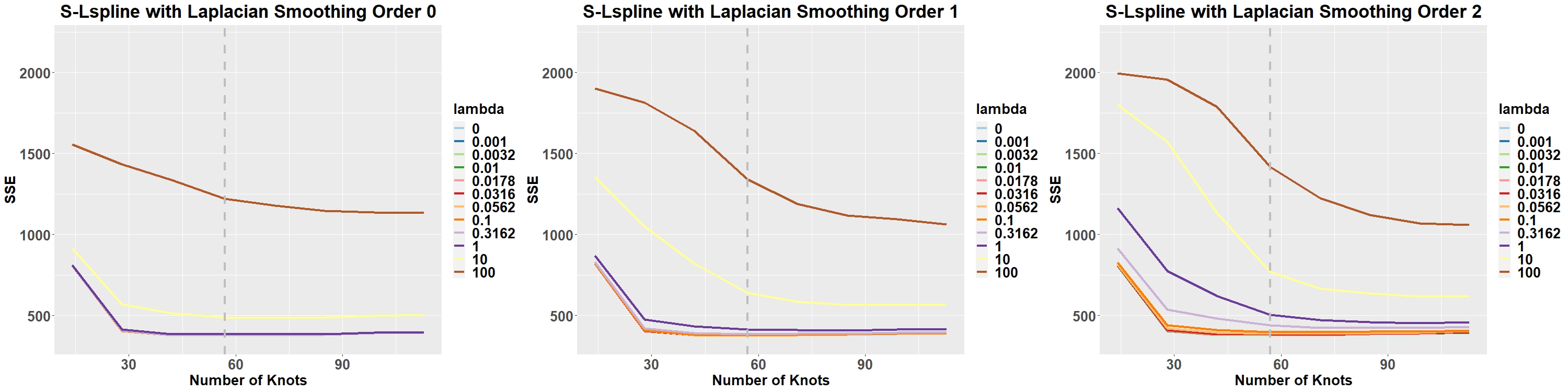}\\
\includegraphics[width=\textwidth]{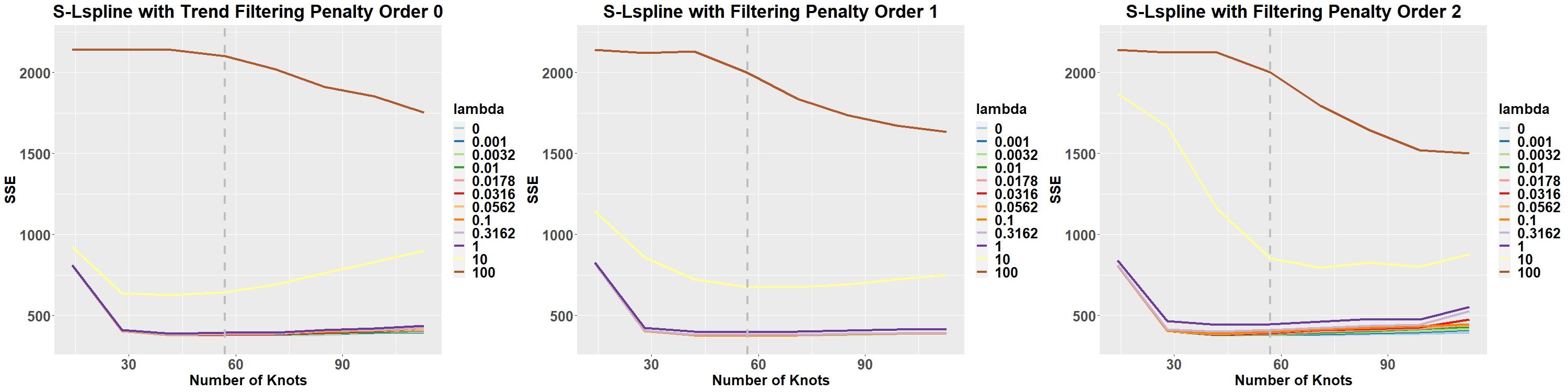}
\captionof{figure}{S-Lspline regression results on Noisy Yinyang $d=1000$ data with varying penalties and parameters. The median SSE across the $100$ simulated datasets with each given parameter setting is plotted.}
\label{fig::NoiseYinyangd1000PenalLspline}

\end{figure}


The results of the penalized S-Lspline methods on the Swill Roll $d=1000$ data are summarized in Table \ref{table:SwissRolld1000PenalLspline} with plots in Figure \ref{fig::SwissRolld1000PenalLspline}. Including penalization terms for the linear spline model does not improve the regression results in this setting. 

\begin{figure}
\begin{tabular}{l l c l} 
 \hline
 Method & Median SSE ($5$\%, $95$\%)  & lambda & nknots \\ [1ex] 
 \hline
  No penalization & 572.7 (521.0, 640.9) & - &  60  \\ 
 Laplacian Smoothing Order 0 & 573.4 (521.4, 632.6) & 0.001 & 60 \\ 
 Laplacian Smoothing Order 1 & 583.7 (524.7, 633.1) & 0.001 & 60\\
 Laplacian Smoothing Order 2 & 573.1 (521.5, 641.3) & 0.001& 60\\
 Trend Filtering Order 0 & 580.2 (524.6, 685.7)  & 0.01 &60\\
 Trend Filtering Order 1 & 583.7,  (524.6, 633.0) & 0.001 & 60 \\
  Trend Filtering Order 2 & 574.4,  (522.0,  657.1) & 0.001 & 60 \\[1ex] 
 \hline
\end{tabular}
\captionof{table}{S-Lspline regression results on Swiss Roll $d=1000$ data with varying penalties and parameters. The smallest medium 5-fold cross-validation SSE from each method is listed with the corresponding parameters used. The $5$th percentile and $95$th percentile of the SSEs from the given parameter settings are reported in brackets.}
\label{table:SwissRolld1000PenalLspline}
\vspace{2em}
\centering
\includegraphics[width=\textwidth]{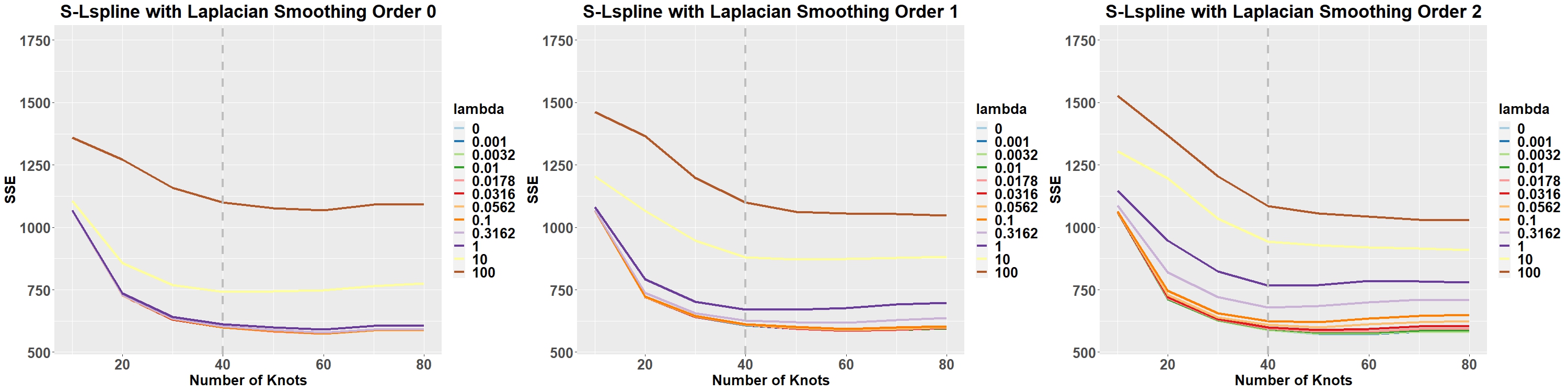}\\
\includegraphics[width=\textwidth]{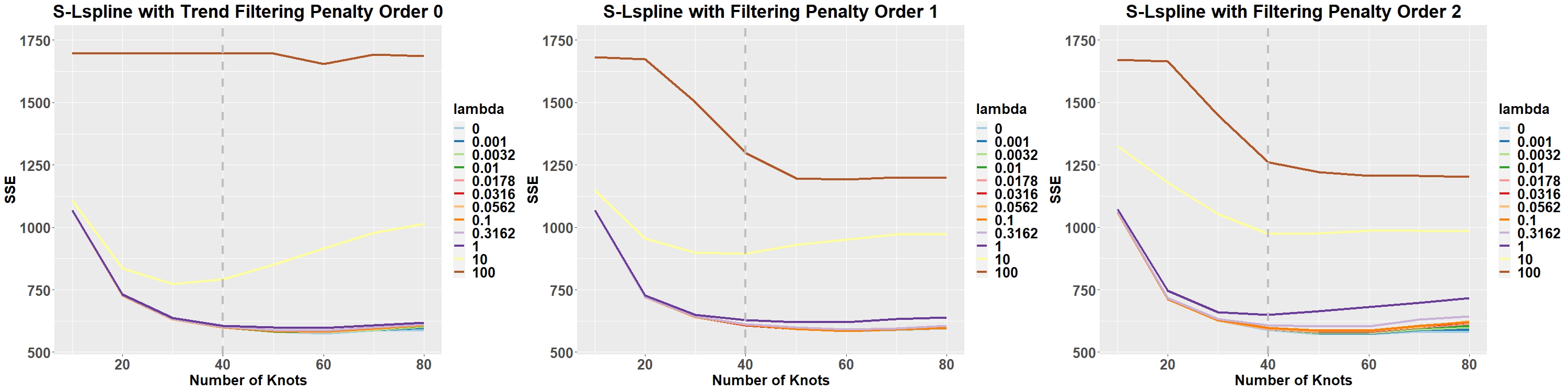}
\captionof{figure}{S-Lspline regression results on Swill Roll $d=1000$ data with varying penalties and parameters. The median SSE across the $100$ simulated datasets with each given parameter setting is plotted.}
\label{fig::SwissRolld1000PenalLspline}

\end{figure}

Overall, we observe that adding penalization terms to the linear spline model based on the skeleton graph has minimal effect on the regression performance, and having no penalization actually gives the best results although by a very tiny margin.
We also test the penalized versions of the S-Lspline method on the real data examples (COIL-20 and SDSS datasets) and similarly observe that having penalization terms only leads to minimal effects on the regression performance much and the vanilla version of the S-Lspline can give the best result for most of the times.

\subsection{Multi-Layer Perceptron Autoencoder Regressor}
\label{sec:MLPAE}

This subsection summarizes the multi-layer perceptron (MLP) Autoencoder methodology and experimental results on high-dimensional synthetic data regression. 

\subsubsection{Autoencoder and Regression Architecture}

~~~~The MLP Autoencoder implementation uses two separate \texttt{MLPRegressor} networks\footnote{\url{https://scikit-learn.org/stable/modules/generated/sklearn.neural_network.MLPRegressor.html}} as encoder and decoder. The encoder maps high-dimensional input to a low-dimensional bottleneck representation, while the decoder reconstructs the original input from the bottleneck. 

We experiment with a 3-layer architecture and a 4-layer architecture .
The 3-layer architecture performs direct compression and decompression, that the encoder maps directly from input dimension $D$ to hidden dimension $h$, and the decoder maps directly from $h$ back to $D$:
\begin{equation}
\text{Input } (D=1000) \rightarrow \text{Hidden } (h) \rightarrow \text{Output } (D=1000)
\end{equation}

The 4-layer architecture uses symmetric intermediate layers for gradual compression, which allows for more hierarchical feature learning through gradual dimensionality reduction.:
\begin{equation}
\text{Input } (D) \rightarrow \text{Hidden}_1 \left(\frac{D+h}{2}\right) \rightarrow \text{Hidden } (h) \rightarrow \text{Hidden}_1 \left(\frac{D+h}{2}\right) \rightarrow \text{Output } (D)
\end{equation}

The Ridge regression model predicts the response $Y$ from the learned embeddings with penalty $\alpha$:
\begin{equation}
\hat{Y} = \text{Ridge}(\text{Encoder}(X); \alpha)
\end{equation}

\subsubsection{Analysis Procedure}

~~~~The autoencoder training uses an alternating optimization strategy by first initialize hidden representation with small random values, and then iterate between training encoder to map input to hidden representation and training decoder to reconstruct input from hidden representation. The model uses ReLU activation and uses Adam optimizer as the solver, with learning rate $0.001$, maximum iterations $1000$, and stops for a validation fraction of $0.1$ with patience $30$.

Similar to the other simulation data studies, we randomly regenerate each simulated dataset for analysis, and for each replication we explore different combinations of hidden dimensions $h \in \{2, 5, 10, 20, 50\}$ and penalization parameters  $\alpha \in \{0.001, 0.01, 0.1, 1.0, 10.0, 100.0\}$.
We use 5-fold cross validation to measure the model performance with the Sum of Squared Errors (SSE) as the regression evaluation metric.

\subsubsection{Noisy Yinyang Regression Data}

We first show the results on the noisy Yinyang dataset as described in Section \ref{sec::NoisyYinyang}. For this challenging dataset that need model to discern underlying structures among noise, we fit autoencoder with both the 3-layer archetecture and the 4-layer archetecture. For the 3-layer architecture, we randomly regenerate the dataset $100$ times like the other methods included in this paper. However, for the 4-layer architecture, due to the increased computation cost, we only replicate the analysis procedure with hyperparameter grid search 5 times on the regenerated dataset. The results are presented below in Figure \ref{fig::MLP_AE_NoiseYinyangd1000}.

\begin{figure}[ht]
\centering
\includegraphics[width=\textwidth]{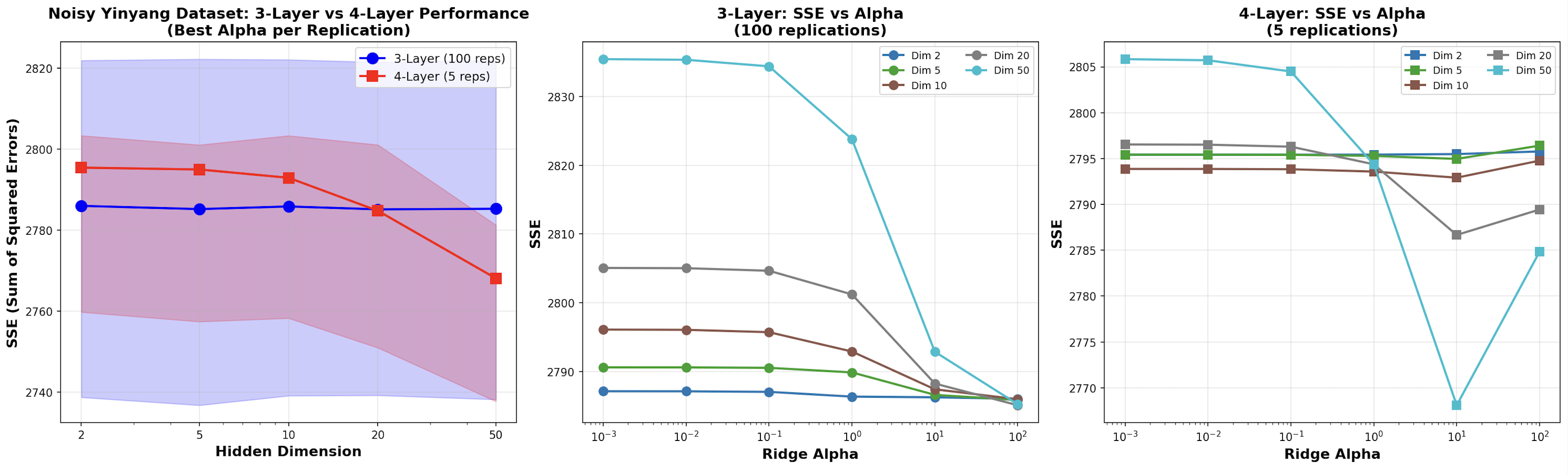}
\captionof{figure}{Regression results on the Noisy Yinyang $d=1000$ data with varying hidden dimensions and penalization paramters. The median SSE across the regenerated datasets is plotted, with the band in the first subplot for the 5th and 95th quantile of the SSE.}
\label{fig::MLP_AE_NoiseYinyangd1000}
\end{figure}

Overvall we see larger regularization gives better performance. Note that although the 4-layer architecture can have more potential with higher hidden dimensions, the SSE performance is not comparable to the other methods as described in Section \ref{sec::NoisyYinyang} and in scale similar to the 3-layer architecture. Hence, for sake of computation, we focus on training the 3-layer autoencoder for the other simulated datasets for analysis.

\subsubsection{Yinyang Regression Data}
The analysis results on the Yinyang regression dataset as described in Section \ref{sec::Yinyang} using the 3-layer autoencoder architecture is shown in Figure \ref{fig::MLP_AE_Yinyangd1000}.
\begin{figure}[ht]
\centering
\includegraphics[width=\textwidth]{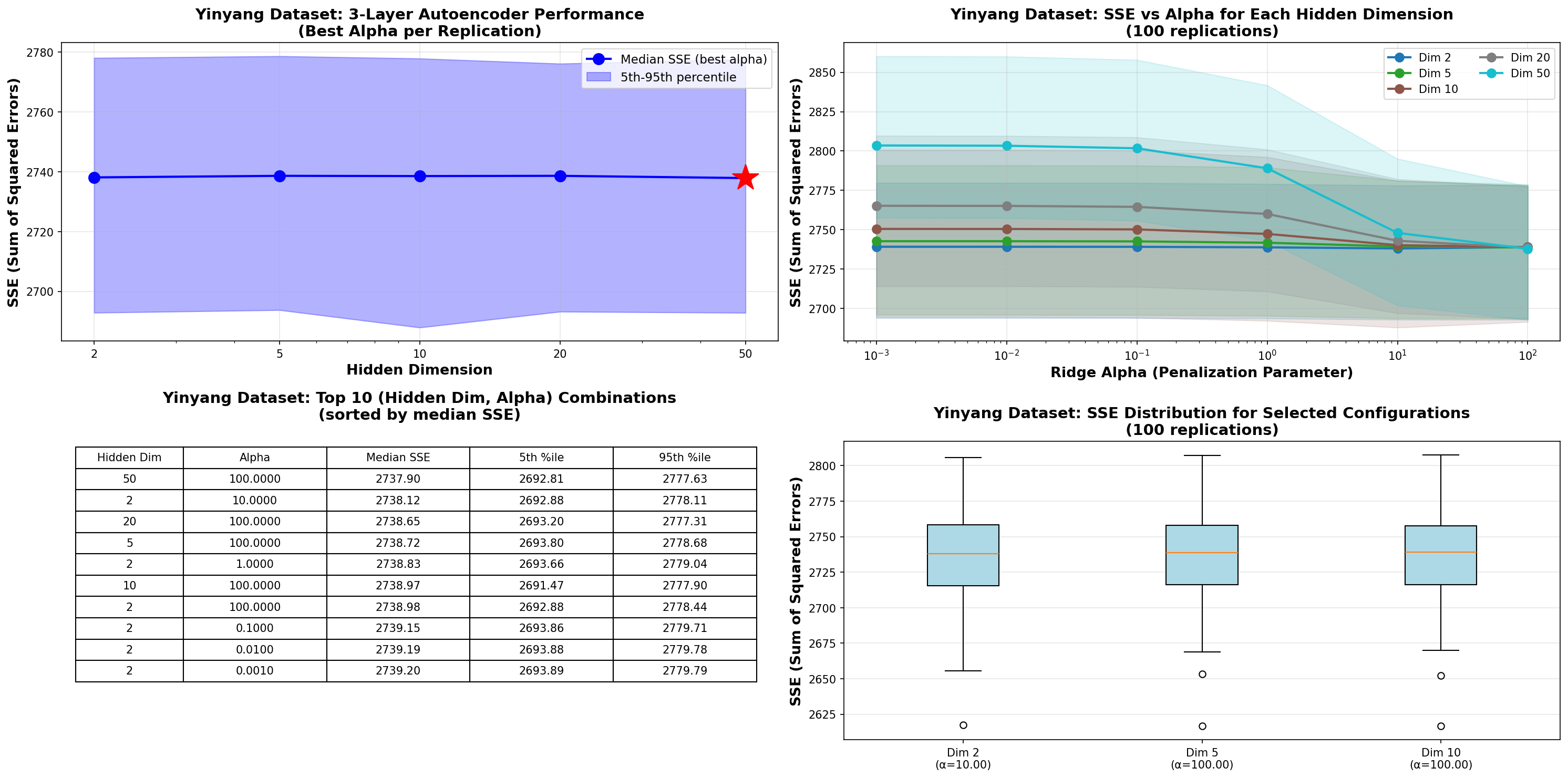}
\captionof{figure}{ Regression results on the Yinyang $d=1000$ data with varying hidden dimensions and penalization parameters.}
\label{fig::MLP_AE_Yinyangd1000}
\end{figure}

Note that penalization parameter has some effect on the regression performance when the hidden dimension is high. However, the best performance is always achieved with large regularization.

\subsubsection{Swiss Roll Regression Data}
The analysis results on the Swiss Roll regression dataset as described in Section \ref{sec::SwissRoll} using the 3-layer autoencoder architecture is shown in Figure \ref{fig::MLP_AE_SwissRolld1000} below.

\begin{figure}[ht]
\centering
\includegraphics[width=\textwidth]{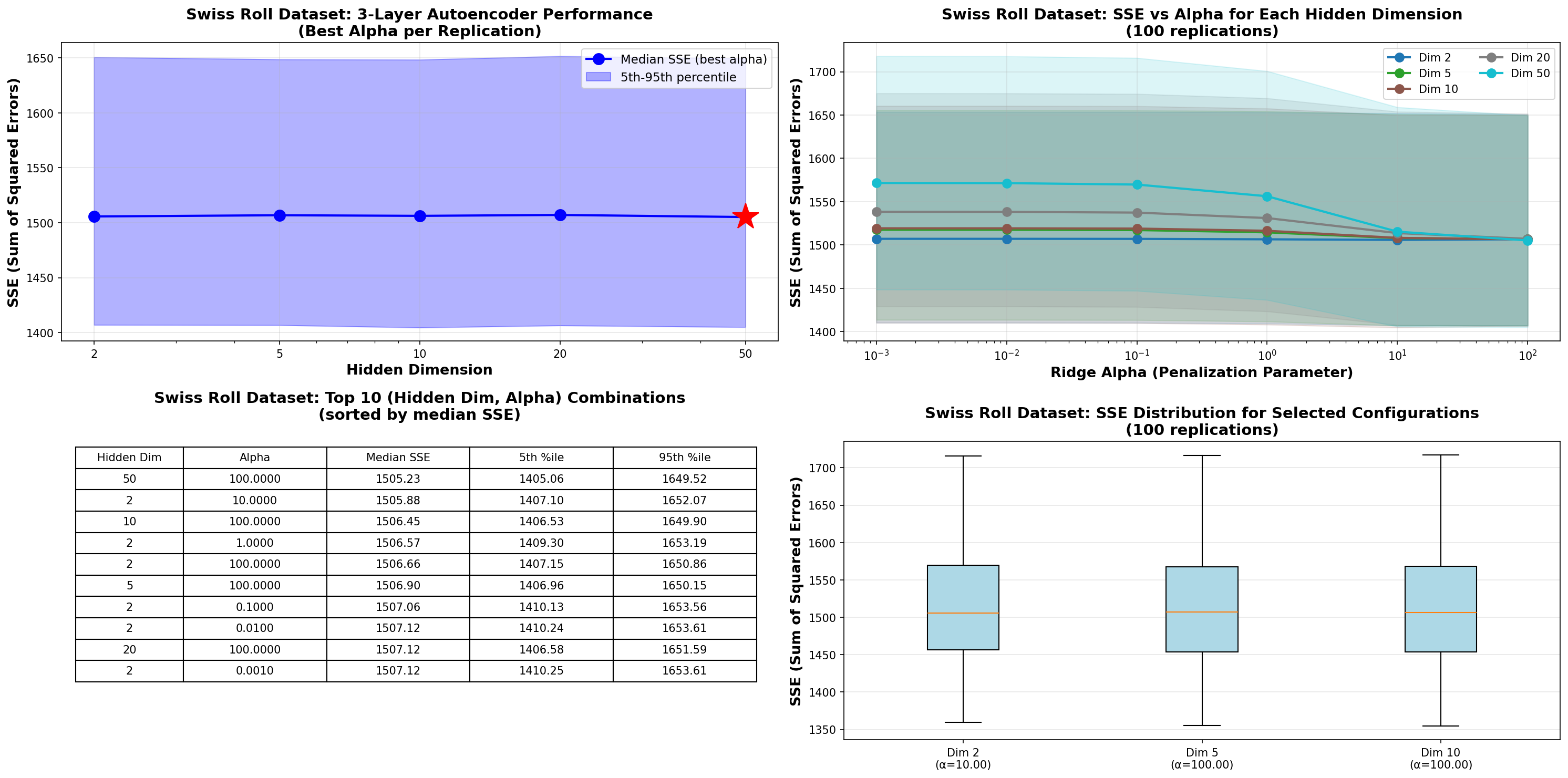}
\captionof{figure}{ regression results on Swill Roll $d=1000$ data with varying penalties and parameters. The median SSE across the $100$ simulated datasets with each given parameter setting is plotted.}
\label{fig::MLP_AE_SwissRolld1000}
\end{figure}

\newpage
\subsection{Low Dimensional Simulation Data Examples}
For the simulations in Section \ref{sec::simulation}, we add random Gaussian variables to make the datasets have a large dimension, and illustrate that the skeleton regression framework can have advantages in dealing with such large-dimensional datasets. 
In this section, we look instead at the low-dimensional settings where the datasets do not have noisy variables. We show that the skeleton-based regression methods can also have competitive performance in such settings compared to existing regression approaches.

\subsubsection{Yinyang Data d = 2}
\label{sec::Yinyangd2}

~~~~We use the same data generation mechanism as in Section \ref{sec::Yinyang} but without additional noisy variables (having $d = 2$).
We follow the same analysis procedure as in Section \ref{sec::Yinyang} and take the median, 5th percentile, and 95th percentile of the 5-fold cross-validation Sum of Squared Errors (SSEs) for each parameter setting of each method over the 100 simulated datasets. We present the smallest median SSE for each method in Table \ref{table:Yinyangd2} along with the corresponding best parameter setting. The plots illustrating the effect of varying the numbers of knots are included in Figure \ref{fig::Yinyangd2Numknots}.

\begin{figure}
\begin{tabular}{c l c l} 
 \hline
 Method & Median SSE ($5$\%, $95$\%) & nknots & Parameter \\ [1ex] 
 \hline
 kNN & 60.1 (57.1, 63.0) & - & neighbor=36 \\ 
 Ridge & 1355.3  (1312.0,  1392.5) & & $\lambda  = 0.001$\\
 Lasso & 1354.8 (1311.4, 1391.9) & & $\lambda = 0.001$\\
 SpecSeries & 71.2  (67.2, 74.1) & - &bandwidth = 0.1\\
 Fast-KRR & 115.9 (108.4, 124.2) & - & $\sigma$ = 1\\
 S-Kernel & 60.4 (57.3,  64.7) & 63 & bandwidth = 4 $r_{hns}$ \\
 S-kNN & 60.9 (58.0, 63.9) & 76 & neighbor = 36 \\
 S-Lspline & 60.0 (57.0, 63.7) & 63 & -  \\[1ex] 
 \hline
\end{tabular}
\captionof{table}{Regression results on Yinyang $d=2$ data. The smallest medium 5-fold cross-validation SSE from each method is listed with the corresponding parameters used. The $5$th percentile and $95$th percentile of the SSEs from the given parameter settings are reported in brackets.}
\label{table:Yinyangd2}
\vspace{2em}

\centering
\includegraphics[width=\textwidth]{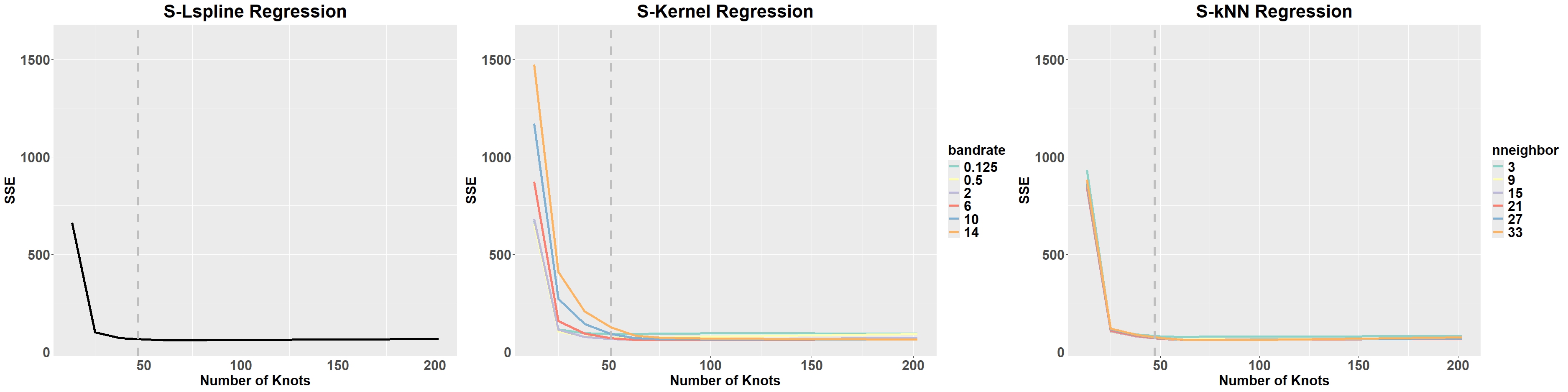}
\captionof{figure}{Yinyang $d = 2$ data regression results with varying number of knots. The median SSE across the $2$ simulated datasets with each given parameter setting is plotted.}
\label{fig::Yinyangd2Numknots}

\end{figure}

We observe that all the skeleton-based methods (S-Kernel, S-kNN, and S-Lspline) have performance comparable to the standard kNN in this setting. Ridge and Lasso regression, despite the regularization effect, resulted in relatively high SSEs. 
The SpecSeries method as a spectral approach and the Fast-KRR method as a kernel machine learning approach have improved performance compared to the classical Ridge and Lasso penalization regression methods, but do not give the top performances. This can be due to the underlying data structure being comprised of multiple disconnected components which diminishes the power of such manifold learning methods.
Therefore, the skeleton regression framework also gives competitive performance in datasets without noisy variables, but the advantage of skeleton-based methods is manifested more if the number of noisy variables in the input vector gets larger (see Appendix \ref{sec::Yinyang}).


\subsubsection{Noisy Yinyang Data d = 2}
\label{sec::NoisyYinyangd2}
~~~~We follow the same data generation mechanism as in Section \ref{sec::NoisyYinyang} to get Noisy Yinyang data without the additional variable dimensions and only have $d = 2$. 
We follow the same analysis procedure as in Section \ref{sec::NoisyYinyang} and take the median, 5th percentile, and 95th percentile of the 5-fold cross-validation Sum of Squared Errors (SSEs) for each parameter setting of each method over the 100 simulated datasets. We present the smallest median SSE for each method in Table \ref{table:NoiseYinyangd2} along with the corresponding best parameter setting, with the plots illustrating the effect of varying the numbers of knots included in Figure \ref{fig::Yinyangd1000Numknots}.

\begin{figure}
\centering
\begin{tabular}{c l c l} 
 \hline
 Method & Median SSE ($5$\%, $95$\%)  & Number of knots & Parameter \\ [1ex] 
 \hline
 kNN & 228.5 (213.0,  244.3) & -  & neighbor=6 \\ 
 Ridge & 1938.5  (1906.0, 1973.0) &- & $\lambda  = 0.005$\\
 Lasso &  1938.6 (1905.8, 1972.9) &- & $\lambda = 0.0016$\\
 SpecSeries & 243.9 (229.0, 259.0) & - & bandwidth = $0.1$\\
 Fast-KRR & 497.8 (475.3, 520.6) & - & $sigma$ = 1 \\
 S-Kernel & 239.0 (223.2, 254.4) & 226 & bandwidth = 4 $r_{hns}$ \\
 S-kNN & 250.7 (234.4,  264.7) & 226 & neighbor = 9 \\
 S-Lspline & 234.3 (218.5 249.8) & 226 &  - \\[1ex] 
 \hline
\end{tabular}
\captionof{table}{Regression results on Noisy Yinyang $d=2$ data.The smallest medium 5-fold cross-validation SSE from each method is listed with the corresponding parameters used. The $5$ percentile and $95$ percentile of the SSEs from the given parameter settings are reported in brackets.}
\label{table:NoiseYinyangd2}
\vspace{2em}
\centering
\includegraphics[width=\textwidth]{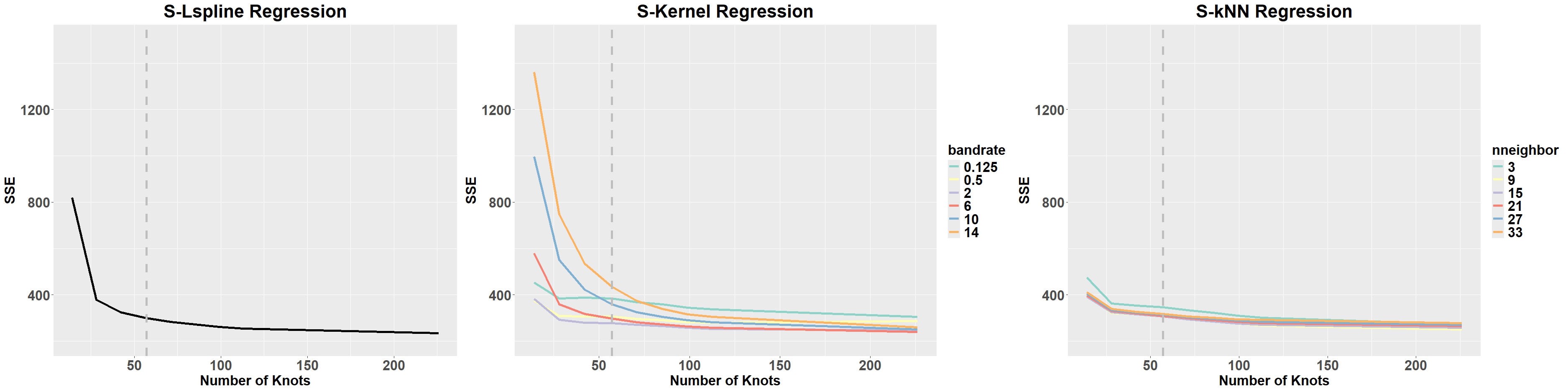}
\captionof{figure}{Noisy Yinyang $d = 1000$ data regression results with varying number of knots. The median SSE across the $100$ simulated datasets with each given parameter setting is plotted.}
\label{fig::NoiseYinyangd2Numknots}

\end{figure}


We observe similar patterns as shown in the Yinyang data that all the skeleton-based methods (S-Kernel, S-kNN, and S-Lspline) have performance comparable to the standard kNN in this setting. Ridge and Lasso regression methods result in relatively high SSEs, while the SpecSeries method and the Fast-KRR method have improved performance compared to the classical Ridge and Lasso penalization regression methods. 
Notably, the Spectral Series regression gives a top-level performance in this setting, and this can be due to the noisy observations adding a uniform density to the data space and making the structures in the data connected.


\subsubsection{SwissRoll Data d = 3}
\label{sec::SwissRolld3}
~~~~We follow the same data generation mechanism as in Section \ref{sec::SwissRoll} but with $d = 3$ so that no random normal variables are added as additional features.
We present the smallest median SSE for each method in Table \ref{table:Swissd3} along with the corresponding best parameter setting, and the plots illustrating the effect of varying the numbers of knots are included in Figure \ref{fig::Yinyangd1000Numknots}.

\begin{figure}

\begin{tabular}{c l c l} 
 \hline
 Method & Median SSE ($5$\%, $95$\%)  & nknots & Parameter \\ [1ex] 
 \hline
 kNN & 309.1 (281.5, 342.1) & - & neighbor=6 \\ 
 Ridge &  1123.8  (1054.6,  1202.4) & -& $\lambda  = 0.00126$\\
 Lasso &  1123.3 (1053.9, 1201.3) & -& $\lambda = 0.0025$\\
 SpecSeries & 331.8 (307.3, 351.8) &- & bandwidth = $0.1$\\
 Fast-KRR & 563.8 (533.0, 598.6) &- & $\sigma = 1$\\
 S-Kernel & 348.1 (345.5, 365.3) & 320 & bandwidth = 8 $r_{hns}$ \\
 S-kNN & 363.5 (361.6, 406.7) & 320 & neighbor = 9 \\
 S-Lspline & 368.9 (329.4, 409.1) & 160 & - \\[1ex] 
 \hline
\end{tabular}
\captionof{table}{Regression results on SwissRoll $d=3$ data. The smallest medium 5-fold cross-validation SSE from each method is listed with the corresponding parameters used. The $5$ percentile and $95$ percentile of the SSEs from the given parameter settings are reported in brackets.}
\label{table:Swissd3}
\vspace{2em}
\centering
\includegraphics[width=\textwidth]{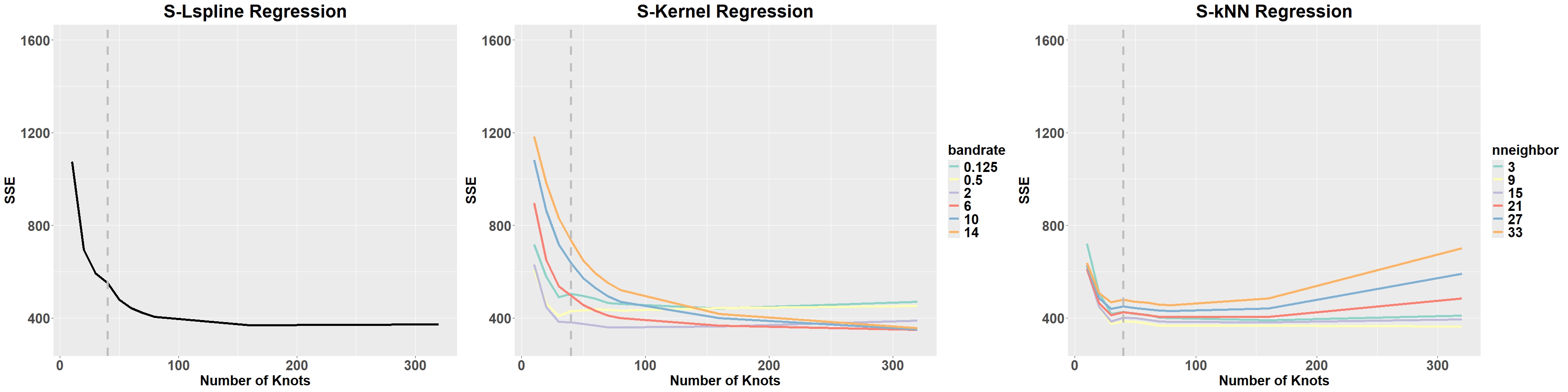}
\captionof{figure}{SwissRoll $d = 3$ data regression results with varying number of knots. The median SSE across the $100$ simulated datasets with each given parameter setting is plotted.}
\label{fig::SwissRolld3}

\end{figure}

The kNN method has the best performance in this setting, and the skeleton-based methods have comparable performance.
Note that the Spectral Series approach has performance slightly better than the skeleton-based methods in this case, which does corroborate its effectiveness in dealing with connected and smooth manifolds.

\end{appendix}

\begin{acks}[Acknowledgments]
The authors would like to thank the anonymous referees, an Associate
Editor and the Editor for their constructive comments that improved the
quality of this paper.
\end{acks}

\begin{funding}
Jerry Wei is supported by NSF Grant DMS - 2112907.
Yen-Chi Chen is supported by NSF DMS-195278, 2112907, 2141808, and NIH U24-AG072122.
\end{funding}

\begin{supplement}
\stitle{Code for Empirical Results}
\sdescription{We include the essential R code files and data files in the code supplement.
We also include a convenient R package implementing the skeleton regression methods.
The R package can also be found online at \url{https://github.com/JerryBubble/skeletonMethods} with instructions and examples, and the Python package implementation can be found at \url{https://pypi.org/project/skeleton-methods/}.}
\end{supplement}

\newpage
\bibliographystyle{imsart-nameyear}
\bibliography{skelreg.bib}

@InProceedings{GraphLaplacianSmoothing,
author="Smola, Alexander J.
and Kondor, Risi",
editor="Sch{\"o}lkopf, Bernhard
and Warmuth, Manfred K.",
title="Kernels and Regularization on Graphs",
booktitle="Learning Theory and Kernel Machines",
year="2003",
publisher="Springer Berlin Heidelberg",
address="Berlin, Heidelberg",
pages="144--158",
isbn="978-3-540-45167-9"
}

@article{SolutionPathLasso,
author = {Ryan J. Tibshirani and Jonathan Taylor},
title = {{The solution path of the generalized lasso}},
volume = {39},
journal = {The Annals of Statistics},
number = {3},
publisher = {Institute of Mathematical Statistics},
pages = {1335 -- 1371},
year = {2011},
doi = {10.1214/11-AOS878},
URL = {https://doi.org/10.1214/11-AOS878}
}

@article{loh2014fifty,
  title={Fifty years of classification and regression trees},
  author={Loh, Wei-Yin},
  journal={International Statistical Review},
  volume={82},
  number={3},
  pages={329--348},
  year={2014},
  publisher={Wiley Online Library}
}

@book{breiman2017classification,
  title={Classification and regression trees},
  author={Leo Breiman and Jerome Friedman and R.A. Olshen and Charles J. Stone},
  year={1984},
  publisher={ Chapman and Hall/CRC}
}

@article{torrence1998practical,
  title={A practical guide to wavelet analysis},
  author={Torrence, Christopher and Compo, Gilbert P},
  journal={Bulletin of the American Meteorological society},
  volume={79},
  number={1},
  pages={61--78},
  year={1998},
  publisher={American Meteorological Society}
}

@article{wahba1975smoothing,
  title={Smoothing noisy data with spline functions},
  author={Wahba, Grace},
  journal={Numerische mathematik},
  volume={24},
  number={5},
  pages={383--393},
  year={1975},
  publisher={Springer}
}

@book{wang2011smoothing,
  title={Smoothing splines: methods and applications},
  author={Wang, Yuedong},
  year={2011},
  publisher={CRC press}
}

@book{wasserman2006all,
  title={All of nonparametric statistics},
  author={Wasserman, Larry},
  year={2006},
  publisher={Springer Science \& Business Media}
}

@book{fan2018local,
  title={Local polynomial modelling and its applications},
  author={Fan, Jianqing and Gijbels, Irene},
  year={2018},
  publisher={Routledge}
}

@article{alam2015eleventh,
  title={The eleventh and twelfth data releases of the Sloan Digital Sky Survey: final data from SDSS-III},
  author={Alam, Shadab and Albareti, Franco D and Prieto, Carlos Allende and Anders, Friedrich and Anderson, Scott F and Anderton, Timothy and Andrews, Brett H and Armengaud, Eric and Aubourg, {\'E}ric and Bailey, Stephen and others},
  journal={The Astrophysical Journal Supplement Series},
  volume={219},
  number={1},
  pages={12},
  year={2015},
  publisher={IOP Publishing}
}

@article{york2000sloan,
  title={The sloan digital sky survey: Technical summary},
  author={York, Donald G and Adelman, J and Anderson Jr, John E and Anderson, Scott F and Annis, James and Bahcall, Neta A and Bakken, JA and Barkhouser, Robert and Bastian, Steven and Berman, Eileen and others},
  journal={The Astronomical Journal},
  volume={120},
  number={3},
  pages={1579},
  year={2000},
  publisher={IOP Publishing}
}

@article{rahman2015clustering,
  title={Clustering-based redshift estimation: comparison to spectroscopic redshifts},
  author={Rahman, Mubdi and M{\'e}nard, Brice and Scranton, Ryan and Schmidt, Samuel J and Morrison, Christopher B},
  journal={Monthly Notices of the Royal Astronomical Society},
  volume={447},
  number={4},
  pages={3500--3511},
  year={2015},
  publisher={Oxford University Press}
}

@article{morrison2017wizz,
  title={The-wiZZ: Clustering redshift estimation for everyone},
  author={Morrison, Christopher B and Hildebrandt, Hendrik and Schmidt, Samuel J and Baldry, Ivan K and Bilicki, Maciej and Choi, Ami and Erben, Thomas and Schneider, Peter},
  journal={Monthly Notices of the Royal Astronomical Society},
  volume={467},
  number={3},
  pages={3576--3589},
  year={2017},
  publisher={Oxford University Press}
}

@InProceedings{MPSN2021,
  title = 	 {Weisfeiler and Lehman Go Topological: Message Passing Simplicial Networks},
  author =       {Bodnar, Cristian and Frasca, Fabrizio and Wang, Yuguang and Otter, Nina and Montufar, Guido F and Li{\'o}, Pietro and Bronstein, Michael},
  booktitle = 	 {Proceedings of the 38th International Conference on Machine Learning},
  pages = 	 {1026--1037},
  year = 	 {2021},
  editor = 	 {Meila, Marina and Zhang, Tong},
  volume = 	 {139},
  series = 	 {Proceedings of Machine Learning Research},
  month = 	 {18--24 Jul},
  publisher =    {PMLR},
  url = 	 {https://proceedings.mlr.press/v139/bodnar21a.html}
}

@ARTICLE{GeometricDL,  author={Bronstein, Michael M. and Bruna, Joan and LeCun, Yann and Szlam, Arthur and Vandergheynst, Pierre},  journal={IEEE Signal Processing Magazine},   title={Geometric Deep Learning: Going beyond Euclidean data},   year={2017},  volume={34},  number={4},  pages={18-42},  doi={10.1109/MSP.2017.2693418}}

@misc{Chen2017,
  doi = {10.48550/ARXIV.1704.03924},
  author = {Chen, Yen-Chi},
  
  title = {A Tutorial on Kernel Density Estimation and Recent Advances},
  publisher = {arXiv},
  year = {2017},
  copyright = {arXiv.org perpetual, non-exclusive license}
}

@inproceedings{clusterSim,
	author = {Marek Walesiak and Andrzej Dudek},
	booksubtitle = {Proceedings of the 35th International Business Information Management Association Conference (IBIMA), 1-2 April 2020, Seville, Spain},
	booktitle = {Education Excellence and Innovation Management: A 2025 Vision to Sustain Economic Development During Global Challenges},
	editor = {Khalid S. Soliman},
	isbn = {978-0-9998551-4-1},
	pages = {325-340},
	publisher = {International Business Information Management Association (IBIMA)},
	title = {The Choice of Variable Normalization Method in Cluster Analysis},
	year = {2020}}

@article{skelclus,
author = {Zeyu Wei and Yen-Chi Chen},
title = {Skeleton Clustering: Dimension-Free Density-Aided Clustering},
journal = {Journal of the American Statistical Association},
volume = {119},
number = {546},
pages = {1124--1135},
year = {2024},
publisher = {ASA Website},
doi = {10.1080/01621459.2023.2174122},
URL = {         https://doi.org/10.1080/01621459.2023.2174122
},
eprint = { 
        https://doi.org/10.1080/01621459.2023.2174122
}
}

@InProceedings{Green2021,
  title = 	 { Minimax Optimal Regression over Sobolev Spaces via Laplacian Regularization on Neighborhood Graphs },
  author =       {Green, Alden and Balakrishnan, Sivaraman and Tibshirani, Ryan},
  booktitle = 	 {Proceedings of The 24th International Conference on Artificial Intelligence and Statistics},
  pages = 	 {2602--2610},
  year = 	 {2021},
  editor = 	 {Banerjee, Arindam and Fukumizu, Kenji},
  volume = 	 {130},
  series = 	 {Proceedings of Machine Learning Research},
  month = 	 {13--15 Apr},
  publisher =    {PMLR}
}

@book{DistributionFreeThoery,
   author = {László Györfi and Michael Kohler and Adam Krzyżak and Harro Walk},
   city = {New York, NY},
   doi = {10.1007/B97848},
   isbn = {978-0-387-95441-7},
   publisher = {Springer New York},
   title = {A Distribution-Free Theory of Nonparametric Regression},
   url = {http://link.springer.com/10.1007/b97848},
   year = {2002},
}

@article{Fan1992,
author = {Fan, Jianqing and Fan, Jianqing},
number = {420},
pages = {998--1004},
title = {{Design-adaptive Nonparametric Regression}},
journal={Journal of the American Statistical Association},
volume = {87},
year = {1992}
}

@article{Fan1996,
author = {Fan, Jianqing and Gijbels, Ir{\`{e}}ne and Hu, Tien Chung and Huang, Li Shan},
issn = {10170405},
journal = {Statistica Sinica},
number = {1},
pages = {113--127},
title = {{A study of variable bandwidth selection for local polynomial regression}},
volume = {6},
year = {1996}
}

@article{Wang2016,
archivePrefix = {arXiv},
arxivId = {1410.7690},
author = {Wang, Yu Xiang and Sharpnack, James and Smola, Alexander J. and Tibshirani, Ryan J.},
eprint = {1410.7690},
issn = {15337928},
journal = {Journal of Machine Learning Research},
pages = {1--41},
title = {{Trend filtering on graphs}},
volume = {17},
year = {2016}
}

@article{Watson1964,
 ISSN = {0581572X},
 URL = {http://www.jstor.org/stable/25049340},
 author = {Geoffrey S. Watson},
 journal = {Sankhyā: The Indian Journal of Statistics, Series A (1961-2002)},
 number = {4},
 pages = {359--372},
 publisher = {Springer},
 title = {Smooth Regression Analysis},
 volume = {26},
 year = {1964}
}

@article{Nadaraya1964,
   author = {E. A. Nadaraya},
   doi = {10.1137/1109020},
   issn = {0040-585X},
   issue = {1},
   journal = {http://dx.doi.org/10.1137/1109020},
   month = {7},
   pages = {141-142},
   publisher = {Society for Industrial and Applied Mathematics},
   title = {On Estimating Regression},
   volume = {9},
   year = {1964},
}

@article{Hastie1993,
   author = {Trevor Hastie and Clive Loader},
   doi = {10.1214/SS/1177011005},
   issn = {0883-4237},
   issue = {2},
   journal = {https://doi.org/10.1214/ss/1177011005},
   month = {5},
   pages = {139-143},
   publisher = {Institute of Mathematical Statistics},
   title = {[Local Regression: Automatic Kernel Carpentry]: Rejoinder},
   volume = {8},
   year = {1993},
}

@book{ESL,
   author = {Trevor Hastie and Robert Tibshirani and Jerome Friedman},
   city = {New York, NY},
   doi = {10.1007/978-0-387-84858-7},
   isbn = {978-0-387-84857-0},
   publisher = {Springer New York},
   title = {The Elements of Statistical Learning},
   url = {http://link.springer.com/10.1007/978-0-387-84858-7},
   year = {2009},
}

@article{Friedman1991,
   author = {Jerome H. Friedman},
   doi = {10.1214/AOS/1176347963},
   issn = {0090-5364},
   issue = {1},
   journal = {https://doi.org/10.1214/aos/1176347963},
   month = {3},
   pages = {1-67},
   publisher = {Institute of Mathematical Statistics},
   title = {Multivariate Adaptive Regression Splines},
   volume = {19},
   year = {1991},
}

@article{Altman1992,
   author = {N. S. Altman},
   doi = {10.1080/00031305.1992.10475879},
   issn = {15372731},
   issue = {3},
   journal = {American Statistician},
   pages = {175-185},
   title = {An introduction to kernel and nearest-neighbor nonparametric regression},
   volume = {46},
   year = {1992},
}

@article{Kpotufe2017,
  author  = {Samory Kpotufe and Nakul Verma},
  title   = {Time-Accuracy Tradeoffs in Kernel Prediction: Controlling Prediction Quality},
  journal = {Journal of Machine Learning Research},
  year    = {2017},
  volume  = {18},
  number  = {44},
  pages   = {1-29},
  url     = {http://jmlr.org/papers/v18/16-538.html}
}

@article{Kpotufe2009,
   author = {Samory Kpotufe},
   journal = {Advances in Neural Information Processing Systems},
   title = {Fast, smooth and adaptive regression in metric spaces},
   volume = {22},
   year = {2009},
}

@article{Kpotufe2011,
   author = {Samory Kpotufe},
   isbn = {9781618395993},
   journal = {Advances in Neural Information Processing Systems 24: 25th Annual Conference on Neural Information Processing Systems 2011, NIPS 2011},
   month = {10},
   title = {k-NN Regression Adapts to Local Intrinsic Dimension},
   url = {https://arxiv.org/abs/1110.4300v1},
   year = {2011},
}

@article{Kpotufe2013,
   author = {Samory Kpotufe and Vikas K Garg},
   journal = {Advances in Neural Information Processing Systems},
   title = {Adaptivity to Local Smoothness and Dimension in Kernel Regression},
   volume = {26},
   year = {2013},
}

@article{Tibshirani2014,
 ISSN = {00905364},
 URL = {http://www.jstor.org/stable/43556281},
 author = {Ryan J. Tibshirani},
 journal = {The Annals of Statistics},
 number = {1},
 pages = {285--323},
 publisher = {Institute of Mathematical Statistics},
 title = {ADAPTIVE PIECEWISE POLYNOMIAL ESTIMATION VIA TREND FILTERING},
 volume = {42},
 year = {2014}
}

@article{Kim2009,
   author = {Seung Jean Kim and Kwangmoo Koh and Stephen Boyd and Dimitry Gorinevsky},
   doi = {10.1137/070690274},
   issn = {00361445},
   issue = {2},
   journal = {http://dx.doi.org/10.1137/070690274},
   month = {5},
   pages = {339-360},
   publisher = {Society for Industrial and Applied Mathematics},
   title = {$\ell_1$ Trend Filtering},
   volume = {51},
   year = {2009},
}

@ARTICLE{Dijkstra1959,
    author = {E. W. Dijkstra},
    title = {A Note on Two Problems in Connexion with Graphs},
    journal = {NUMERISCHE MATHEMATIK},
    year = {1959},
    volume = {1},
    number = {1},
    pages = {269--271}
}

@article{Hartigan1979,
	Author = {Hartigan, J. A. and Wong, M. A.},
	Doi = {10.2307/2346830},
	Issn = {00359254},
	Journal = {Applied Statistics},
	Number = {1},
	Pages = {100},
	Publisher = {JSTOR},
	Title = {{Algorithm AS 136: A K-Means Clustering Algorithm}},
	Volume = {28},
	Year = {1979}}

@article{Cheng2013,
author = { Ming-Yen   Cheng  and  Hau-Tieng   Wu },
title = {Local Linear Regression on Manifolds and Its Geometric Interpretation},
journal = {Journal of the American Statistical Association},
volume = {108},
number = {504},
pages = {1421-1434},
year  = {2013},
publisher = {Taylor & Francis},
doi = {10.1080/01621459.2013.827984},
eprint = { 
        https://doi.org/10.1080/01621459.2013.827984
    
}

}

@article{Aswani2011,
author = {Anil Aswani and Peter Bickel and Claire Tomlin},
title = {{Regression on manifolds: Estimation of the exterior derivative}},
volume = {39},
journal = {The Annals of Statistics},
number = {1},
publisher = {Institute of Mathematical Statistics},
pages = {48 -- 81},
year = {2011},
doi = {10.1214/10-AOS823},
URL = {https://doi.org/10.1214/10-AOS823}
}

@article{PCR,
author = { William F.   Massy },
title = {Principal Components Regression in Exploratory Statistical Research},
journal = {Journal of the American Statistical Association},
volume = {60},
number = {309},
pages = {234-256},
year  = {1965},
publisher = {Taylor & Francis},
doi = {10.1080/01621459.1965.10480787},

URL = { 
        https://www.tandfonline.com/doi/abs/10.1080/01621459.1965.10480787
    
},
eprint = { 
        https://www.tandfonline.com/doi/pdf/10.1080/01621459.1965.10480787
    
}

}

@article{Wold1975,
   author = {Herman Wold},
   doi = {10.1017/S0021900200047604},
   issn = {0021-9002},
   issue = {S1},
   journal = {Journal of Applied Probability},
   pages = {117-142},
   publisher = {Cambridge University Press (CUP)},
   title = {Soft Modelling by Latent Variables: The Non-Linear Iterative Partial Least Squares (NIPALS) Approach},
   volume = {12},
   year = {1975},
}

@article{Marzio2014,
author = {Marco Di Marzio and Agnese Panzera and Charles C. Taylor},
title = {Nonparametric Regression for Spherical Data},
journal = {Journal of the American Statistical Association},
volume = {109},
number = {506},
pages = {748-763},
year  = {2014},
publisher = {Taylor & Francis},
doi = {10.1080/01621459.2013.866567},

URL = { 
        https://doi.org/10.1080/01621459.2013.866567
    
},
eprint = { 
        https://doi.org/10.1080/01621459.2013.866567
    
}

}

@article{Lee2016,
author = {Ann B. Lee and Rafael Izbicki},
title = {{A spectral series approach to high-dimensional nonparametric regression}},
volume = {10},
journal = {Electronic Journal of Statistics},
number = {1},
publisher = {Institute of Mathematical Statistics and Bernoulli Society},
pages = {423 -- 463},
year = {2016},
doi = {10.1214/16-EJS1112},
URL = {https://doi.org/10.1214/16-EJS1112}
}

@article{belkin06a,
  author  = {Mikhail Belkin and Partha Niyogi and Vikas Sindhwani},
  title   = {Manifold  Regularization: A Geometric Framework for Learning from Labeled and Unlabeled Examples},
  journal = {Journal of Machine Learning Research},
  year    = {2006},
  volume  = {7},
  number  = {85},
  pages   = {2399-2434},
  url     = {http://jmlr.org/papers/v7/belkin06a.html}
}

@InProceedings{Yuchen2013,
  title = 	 {Divide and Conquer Kernel Ridge Regression},
  author = 	 {Zhang, Yuchen and Duchi, John and Wainwright, Martin},
  booktitle = 	 {Proceedings of the 26th Annual Conference on Learning Theory},
  pages = 	 {592--617},
  year = 	 {2013},
  editor = 	 {Shalev-Shwartz, Shai and Steinwart, Ingo},
  volume = 	 {30},
  series = 	 {Proceedings of Machine Learning Research},
  address = 	 {Princeton, NJ, USA},
  month = 	 {12--14 Jun},
  publisher =    {PMLR},
  url = 	 {https://proceedings.mlr.press/v30/Zhang13.html}
}

@article{Guhaniyogi2016,
  author  = {Rajarshi Guhaniyogi and David B. Dunson},
  title   = {Compressed Gaussian Process for Manifold Regression},
  journal = {Journal of Machine Learning Research},
  year    = {2016},
  volume  = {17},
  number  = {69},
  pages   = {1-26},
  url     = {http://jmlr.org/papers/v17/14-230.html}
}

@ARTICLE{Zhang2019,
  author={Zhang, Xiaowei and Shi, Xudong and Sun, Yu and Cheng, Li},
  journal={IEEE Transactions on Pattern Analysis and Machine Intelligence}, 
  title={Multivariate Regression with Gross Errors on Manifold-Valued Data}, 
  year={2019},
  volume={41},
  number={2},
  pages={444-458},
  doi={10.1109/TPAMI.2017.2776260}}

@article{Lin2020,
    author = {Lin, Zhenhua and Yao, Fang},
    title = "{Functional regression on the manifold with contamination}",
    journal = {Biometrika},
    volume = {108},
    number = {1},
    pages = {167-181},
    year = {2020},
    month = {07},
    issn = {0006-3444},
    doi = {10.1093/biomet/asaa041},
    url = {https://doi.org/10.1093/biomet/asaa041},
    eprint = {https://academic.oup.com/biomet/article-pdf/108/1/167/36441008/asaa041.pdf},
}

@book{Bernhard2002,
   author = {Bernhard Schölkopf and Alexander J. Smola},
   isbn = {9780262194754},
   pages = {626},
   title = {Learning with Kernels: Support Vector Machines, Regularization, Optimization, and Beyond Adaptive computation and machine learning},
   year = {2002},
   publisher={The MIT Press}
}

@article{Bierens1983,
 ISSN = {01621459},
 URL = {http://www.jstor.org/stable/2288140},
 author = {Herman J. Bierens},
 journal = {Journal of the American Statistical Association},
 number = {383},
 pages = {699--707},
 publisher = {[American Statistical Association, Taylor & Francis, Ltd.]},
 title = {Uniform Consistency of Kernel Estimators of a Regression Function Under Generalized Conditions},
 volume = {78},
 year = {1983}
}

@techreport{COIL20,
   author = {S. A. Nene and S. K. Nayar and H. Murase},
   month = {2},
   title = {Columbia Object Image Library (COIL-20)},
   url = {https://www.cs.columbia.edu/CAVE/software/softlib/coil-20.php},
   year = {1996},
   institution = {Columbia University}
}

@INPROCEEDINGS{InvertedFiles,
  author={Sivic and Zisserman},
  booktitle={Proceedings Ninth IEEE International Conference on Computer Vision}, 
  title={Video Google: a text retrieval approach to object matching in videos}, 
  year={2003},
  volume={},
  number={},
  pages={1470-1477 vol.2},
  doi={10.1109/ICCV.2003.1238663}}

@INPROCEEDINGS{InvertedMultiIndex,
  author={Babenko, Artem and Lempitsky, Victor},
  booktitle={2012 IEEE Conference on Computer Vision and Pattern Recognition}, 
  title={The inverted multi-index}, 
  year={2012},
  volume={},
  number={},
  pages={3069-3076},
  doi={10.1109/CVPR.2012.6248038}}

\end{document}